\documentclass{article}

\PassOptionsToPackage{numbers,compress}{natbib}
\usepackage[preprint]{neurips_2026}

\usepackage[utf8]{inputenc} 
\usepackage[T1]{fontenc}    
\usepackage[hypertexnames=false]{hyperref}       
\usepackage{url}            
\usepackage{booktabs}       
\usepackage{amsfonts}       
\usepackage{nicefrac}       
\usepackage{microtype}      
\usepackage{xcolor}         

\usepackage{amsmath}
\usepackage{amssymb}
\usepackage{bm}
\usepackage{mathrsfs}
\usepackage{enumitem}
\usepackage{makecell}
\usepackage{graphicx}
\usepackage{subfigure}
\usepackage{Definitions}
\usepackage{color}
\usepackage{color, colortbl}
\usepackage{placeins}
\usepackage{tablefootnote}

\ifx\theorem\undefined
\newtheorem{theorem}{Theorem}

\newtheorem{lemma}[theorem]{Lemma}
\newtheorem{proposition}[theorem]{Proposition}

\newtheorem{corollary}[theorem]{Corollary}
\newtheorem{definition}[theorem]{Definition}

\newtheorem{assumption}{Assumption}
\fi

\title{Toward Online Robust Zero-Sum Markov Games with Function Approximation}

\author{%
  Chenhao Zhou \\
  College of Computer Science and Technology \\
  Zhejiang University
  \And
  Jiyu Wei\thanks{Corresponding author. Email: \texttt{weijiyu@zju.edu.cn}} \\
  College of Computer Science and Technology \\
  Zhejiang University
}

\begin{document}

\maketitle

\begin{abstract}
  Online robust zero-sum Markov games pose a coupled operator-learning and game-solving problem:
  from nominal interaction data, the learner must recover a worst-case minimax Bellman operator
  while maintaining strategic coverage and controlling stage-game approximation error. Existing
  robust reinforcement-learning theory does not yet provide such an operator-learning perspective
  beyond specialized multi-agent settings. We develop an optimistic dual robust Bellman framework that
  combines dual robust Bellman fitting, optimistic minimax planning, backward-validated confidence
  sets, and exploiter-mix data collection. Our analysis is organized around the robust Bellman
  residuals induced by the candidate value class and yields a regret reduction that separates
  exploration complexity, robust operator-estimation error, stage-game solution error, and
  misspecification. Under explicit interpretable assumptions, we obtain conditional theorem
  interfaces for general and projected function approximation, together with closed results in
  several structured regimes, including tabular, epochwise/cross-fit linear, self-normalized
  pure-online linear / finite-rank kernel, and projected RKHS truncation. Small-scale
  numerical experiments illustrate robustness calibration, strategic coverage, and the
  return--gap trade-off predicted by the framework.
\end{abstract}

\section{Introduction}

Robust reinforcement learning studies deployment-time mismatch between nominal training dynamics
and test dynamics \citep{iyengar2005robust,nilim2005robust,wiesemann2013robust}. In multi-agent
settings the mismatch is also strategic: transitions depend on both players, robustness is defined
relative to minimax behavior, and a worst-case Bellman operator must be learned from nominal data
\citep{zhang2020robustmarl,zheng2025drmarkov}. This leads to online robust zero-sum Markov games.

The single-agent side of this picture has recently advanced to online robust MDPs with general
function approximation \citep{ghosh2026online}, but the same extension is substantially harder in
zero-sum Markov games. First, the robust Bellman operator now combines minimax optimization with
distributional robustness, so both closure and statistical estimation must be carried out for a
robust minimax backup rather than a single-agent robust backup. Second, strategic coverage is more
delicate: the occupancies relevant for regret are induced by minimax execution, whereas the data
used to fit robust operators are generated by adaptive play, so exploration and validation must
control a genuinely multi-agent coverage gap. Third, the non-robust large-state Markov-game tools
of \citet{huang2022towards,jin2022power} do not by themselves address robust operator learning,
while existing robust Markov-game results remain concentrated in tabular, generative-model, or
linear regimes \citep{ma2023decentralized,shi2024sampleefficient,shi2024breaking,zheng2025drmarkov}.
These obstacles explain why a broad robust zero-sum Markov-game theory with function approximation
has remained open.

This paper develops an optimistic dual robust Bellman (ODRB) framework for online robust zero-sum
Markov games. At a high level, ODRB fits a worst-case Bellman backup from nominal transitions through
a dual representation, tests candidate value functions on fresh data using a frozen backup, plans
optimistically by solving the induced zero-sum stage games, and occasionally executes one-sided
response branches to cover deviations that the target pair may not visit. Epochwise freezing and
cross-fitting keep the data used to fit the robust backup separate from the data used to validate
its residual. The analysis is organized around the robust residual class \(
(I-\Tcal^\sigma)\Fcal
\) and a square-loss robust Markov-game Bellman--Eluder complexity, which together separate
exploration, robust operator estimation, stage-game solution error, and misspecification. On the
general side, we obtain conditional regret guarantees for general function approximation and
projected completeness under interpretable assumptions. On the structured side, we prove tabular,
epochwise/cross-fit linear, modified pure-online linear/finite-rank kernel, and projected
finite-dimensional or RKHS guarantees. The epochwise/cross-fit linear theorem is the main fully
specified scalable result under its stated structural conditions; the appendix also gives a
self-normalized pure-online linear/finite-rank specialization.

Our main contributions are:
\begin{itemize}[leftmargin=*]
\item We formulate an optimistic dual robust Bellman framework for online robust zero-sum Markov
games with function approximation, combining robust Bellman fitting, minimax planning,
backward-validated confidence sets, and strategic exploration. This gives a single algorithmic
template connecting robust-MDP operator learning with minimax planning in large-state zero-sum
Markov games.
\item We introduce a robust multi-agent residual perspective based on the Bellman residuals induced
by the candidate value class, namely the class
\(
(I-\Tcal^\sigma)\Fcal
\).
This yields a square-loss robust MG-BE complexity measure, a deterministic regret reduction, and
conditional interface theorems for general function approximation and projected completeness. It
separates approximation, operator learning, coverage, and stage-game errors in a form that can be
instantiated across multiple function classes.
\item We prove a structured theorem package including tabular,
epochwise/cross-fit linear, self-normalized pure-online linear/finite-rank kernel, and projected
finite-dimensional or RKHS-truncation regimes, together with concrete closure and coverage
conditions such as TV-linear-mixture closure and primitive BR-based coverage. The
epochwise/cross-fit linear theorem is the main fully specified scalable guarantee, and the projected
results show how the same framework extends to finite-dimensional lifts and kernel truncations.
\end{itemize}

We report theorem-aligned experiments illustrating robustness calibration, strategic coverage, and
the return--gap trade-off, without claiming uniform empirical superiority across operating points.

\subsection{Related Work}

\paragraph{Robust RL and robust MDPs.}
Classical robust dynamic programming and distributionally robust MDPs were developed by
\cite{iyengar2005robust,nilim2005robust,xu2010drmdp,wiesemann2013robust,mannor2016krectangular};
recent surveys include \cite{moos2022review}. More recent work studies robust function
approximation, online robust learning, and off-dynamics or offline estimation
\cite{tamar2014robustfa,ho2021partial,panaganti2022sample,shi2023curious,wang2021onlinerobust,lu2024interactive,shi2024offline,liu2024offdynamics}.
Our starting point is \cite{ghosh2026online}, which gives the robust-MDP methodology based on
dual Bellman fitting, global confidence sets, and intrinsic robust complexity.

\paragraph{Function approximation in Markov games.}
For non-robust zero-sum Markov games, recent work develops optimistic elimination, minimax
complexity measures, exploiter-assisted exploration, and spectral embeddings for continuous-state
games under rich function approximation
\cite{huang2022towards,jin2022power,xie2020learning,cisneros2023nashq,wang2023breaking,zhou2024spectral}. More
generally, the modern RL literature supplies the complexity language behind our analysis, including
Eluder-style dimensions, Bellman rank, and coverage-sensitive offline or approximate dynamic
programming viewpoints \cite{russo2013eluder,li2022eluder,jiang2017cdp,sun2019modelbased,munos2003api,munos2008fvi,chen2019batch,jin2021pessimism,xie2021bellman,jin2021bed,zanette2020ibe,wang2020gvfa,xie2022coverage}.

\paragraph{Robust Markov games and robust MARL.}
Existing robust multi-agent results are mostly tabular, generative-model, or linear
\cite{zhang2020robustmarl,ma2023decentralized,shi2024sampleefficient,shi2024breaking,farhat2026mornavi,zheng2025drmarkov}.
The recent model-based online DRMG line, represented by \cite{farhat2026mornavi}, works with
explicit TV/KL ambiguity sets and bonus-style optimistic planning in tabular settings, with
guarantees expressed in state- and joint-action-dependent terms. The closest large-state
comparison is the linear robust Markov-game result of \cite{zheng2025drmarkov}, whose algorithm
DR-CCE-LSI targets robust coarse correlated equilibrium in general-sum games under linear function
approximation, minimum-value style assumptions, and a different equilibrium notion from the
zero-sum robust minimax value studied here. Our comparison to these works is therefore assumption-
and theorem-level, not only narrative: tabular RONAVI-style methods are model-based and
bonus-driven, whereas our method learns a robust Bellman operator through dual fitting; DR-CCE-LSI
is a linear equilibrium-learning result, whereas our linear theorem is the main scalable
specialization of a broader residual-based operator-learning framework. Accordingly, our tabular,
linear, and projected guarantees should be read as structured instances of one theorem hierarchy,
rather than as isolated regime-specific algorithms, as summarized in
Table~\ref{tab:theory_positioning}.

\begin{table}[t]
\centering
\scriptsize
\setlength{\tabcolsep}{3pt}
\renewcommand{\arraystretch}{1.08}
\caption{Positioning relative to representative prior work. ``General FA theorem'' indicates that a
paper states a theorem-level guarantee for general or projected function approximation. ``Closed
scalable theorem'' indicates a fully verified large-state result inside the paper, rather than only
a tabular or generative specialization.}
\label{tab:theory_positioning}
\resizebox{\linewidth}{!}{%
\begin{tabular}{lccccc}
\toprule
Work & Problem class & Robust & Approx.\ regime & General FA theorem & Closed scalable theorem \\
\midrule
\citet{ghosh2026online}
& robust MDP
& $\checkmark$
& GFA
& $\checkmark$
& $\checkmark$ \\
\citet{huang2022towards,jin2022power}
& zero-sum MG
& $\times$
& GFA
& $\checkmark$
& $\checkmark$ \\
\citet{farhat2026mornavi}
& robust MG
& $\checkmark$
& tabular / model-based
& $\times$
& $\checkmark$ \\
\citet{ma2023decentralized,shi2024sampleefficient,shi2024breaking}
& robust MG / MARL
& $\checkmark$
& tabular / generative
& $\times$
& problem-specific \\
\citet{zheng2025drmarkov}
& robust MG
& $\checkmark$
& linear
& $\times$
& $\checkmark$ \\
 \textbf{Ours}
& \textbf{robust zero-sum MG}
& $\checkmark$
& \textbf{GFA, projected, linear / kernel}
& $\checkmark$
& $\checkmark$ \\
\bottomrule
\end{tabular}}
\end{table}

\section{Preliminaries}

We study episodic two-player zero-sum robust Markov games
\(
\Mcal=(\Scal,\Acal,\Bcal,H,\{r_h\},\{P_h^0\},\{\Ucal_h\})
\),
where player~1 maximizes reward and player~2 minimizes it. At stage
\(
h\in[H]
\),
the game is in state
\(
s_h\in\Scal
\),
player~1 chooses
\(
a_h\in\Acal
\),
player~2 chooses
\(
b_h\in\Bcal
\),
and player~1 receives reward
\(
r_h(s_h,a_h,b_h)\in[0,1]
\).
The nominal next-state law is
\(
P_h^0(\cdot\mid s_h,a_h,b_h)
\),
but the true transition is allowed to vary inside an uncertainty set
\(
\Ucal_h(s,a,b)
\)
around
\(
P_h^0(\cdot\mid s,a,b)
\),
typically defined by TV, KL, or more general $\phi$-divergence balls
\citep{iyengar2005robust,nilim2005robust,ghosh2026online}. A policy pair
\(
(\pi,\nu)
\)
consists of Markov decision rules
\(
\pi_h(\cdot\mid s)\in\Delta(\Acal)
\)
and
\(
\nu_h(\cdot\mid s)\in\Delta(\Bcal)
\),
where
\(
\Delta(\Acal)
\)
and
\(
\Delta(\Bcal)
\)
denote the simplices over the two action sets.
We regard
\(
(a_h,b_h)
\)
as a joint action in
\(
\Acal\times\Bcal
\),
while keeping the two-player notation explicit because minimax structure is central throughout the
paper.
Throughout the paper, \(\sigma\) denotes the robustness radius indexing the ambiguity sets; a
representative form is
\[
\Ucal_h(s,a,b)
=
\Bigl\{
P:\ D\bigl(P,P_h^0(\cdot\mid s,a,b)\bigr)\le \sigma
\Bigr\},
\]
where \(D\) may be total variation, KL, or \(\chi^2\) divergence. We use \(\sigma\) in the value
and Bellman notation to indicate dependence on this robustness level, and \(\sigma=0\) recovers
the nominal, non-robust game.

\subsection{Robust Value Functions and Bellman Equation}

Given a policy pair
\(
(\pi,\nu)
\),
the robust action-value and value functions are defined backward with terminal condition
\(
V_{H+1}^{\pi,\nu,\sigma}\equiv 0
\)
by
\begin{equation*}
Q_h^{\pi,\nu,\sigma}(s,a,b)
=
r_h(s,a,b)+\inf_{P\in\Ucal_h(s,a,b)} \EE_{s'\sim P}\!\left[V_{h+1}^{\pi,\nu,\sigma}(s')\right],
\end{equation*}
with
\(
V_h^{\pi,\nu,\sigma}(s)
=
\EE_{a\sim\pi_h(\cdot\mid s),\,b\sim\nu_h(\cdot\mid s)}
\left[Q_h^{\pi,\nu,\sigma}(s,a,b)\right].
\)
Thus
\(
Q_h^{\pi,\nu,\sigma}(s,a,b)
\)
is the robust continuation value after fixing the current action pair, whereas
\(
V_h^{\pi,\nu,\sigma}(s)
\)
averages that quantity under the two players' mixed actions at state
\(
s
\).
Equivalently, the pair
\(
(Q_h^{\pi,\nu,\sigma},V_h^{\pi,\nu,\sigma})
\)
satisfies the robust Bellman recursion for the joint Markov policy
\(
(\pi,\nu)
\).
The optimal robust value at stage
\(
h
\)
is the minimax quantity
\(
V_h^{\star,\sigma}(s):=\max_{\pi}\min_{\nu}V_h^{\pi,\nu,\sigma}(s)
\),
with the max and min interpreted over Markov policy pairs.
This is the zero-sum robust analogue of the usual game value: player~1 optimizes against both
player~2 and the worst-case transition selector, while player~2 minimizes the resulting robust
return. In the online setting considered here, regret is measured relative to this robust minimax
benchmark at the initial state.

The model uses one shared zero-sum payoff
\(
J^\sigma(\pi,\nu):=V_1^{\pi,\nu,\sigma}(s_1)
\): player~1 maximizes it, whereas player~2 and Nature minimize it. Thus the apparent asymmetry is
operational---the maximizing controller must withstand both a strategic opponent and transition
misspecification---rather than a mismatch between two payoff definitions. This covers, for example,
a defender facing a strategic attacker under uncertain failures. If each player instead separately
robustifies its own subjective payoff, then in general
\(
\inf_P J_P\neq-\inf_P(-J_P)=\sup_P J_P
\), so the common zero-sum payoff and the Bellman saddle recursion no longer apply.

\subsection{Robust Minimax Bellman Operator}

For a candidate $Q$-sequence $f=\{f_h\}_{h=1}^H$, define the stage-game value
\begin{equation*}
V_{f,h}(s)
:=
\max_{p\in\Delta(\Acal)}\min_{q\in\Delta(\Bcal)}
\sum_{a,b} p(a)q(b) f_h(s,a,b).
\end{equation*}
This is the value of the zero-sum matrix game obtained by freezing the state at
\(
s
\)
and using
\(
f_h(s,\cdot,\cdot)
\)
as the payoff matrix. In particular,
\(
V_{f,h}(s)
\)
compresses a state-dependent action-value table into the corresponding minimax state value.
The corresponding robust minimax Bellman operator is
\begin{equation*}
(\Tcal_h^\sigma f_{h+1})(s,a,b)
=
r_h(s,a,b)+\inf_{P\in\Ucal_h(s,a,b)} \EE_{s'\sim P}\!\left[V_{f,h+1}(s')\right].
\end{equation*}
This is the multi-agent analogue of the robust Bellman operator in robust MDPs, except that the
continuation value is now itself a minimax object.
For the optimal robust value, the Bellman equation takes the fixed-point form
\(
Q_h^{\star,\sigma}=\Tcal_h^\sigma Q_{h+1}^{\star,\sigma}
\)
and
\(
V_h^{\star,\sigma}(s)=V_{Q^\star,h}(s)
\).
Our algorithmic and theoretical developments are therefore organized around estimating and
controlling this robust minimax Bellman operator.

\subsection{Online Objective and Occupancies}

In episode
\(
k
\),
the learner selects a policy pair
\(
(\pi^k,\nu^k)
\)
and is evaluated against the corresponding worst-case selector inside the uncertainty set. The
performance criterion in the main text is robust regret relative to the optimal robust minimax
value at the initial state. The analysis therefore uses the executed occupancy induced jointly by
the policy pair and the worst-case selector. By contrast, robust Bellman fitting and backward
validation are performed from nominal interaction data generated by the data-collection policy.
This distinction is important in the multi-agent robust setting: regret is controlled on the
worst-case execution trajectory, whereas statistical estimation is carried out on nominal data, and
the coverage conditions in the theory are precisely what connect these two occupancy families.
More precisely, if
\(
s_1
\)
is the fixed initial state, then the cumulative performance criterion studied in the paper is
\begin{equation*}
\mathrm{Reg}(K)
:=
\sum_{k=1}^K
\Bigl(
V_1^{\star,\sigma}(s_1)-V_1^{\pi^k,\nu^k,\sigma}(s_1)
\Bigr).
\end{equation*}
The comparator is the robust minimax value, but the summand above is a signed, one-sided
value-comparator difference for the jointly generated pair. It is not the two-sided saddle
exploitability
\begin{equation*}
\max_{\pi'}V_1^{\pi',\nu^k,\sigma}(s_1)
-\min_{\nu'}V_1^{\pi^k,\nu',\sigma}(s_1).
\end{equation*}
Consequently, a sublinear bound on \(\mathrm{Reg}(K)\) controls the stated robust value comparator
but does not by itself certify that an output pair is an approximate Nash equilibrium; the latter
would require separate control of both unilateral-deviation terms.

\subsection{Function Classes and Residual Complexity}

Let
\(
\Fcal=\Fcal_1\times\cdots\times\Fcal_H
\)
be the candidate $Q$-function class and let $\Gcal_h$ denote the dual class used to fit the robust
backup at stage $h$. Our analysis is organized around the robust residual class
\begin{equation*}
\Rscr_h
:=
\left\{f_h-\Tcal_h^\sigma f_{h+1}: f\in\Fcal\right\}.
\end{equation*}
Following the Bellman--Eluder viewpoint of \citet{ghosh2026online} and the Markov-game GFA
literature \citep{huang2022towards,jin2022power}, we measure exploration difficulty through a
square-loss robust MG-BE dimension attached to this residual process. The exact counting definition
and its projected variants appear in the appendix; in the main text it suffices to interpret the
dimension as the number of statistically distinguishable robust residual directions encountered
along learning.
When the ambiguity set collapses to the singleton nominal kernel, \(\Tcal_h^\sigma\) reduces to
the usual Markov-game Bellman operator and the robust MG-BE notion reduces to its non-robust
counterpart. The resulting bounds recover the standard \(\sqrt K\)-type episode dependence of
non-robust zero-sum Markov games with general or linear function approximation
\citep{huang2022towards,jin2022power}, up to the different complexity, coverage, horizon, and
logarithmic factors; we do not claim an improved nominal rate. Under nontrivial ambiguity sets, the
residual class may be richer because the backup also contains worst-case transition directions.

The residual class is the right object because it is exactly what appears after Bellman
telescoping in the regret decomposition. In other words, the analysis does not need to identify the
entire robust Bellman operator uniformly over the ambient function class; it only needs to control
those residual directions that are actually induced by candidate functions in
\(
\Fcal
\).
This viewpoint is what makes it possible to separate exploration complexity from the lower-order
errors due to dual approximation, projection bias, and inexact stage-game solution.
Two concrete examples are useful. In tabular games, the robust residual class lives in a space
with one coordinate per joint state-action triple, so the robust MG-BE complexity is
\(
O(H|\Scal||\Acal||\Bcal|)
\)
up to logarithmic factors. In linear classes, if the residual process admits witness features of
dimension
\(
d
\),
then the appendix shows that the robust MG-BE dimension scales on the order of
\(
d\log(1+1/\epsilon^2)
\),
equivalently
\(
O(d\log K)
\)
when evaluated at
\(
\epsilon=1/\sqrt K
\),
matching the usual linear Bellman--Eluder behavior.

\section{Method}

We study a family of optimistic algorithms for online robust zero-sum Markov games. The method has
four components: robust Bellman fitting through a dual representation, global confidence sets over
the value-function class, minimax policy extraction, and strategically mixed data collection. We
refer to this template as \emph{optimistic dual robust Bellman} (ODRB). The main theorem analyzes
an epochwise implementation with cross-fit validation; the appendix gives additional structured
variants, including a self-normalized pure-online specialization.

\subsection{Dual Robust Bellman Fitting}

For a candidate action-value function
\(
f=(f_1,\ldots,f_H)\in\Fcal
\),
the robust Bellman target at stage
\(
h
\)
is
\begin{equation}
(\Tcal_h^\sigma f_{h+1})(s,a,b)
=
r_h(s,a,b)
+
\inf_{P\in\Ucal_h(s,a,b)} \EE_{s'\sim P}\!\left[V_{f,h+1}(s')\right].
\label{eq:robust_mg_backup}
\end{equation}
Here
\(
V_{f,h+1}(s')
\)
is the value of the zero-sum stage game with payoff matrix
\(
f_{h+1}(s',\cdot,\cdot)
\).

For $\phi$-divergence uncertainty sets, the inner robust expectation admits a dual representation.
Writing
\(
g=(g_\eta,g_\nu)\in\Gcal_h
\)
for a stagewise dual variable, and writing
\(
\phi^\star(u):=\sup_{t}\{ut-\phi(t)\}
\)
for the convex conjugate of \(\phi\), the corresponding pointwise dual loss is
\begin{equation*}
\ell_\phi(f_{h+1};s,a,b,s';g)
=
g_\eta(s,a,b)\sigma-g_\nu(s,a,b)
+
g_\eta(s,a,b)\phi^\star\!\left(
\frac{g_\nu(s,a,b)-V_{f,h+1}(s')}{g_\eta(s,a,b)}
\right).
\end{equation*}
With this loss convention, the robust continuation value equals
\(
-\inf_g\EE_{P_h^0}[\ell_\phi(f_{h+1};s,a,b,s';g)]
\); the leading minus sign is essential.
This converts robust Bellman fitting into empirical risk minimization over the dual class
\(
\Gcal_h
\)
using nominal interaction data. Concretely, for a stagewise sample buffer
\(
\Dcal_h=\{(s_i,a_i,b_i,s_i')\}_{i=1}^{N_h}
\),
we define the empirical dual objective
\begin{equation*}
\widehat{\Lcal}_{h}(g;f)
:=
\frac{1}{N_h}
\sum_{i=1}^{N_h}
\ell_\phi(f_{h+1};s_i,a_i,b_i,s_i';g).
\end{equation*}
In the epochwise implementation, \(
\widehat{\Lcal}_{h,e}^{\mathrm{pre}}(g;f)
\) denotes the same empirical average computed on the pre-epoch buffer. The reward term does not
appear in \(\widehat{\Lcal}_h\) because the dual fit estimates only the robust expectation term in
\eqref{eq:robust_mg_backup}; the immediate reward is added back when the plug-in Bellman operator
is formed. The resulting plug-in operator is the robust analogue of fitted Bellman evaluation, and
the theory isolates its dual approximation error explicitly.

\subsection{Confidence Sets and Optimistic Planning}

After collecting data up to the current update time, the algorithm refits the dual variable on
pooled stage-$h$ data:
\begin{equation*}
g_{h,f}\in
\argmin_{g\in\Gcal_h}\widehat{\Lcal}_{h}(g;f).
\end{equation*}
This induces a plug-in robust Bellman operator, obtained by taking the dual representation of the
robust Bellman backup and replacing the population dual optimizer by the empirical refit
\(
g_{h,f}
\):
\(
\widehat{\Tcal}_{h,f}^{\sigma}
\),
defined by
\begin{equation*}
\bigl(\widehat{\Tcal}_{h,f}^{\sigma}f_{h+1}\bigr)(s,a,b)
:=
r_h(s,a,b)
-
\EE_{s'\sim P_h^0(\cdot\mid s,a,b)}
\sbr{\ell_\phi(f_{h+1};s,a,b,s';g_{h,f})}.
\end{equation*}
The learner then maintains a global confidence set
\begin{equation}
\mathfrak F
:=
\left\{
f\in\Fcal:
\widehat{\mathrm{BV}}_{h}(f)\le \beta_h
\text{ for every stage }h
\right\},
\label{eq:backward_conf_set}
\end{equation}
where
\(
\widehat{\mathrm{BV}}_h(f)
\)
is built from the plug-in robust Bellman residual
\(
\widehat\delta_h(f;s,a,b):=
f_h(s,a,b)-(\widehat{\Tcal}_{h,f}^{\sigma}f_{h+1})(s,a,b).
\)
Concretely, if the past stage-$h$ data are partitioned into batches
\(
\{\Dcal_h^{(j)}\}_{j<k}
\)
with
\(
n_{h,j}:=|\Dcal_h^{(j)}|,
\)
the one-update-per-episode statistic is
\begin{equation*}
\widehat{\mathrm{BV}}_{h,k}(f)
:=
\sum_{j=1}^{k-1}
\frac{1}{n_{h,j}}
\sum_{(s,a,b,r,s')\in\Dcal_h^{(j)}}
\sbr{\widehat\delta_{h,k}(f;s,a,b)}^2.
\end{equation*}
The epochwise/cross-fit version uses the same score on the fresh within-epoch fold with the frozen operator
\(
\widehat{\Tcal}_{h,f}^{\sigma,(e)}.
\)

The optimistic value function is chosen from
\(
\mathfrak F
\)
by maximizing the value at the initial state. The role of backward validation is to remove
functions whose plug-in robust Bellman residual is inconsistent with the observed data while, with
high probability, retaining a Bellman-consistent comparator. This is the key mechanism linking
statistical calibration to the regret analysis.
Here ``backward'' means that consistency is checked against the stagewise robust Bellman recursion,
starting from the terminal continuation. In the epochwise version, ``cross-fit'' means that the
candidate-dependent dual and plug-in operator are fit and frozen on pre-epoch data, whereas the
induced residual witness is evaluated on a fresh within-epoch fold.

\subsection{Policy Extraction and Data Collection}

Given an optimistic function
\(
f\in\mathfrak F
\),
the policy pair for stage
\(
h
\)
is obtained by solving the zero-sum matrix game with payoff matrix
\(
f_h(s,\cdot,\cdot)
\).
This yields a stagewise minimax policy pair
\(
(\pi_h^f,\nu_h^f)
\)
used both for planning and for defining the robust value
\(
V_{f,h}
\).

To guarantee strategic coverage, data are not collected solely from the current target policy.
Let \(\widehat V_1^k(\pi,\nu)\) be the policy-pair value under the current epoch's frozen plug-in
robust operator, and define the fitted-game response branches
\begin{equation*}
\operatorname{BR}_1^k\in\argmax_{\pi}\widehat V_1^k(\pi,\nu^k),
\qquad
\operatorname{BR}_2^k\in\argmin_{\nu}\widehat V_1^k(\pi^k,\nu).
\end{equation*}
In each episode, the exploiter-mix rule executes
\begin{equation}
\begin{aligned}
(\pi^k,\nu^k)&\quad\text{with probability }1-p,\\
(\operatorname{BR}_1^k,\nu^k)&\quad\text{with probability }p/2,\qquad
(\pi^k,\operatorname{BR}_2^k)\quad\text{with probability }p/2.
\end{aligned}
\label{eq:main_exploiter_mix}
\end{equation}
Approximate fitted-game responses may be used. Exact response optimality is not an additional
premise of the regret conversion, which uses history measurability and bounded rewards; the
coverage proof instead invokes the stated branch-induced reachability and excitation conditions.
The two response branches expose occupancies reachable under one unilateral strategic deviation
that the target pair may not visit.

\subsection{Implementation Used in the Main Theory}

The main theorem studies an epochwise implementation with cross-fit validation. Here ``epochwise''
means that the dual operator is refit only at epoch boundaries, while ``cross-fit validation''
means that the resulting frozen witness family is validated on fresh within-epoch data rather than
on the same samples used for refitting. This construction removes the most delicate
adaptive-dependence issue in the raw online validation rule while preserving the same optimistic
planning and exploiter-mix structure.

\begin{algorithm}[t]
\caption{Optimistic Dual Robust Bellman (ODRB) with Epochwise Refitting and Cross-Fit Validation}
\label{alg:odrb_epochwise}
\begin{algorithmic}[1]
\REQUIRE Value class $\Fcal$, dual classes $\{\Gcal_h\}_{h=1}^H$, epoch partition $\{E_e\}_{e\ge 1}$,
validation radii $\{\beta_{h,e}\}$, exploiter-mix parameter $p$
\STATE \textbf{Output:} deployed sequence $\{(f^e,\pi^e,\nu^e)\}_{e\ge 1}$ of optimistic functions and stagewise minimax policy pairs
\STATE Initialize pooled stagewise data buffers $\Dcal_h\gets\emptyset$ and confidence set $\mathfrak F_1\gets\Fcal$
\FOR{epochs $e=1,2,\ldots$}
\STATE Using pre-epoch data $\Dcal_h$, fit frozen dual variables $g_{h,f}^{(e)}\in\arg\min_{g\in\Gcal_h}\widehat{\Lcal}_{h,e}^{\mathrm{pre}}(g;f)$
\STATE Form frozen plug-in operators $\widehat{\Tcal}_{h,f}^{\sigma,(e)}$ and the epoch confidence set
\[
\mathfrak F_e
:=
\left\{
f\in\Fcal:\widehat{\mathrm{BV}}_{h,e}(f)\le \beta_{h,e},\ \forall h
\right\}
\]
\STATE Choose $f^e\in\arg\max_{f\in\mathfrak F_e} V_{f,1}(s_1)$
\STATE Extract the stagewise minimax policy pair $(\pi^e,\nu^e)$ from $\{f_h^e\}_{h=1}^H$
\FOR{episodes $k\in E_e$}
\STATE Sample and execute one of the three branches in \eqref{eq:main_exploiter_mix}; collect trajectory $\tau^k$
\STATE Add $\tau^k$ to the within-epoch validation fold and to the pooled buffers $\{\Dcal_h\}_{h=1}^H$
\ENDFOR
\STATE Use the fresh within-epoch fold together with the frozen witnesses induced by $\widehat{\Tcal}_{h,f}^{\sigma,(e)}$ to define the next validation statistics $\widehat{\mathrm{BV}}_{h,e+1}(f)$
\ENDFOR
\STATE \textbf{return} $\{(f^e,\pi^e,\nu^e)\}_{e\ge 1}$
\end{algorithmic}
\end{algorithm}

Algorithm~\ref{alg:odrb_epochwise} is the template analyzed by the main structured scalable linear
theorem. The appendix further studies two structured extensions: a projected version, where exact
closure is replaced by stagewise projection and explicit bias terms, and a self-normalized
pure-online version, where the raw validation score is replaced by a self-normalized statistic with
ridge dual refits.

Epochwise/cross-fit validation is introduced for statistical rather than algorithmic reasons. If
the same data are used both to refit the robust plug-in operator and to validate the induced
witness, the resulting witness class adapts too tightly to the data stream. Freezing the operator
at the start of each epoch and validating it on fresh within-epoch samples breaks this feedback
loop while preserving the same optimistic planning rule and exploiter-mix data collection.

The main hyperparameters have transparent roles: increasing \(p\) can improve strategic coverage
and witness excitation, but incurs an explicit exploration cost and the factor \(1/(1-p)\) in the
regret bound; the validation radii \(\beta_{h,e}\) control how conservatively functions are retained
in the confidence set, and the epoch schedule controls the tradeoff between frequent operator
refits and statistically separated validation.
From a computational perspective, the extra work relative to a matched non-robust backbone comes from stagewise dual
refits and repeated zero-sum matrix-game solves; in the linear specialization, the former reduces
to ridge problems and the latter to \(|\Acal|\times|\Bcal|\) matrix games.

\section{Theory}

This section states the main assumptions and the principal regret guarantees. Full proofs are
deferred to the appendix. In addition to the results stated here, the appendix establishes a
self-normalized pure-online theorem for linear and finite-rank kernel classes.
Throughout, \(\mathrm{Reg}(K)\) denotes the signed cumulative one-sided comparator difference
relative to the robust minimax value defined in the preliminaries.

\subsection{Main-Text Assumptions}

\begin{assumption}[Robust Bellman closure]
\label{asmp:main_closure}
For the exact-closure results, the robust Bellman operator preserves the working class:
\(
\Tcal_h^\sigma f_{h+1}\in\Fcal_h
\)
for every
\(
h\in[H]
\)
and
\(
f\in\Fcal
\).
For projected results, this requirement is replaced by stagewise projectors together with explicit
projection and operator-remainder terms.
\end{assumption}

\begin{assumption}[Dual learnability]
\label{asmp:main_dual}
For a probability measure \(\zeta\) on \(\Scal\times\Acal\times\Bcal\), let
\begin{equation*}
\Lcal_h^\zeta(g;f)
:=
\EE_{(s,a,b)\sim\zeta}\EE_{s'\sim P_h^0(\cdot\mid s,a,b)}
\sbr{\ell_\phi(f_{h+1};s,a,b,s';g)}.
\end{equation*}
For every \(h\), \(f\in\Fcal\), and relevant \(\zeta\), the working dual class satisfies
\begin{equation}
\inf_{g\in\Gcal_h}\Lcal_h^\zeta(g;f)
\leq
\inf_{\bar g\in\bar\Gcal_h}\Lcal_h^\zeta(\bar g;f)
+\epsilon_{\mathrm{dual}},
\label{eq:main_dual_learnability}
\end{equation}
where \(\bar\Gcal_h\) contains the exact robust dual optimizer. The dual loss is bounded and
Lipschitz and the working classes admit the finite covers used for uniform concentration; the
corresponding constants are stated in Assumptions~\ref{asmp:bounded_dual}--\ref{asmp:cover_dual}.
\end{assumption}

\begin{assumption}[Support preservation and occupancy transfer]
\label{asmp:main_support}
Measurable worst-case selectors exist, remain inside nominally learnable support, and induce the
stage-\(h\) executed occupancy \(\widetilde\mu_h^k\). Let the pooled past data-collection occupancy
be
\(
\bar{\widetilde\eta}_{h,k}
:=\sum_{j<k}\lambda_{h,j}^{(k)}\widetilde\eta_h^j
\), where \(\lambda_{h,j}^{(k)}\) is proportional to the stage-\(h\) sample count from episode
\(j\). For every \(h\) and \(k\ge2\),
\begin{equation}
\widetilde\mu_h^k\ll\bar{\widetilde\eta}_{h,k},
\qquad
\frac{d\widetilde\mu_h^k}{d\bar{\widetilde\eta}_{h,k}}
\le C_h^{\mathrm{mix}}
\quad \bar{\widetilde\eta}_{h,k}\text{-almost surely}.
\label{eq:main_occupancy_transfer}
\end{equation}
\end{assumption}

A concrete sufficient condition is useful for interpreting
Assumption~\ref{asmp:main_support}. If each executed policy contains a smoothing distribution
\(u_h(a,b\mid s)\) with weight at least \(\gamma_{\min}\),
\(\underline u_h:=\inf_{s,a,b}u_h(a,b\mid s)>0\), and the current state marginal is dominated by
the pooled past marginal with factor \(C_{S,h}\), then
\begin{equation}
C_h^{\mathrm{mix}}
\le \frac{C_{S,h}}{\gamma_{\min}\underline u_h}.
\label{eq:main_cmix_sufficient}
\end{equation}
The robustness radius \(\sigma\) affects this coefficient through the worst-case state marginal
and hence \(C_{S,h}=C_{S,h}(\sigma)\). Without an additional lower-mass or mixing condition there
is no universal radius-only or monotone law for \(C_h^{\mathrm{mix}}\).

\begin{assumption}[Inexact stage-game solution]
\label{asmp:main_stage}
For every episode \(k\), stage \(h\), and state \(s\), the extracted pair satisfies
\begin{equation}
V_{f^k,h}(s)
\le
\EE_{a\sim\pi_h^k(\cdot\mid s),\,b\sim\nu_h^k(\cdot\mid s)}
\sbr{f_h^k(s,a,b)}
+\epsilon_{\mathrm{eq}}.
\label{eq:main_stage_game_accuracy}
\end{equation}
The matrix-game value is \(1\)-Lipschitz under the entrywise sup norm. Thus extraction contributes
at most \(H\epsilon_{\mathrm{eq}}\) per episode.
\end{assumption}

Assumptions~\ref{asmp:main_closure}--\ref{asmp:main_stage} separate structural requirements
(closure and stage-game solution) from statistical ones (dual learning and occupancy transfer).
In spirit, they are robust zero-sum counterparts of the
realizability/completeness/regularity assumptions standard in general-function-approximation RL
and Markov games \citep{huang2022towards,jin2022power,ghosh2026online}. The additional ingredients
here are occupancy transfer from nominal interaction to worst-case evaluation and Lipschitz control
of approximate stage-game solves.
The abstract reduction depends only on residual and operator budgets; concrete classes differ in how
they verify them. Theorems~\ref{thm:main_regret}--\ref{thm:gfa_projected_main} are conditional GFA
interfaces, Theorem~\ref{thm:linear_case} closes the budgets under linear coverage/burn-in
conditions, and Theorem~\ref{thm:rkhs_projected_main} exposes the RKHS rank--bias trade-off.

\subsection{Main Results}

We first state the general regret guarantee. Let
\(
d_{\mathrm{rob\text{-}MG}}
=
\dim_{\mathrm{BE},2}^{\mathrm{rob\text{-}MG}}(\Fcal,\Pi,1/\sqrt{K})
\)
denote the square-loss robust Markov-game Bellman--Eluder dimension, where
\(
\Pi=(\Pi_h)_{h=1}^H
\)
is the stagewise occupancy family induced by the executed trajectories.
As summarized in the preliminaries and verified in the appendix, this reduces to tabular counting
in finite games and to a linear witness dimension, up to logarithmic factors, in linear classes.
Here \(B_\phi(\sigma)\) is the dual-loss envelope (specialized in the appendix for TV, KL, and
\(\chi^2\)). Throughout, \(
\epsilon_{\mathrm{dual}}
\) denotes dual approximation error, \(
\epsilon_{\mathrm{eq}}
\) stage-game solution error, and \(
\epsilon_{\mathrm{appr}}
\) the comparator or model-misspecification term, which vanishes under exact realizability.

\begin{theorem}[General regret guarantee]
\label{thm:main_regret}\label{thm:gfa_case}
Suppose Assumptions~\ref{asmp:main_closure}--\ref{asmp:main_stage} hold, and that the
backward-validation and dual-learning arguments provide optimism, a residual budget of order
\(
\widetilde O(HB_\phi^2(\sigma)d_{\mathrm{rob\text{-}MG}})
\),
and an operator-learning budget of order
\(
\widetilde O(HB_\phi(\sigma)\sqrt K)
\).
Then, with probability at least
\(
1-\delta
\),
\begin{equation*}
\mathrm{Reg}(K)
\lesssim
H
+
\frac{1}{1-p}
\left[
\widetilde O\!\left(
H B_\phi(\sigma)\sqrt{d_{\mathrm{rob\text{-}MG}}K}
\right)
+
HK\epsilon_{\mathrm{dual}}
+
HK\epsilon_{\mathrm{eq}}
+
K\epsilon_{\mathrm{appr}}
+
pHK
\right].
\end{equation*}
\end{theorem}

Theorem~\ref{thm:main_regret} is the compressed main-text form of the appendix statistical-input
corollary: residual complexity drives the leading statistical term, while dual approximation,
stage-game error, and model misspecification remain additive and scale linearly in \(K\). A
sufficient regime for the displayed upper bound to be \(o(K)\) is
\begin{equation}
\begin{gathered}
p=p_K=o(1),\qquad
\epsilon_{\mathrm{dual}},\epsilon_{\mathrm{eq}},\epsilon_{\mathrm{appr}}=o(1),\\
B_\phi(\sigma)\sqrt{d_{\mathrm{rob\text{-}MG}}K}=o(K).
\end{gathered}
\label{eq:main_sublinear_regime}
\end{equation}
Under exact realizability, exact dual specification, and exact stage-game solving, the three error
terms vanish. If the coverage and excitation constants stay bounded independently of \(p\), then
\(p_K\asymp K^{-1/2}\) preserves a \(\widetilde O(\sqrt K)\)-type rate.

For broader classes, exact closure may be replaced by projected closure. We write
\(
R_{\mathrm{BV,proj}}(K)
\)
and
\(
R_{\mathrm{dual,proj}}(K,\delta)
\)
for the projected residual-validation and projected dual/transfer contributions, respectively; the
latter absorbs the projected operator-learning term together with the dual approximation component
that remains explicit in the appendix.

\begin{theorem}[Projected function-approximation guarantee]
\label{thm:gfa_projected_main}
Suppose Assumptions~\ref{asmp:main_dual}--\ref{asmp:main_stage} hold, and replace
Assumption~\ref{asmp:main_closure} by stagewise projectors
\(
\{\mathsf P_h\}_{h=1}^H
\)
with projection bias
\(
\epsilon_{\mathrm{proj}}
\)
and operator remainder
\(
\epsilon_{\mathrm{op,proj}}
\).
If the corresponding projected residual and operator budgets are available, then, with probability
at least
\(
1-\delta
\),
\begin{equation*}
\begin{aligned}
\mathrm{Reg}(K)
\le\;&
H
+
\frac{1}{1-p}
\Bigl[
R_{\mathrm{BV,proj}}(K)
+
R_{\mathrm{dual,proj}}(K,\delta)
+
HK\epsilon_{\mathrm{proj}}
+
HK\epsilon_{\mathrm{op,proj}}
\\
&\qquad
+
K\epsilon_{\mathrm{appr}}
+
HK\epsilon_{\mathrm{eq}}
+
pHK
\Bigr].
\end{aligned}
\end{equation*}
\end{theorem}

Theorem~\ref{thm:gfa_projected_main} is the compressed projected counterpart of the appendix
projected corollary: \(R_{\mathrm{BV,proj}}\) and \(R_{\mathrm{dual,proj}}\) abbreviate the
explicit projected residual and operator/dual bookkeeping terms expanded there, while departures
from exact closure are recorded through \(\epsilon_{\mathrm{proj}}\) and
\(\epsilon_{\mathrm{op,proj}}\).

\begin{theorem}[Epochwise/cross-fit linear theorem]
\label{thm:linear_case}
Assume TV uncertainty, a linear-mixture nominal kernel with augmented features, affine
continuation coupling for the dual fit, and a predictable covariance budget after burn-in. Here
\(d_{\mathrm{lin}}\) denotes the resulting linear robust MG-BE dimension, \(B_F\) is the uniform
bound on the value class, \(\kappa_{\mathrm{lin}}\) is the witness feature norm bound,
\(\Psi_{\mathrm{dual}}\) is the dual estimation scale from the ridge/self-normalized analysis,
\(\underline\nu_{\min}:=\min_h \underline\nu_h/2\) is the post-burn-in covariance lower bound, and
\(c_{\mathrm{ep}}\) is the geometric-epoch constant. If \(d_{W,h}\) and \(\kappa_{W,h}\) are the
stagewise witness dimension and feature-norm bound and \(M\) is the number of epochs, define
\begin{equation}
N_{\mathrm{burn}}(\delta)
:=
\max_{h\in[H]}
\left\lceil
\frac{64\kappa_{W,h}^4\log(2d_{W,h}HM/\delta)}{\underline\nu_h^2}
\right\rceil.
\label{eq:main_linear_burnin}
\end{equation}
Then the epochwise/cross-fit algorithm satisfies, with probability at least
\(
1-\delta
\),
\begin{equation*}
\begin{aligned}
\mathrm{Reg}(K)
\le\;&
c_{\mathrm{ep}}H N_{\mathrm{burn}}(\delta)
+
\frac{1}{1-p}
\Biggl[
\widetilde O\!\left(
H\Bigl(
B_F+\kappa_{\mathrm{lin}}\Psi_{\mathrm{dual}}/\sqrt{\underline\nu_{\min}}
\Bigr)
\sqrt{d_{\mathrm{lin}}K}
\right)
\\
&\qquad
+
K\epsilon_{\mathrm{appr}}
+
HK\epsilon_{\mathrm{eq}}
+
pHK
\Biggr].
\end{aligned}
\end{equation*}
\end{theorem}

Here \(N_{\mathrm{burn}}\) is an analysis threshold over the actual run, not a separate warm-up;
early episodes cost at most \(c_{\mathrm{ep}}HN_{\mathrm{burn}}\). Under BR-only coverage,
\(\underline\nu_h\propto p\) gives costs \(\sqrt{K/p}\), \(p^{-2}\), and \(pK\), balanced by
\(p_K\asymp K^{-1/3}\) for a \(\widetilde O(K^{2/3})\)-type bound. Adapting to unknown coverage or
excitation constants is outside the analysis.

\begin{theorem}[RKHS spectral-truncation projected theorem]
\label{thm:rkhs_projected_main}
Let the continuation, dual, and residual classes lie in stagewise RKHSs, and let \(\Pi_{h,m}\)
denote the rank-$m$ spectral projector at stage \(h\). Applying
Theorem~\ref{thm:gfa_projected_main} to the truncated classes yields a truncation-dependent burn-in
threshold
\begin{equation*}
N_{\mathrm{burn}}^{\mathrm{RKHS}}(m,\delta)
:=
\max_{h\in[H]}
\left\lceil
\frac{64(\kappa_{W,h}^{(m)})^4}{(\gamma_{\min}\bar\omega_h^{(m)})^2}
\log(2mHM/\delta)
\right\rceil,
\end{equation*}
where \(\kappa_{W,h}^{(m)}\) and \(\bar\omega_h^{(m)}\) are the truncated witness-norm and
excitation constants. Then
\begin{equation*}
\begin{aligned}
\mathrm{Reg}(K)
\le\;&
H + c_{\mathrm{ep}}H N_{\mathrm{burn}}^{\mathrm{RKHS}}(m,\delta)
\\
&\quad +
\frac{1}{1-p}
\Bigl[
\widetilde O\!\left(H\sqrt{mK}\right)
+
\\
&\qquad
HK\bigl(
\epsilon_{\mathrm{proj}}(m)+\epsilon_{\mathrm{op,proj}}(m)
+\epsilon_{\mathrm{dual}}+\epsilon_{\mathrm{eq}}
\bigr)
+
K\epsilon_{\mathrm{appr}}+pHK
\Bigr],
\end{aligned}
\end{equation*}
where the \(\widetilde O(H\sqrt{mK})\) term hides the truncated RKHS envelopes and
self-normalized dual constants, while the projection terms are controlled by the spectral tail.
\end{theorem}

\subsection{Proof Sketch}

The confidence set retains a Bellman-consistent comparator, after which optimism and telescoping
under the worst-case-selector occupancy \(\mu_{h,\mathrm{wc}}^k\) give
\begin{equation}
\mathrm{Reg}(K)
\lesssim
\sum_{k=1}^K\sum_{h=1}^H
\left\|
f_h^k-\Tcal_h^\sigma f_{h+1}^k
\right\|_{1,\mu_{h,\mathrm{wc}}^k}
+
HK\epsilon_{\mathrm{eq}}.
\label{eq:regret_telescoping}
\end{equation}
Adding and subtracting the plug-in operator separates a residual (backward validation and
Bellman--Eluder counting) from an operator gap (dual concentration and occupancy transfer). The
game-specific proof fits the dual before each epoch, freezes and cross-validates its witnesses,
transfers to worst-case occupancies, and derives covariance growth from mixed response branches.

\section{Experiments}

We compare ODRB with \textsc{Sigma0} and a no-exploiter ablation on linear-mixture, tabular, and
support-shift games. The main \(d=8\) comparison uses ten paired seeds; details are in Appendix~\ref{sec:appendix_experiments}.

\begin{table}[t]
\centering
\small
\setlength{\tabcolsep}{4pt}
\caption{Main \texttt{linear\_mixture} results at \(d=8\) and \(5000\) episodes.}
\label{tab:main_linear_results}
\begin{tabular}{@{}lccc@{}}
\toprule
Method & Proxy & Deploy & Proxy gap \\
\midrule
\textsc{Sigma0} & \(0.230 \pm 0.088\) & \(0.230 \pm 0.088\) & \(0.0080 \pm 0.0047\) \\
\textsc{ODRB} \((\beta=0.30,\tau=0.02)\) & \(0.204 \pm 0.069\) & \(0.239 \pm 0.089\) & \(0.0072 \pm 0.0074\) \\
\textsc{ODRB} \((\beta=0.20,\tau=0.02)\) & \(0.162 \pm 0.057\) & \(0.236 \pm 0.089\) & \(0.0069 \pm 0.0066\) \\
\textsc{NoEx} \((\beta=0.30)\) & \(0.254 \pm 0.102\) & \(0.288 \pm 0.115\) & \(0.0644 \pm 0.0687\) \\
\bottomrule
\end{tabular}
\end{table}

For \textsc{ODRB}\((\beta=0.30,\tau=0.02)\) versus \textsc{Sigma0}, the deployment-return
difference is \(0.0094\) (paired-\(t\) 95\% CI \([0.0007,0.0182]\); exact sign-flip Holm
\(p=0.0469\)). For \((0.20,0.02)\), it is \(0.0065\), but CI \([-0.0017,0.0148]\) includes zero.
This supports a modest primary-setting improvement, not uniform superiority; the higher-return
\(\tau=0\) setting has a much larger gap.

\section{Conclusion}

ODRB gives conditional GFA interfaces and structured linear/RKHS guarantees for online robust
zero-sum games under shared-payoff, support, and coverage assumptions. Its signed comparator is not
an exploitability certificate, and the experiments are diagnostic. Player-specific ambiguity,
adaptive tuning, and equilibrium convergence remain open.

\begin{ack}
\textbf{Funding.} There was no funding in direct support of this work.\\
\textbf{Competing interests.} The authors declare no competing interests.
\end{ack}

\clearpage
\small
\bibliographystyle{unsrtnat}
\bibliography{neurips_2026}
\normalsize
\clearpage


\appendix

\section{Appendix Overview}
\label{sec:appendix_overview}

This appendix expands the proof sketch from the main text into a step-by-step derivation. For
clarity, the analysis is separated into two parts:
\begin{enumerate}[leftmargin=*]
\item a deterministic regret decomposition, which we prove in full;
\item two statistical verification blocks, which we prove under explicit backward-validation and
entropy assumptions.
\end{enumerate}
The deterministic part already makes precise how the robust minimax Bellman operator, the dual
plug-in backup, and the stage-game solver interact. The only missing ingredients are the
high-probability control needed to instantiate the residual budget and the operator-learning budget
for the exact algorithm and target function classes.

\paragraph{Appendix organization.}
Section~\ref{sec:appendix_deterministic} develops the deterministic reduction and its core lemmas.
Section~\ref{sec:appendix_statistical} collects the residual-counting, backward-validation,
optimism-closure, and dual-concentration ingredients. Section~\ref{sec:appendix_structured}
turns these inputs into closed structured specializations, including tabular, linear, pure-online
structured, projected, and RKHS-truncation results. Section~\ref{sec:appendix_future} briefly
summarizes the main future directions suggested by the present proof package. Finally,
Section~\ref{sec:appendix_experiments} records the benchmark definitions, compared methods,
evaluation metrics, statistical protocol, and additional result tables used by the theorem-aligned
benchmark suite in the main text.

\paragraph{Status of the formal results.}
The appendix contains two different levels of claims, and we state this distinction explicitly. On
the one hand,
Theorem~\ref{thm:tabular_closed_appendix},
Theorem~\ref{thm:linear_formal_appendix},
Theorem~\ref{thm:pureonline_linear_kernel},
Theorem~\ref{thm:structured_projected_appendix},
and Theorem~\ref{thm:rkhs_projected_appendix} are structured results whose statistical budgets are
verified inside the paper once their stated structural assumptions hold. In particular,
Theorem~\ref{thm:linear_formal_appendix} closes the epochwise/cross-fit linear route only after
Assumption~\ref{asmp:linear_ridge_dual} is imposed; Proposition~\ref{prop:affine_ridge_specialization}
provides a primitive sufficient condition for that hypothesis. On the other hand,
Theorem~\ref{thm:gfa_interpretable} and Theorem~\ref{thm:gfa_projected} remain conditional
broad-function-class interfaces: they record the regret consequences of explicit optimism,
residual-budget, operator-learning, and coverage hypotheses, but they do not claim that those
hypotheses have already been verified for the original one-update-per-episode backward-validation
general-GFA route. That unresolved point is isolated later in Proposition~\ref{prop:online_bv_gap}.
The paper therefore claims a structured theory package with explicit
general-function-approximation interfaces, rather than a fully assumption-free broad closure.

\paragraph{Main-text theorem map.}
For ease of cross-checking, the four main-text theorem statements are compressed from the following
appendix results: the General regret guarantee
\(
\textup{(Theorem~\ref{thm:main_regret})}
\)
from Corollary~\ref{cor:appendix_main}, the Projected function-approximation guarantee
\(
\textup{(Theorem~\ref{thm:gfa_projected_main})}
\)
from Corollary~\ref{cor:appendix_projected_main}, the Epochwise/cross-fit linear theorem
\(
\textup{(Theorem~\ref{thm:linear_case})}
\)
from Corollary~\ref{cor:linear_formal_burnin}, and the RKHS spectral-truncation projected theorem
\(
\textup{(Theorem~\ref{thm:rkhs_projected_main})}
\)
from Theorem~\ref{thm:rkhs_projected_appendix}. The appendix-level interface theorems
Theorem~\ref{thm:gfa_interpretable} and Theorem~\ref{thm:gfa_projected} provide the more explicit
residual-budget and operator-budget decompositions from which the two main-text interface
statements are obtained.

\paragraph{Nominal-game reduction and comparison.}
When \(\sigma=0\), the ambiguity set is the singleton nominal kernel, the dual robust backup
reduces to ordinary conditional expectation, and the robust residual class becomes the nominal
Markov-game Bellman-residual class. Table~\ref{tab:nominal_comparison} positions this reduction
relative to the closest non-robust function-approximation results; the comparison concerns the
problem structure and does not claim a better nominal rate.

\begin{table}[t]
\centering
\small
\setlength{\tabcolsep}{3pt}
\caption{Relationship to representative non-robust zero-sum Markov-game results.}
\label{tab:nominal_comparison}
\begin{tabular}{@{}p{0.19\linewidth}p{0.19\linewidth}p{0.22\linewidth}p{0.30\linewidth}@{}}
\toprule
Work & Transition model & Complexity / exploration & Distinguishing target \\
\midrule
\citet{huang2022towards} & Nominal & Minimax-Eluder complexity & Value regret under general function approximation \\
\citet{jin2022power} & Nominal & Multi-agent BE dimension and exploiter & Nash-policy learning in large state spaces \\
This paper, \(\sigma=0\) & Nominal & Robust MG-BE reduces to nominal counterpart & Same \(\sqrt K\)-type episode dependence, up to complexity, coverage, horizon, and logs \\
This paper, \(\sigma>0\) & Ambiguity set around nominal kernel & Response-branch coverage plus operator transfer & Candidate-dependent robust dual learned from nominal transitions and evaluated on worst-case occupancies \\
\bottomrule
\end{tabular}
\end{table}

\section{Deterministic Reduction Setup}
\label{sec:appendix_deterministic}

\subsection{Setup and Additional Standing Conditions}

For clarity, we analyze the on-policy core of the algorithm. That is, in episode $k$ the learner
computes $f^{(k)}\in\Fcal$, extracts a policy pair $(\pi^k,\nu^k)$ from $f^{(k)}$, and executes
$(\pi^k,\nu^k)$ for one episode. The exploiter-mix variant changes only the occupancy family that
appears in the statistical counting lemma; the algebra below is unchanged.

\paragraph{Standing selector hypothesis.}
Whenever the analysis refers to a worst-case executed occupancy, we assume that for every Markov
policy pair and stage $h$ the inner infimum in the robust Bellman backup admits a measurable
selector attaining the robust value. This makes the occupancy process
\(
\mu_h^k
\)
well defined. In finite / tabular settings this is automatic; more generally, one could replace
attainment by an $\epsilon_{\mathrm{sel}}$-optimal measurable selector at the cost of an
additional $O(HK\epsilon_{\mathrm{sel}})$ term in the final regret bounds.

Throughout the deterministic Bellman telescope,
\(
\mu_h^k
\)
denotes the stage-$h$ occupancy of the executed policy pair under a worst-case selector that
attains the robust value of that pair. For the statistical lemmas in the present on-policy core, we
idealize the stage-$h$ data batch as having state-action marginal
\(
\mu_h^k
\)
and next-state conditional law
\(
P_h^0
\).
Section~A.11 later separates these two roles: it keeps
\(
\widetilde\mu_h^k
\)
for the worst-case executed occupancy used by the regret telescope, introduces a distinct
data-collection occupancy for the empirical-process arguments, and adds an explicit transfer
assumption between them.

For each episode $k$ and stage $h$, define the exact robust residual and the operator error
\begin{align*}
\delta_h^k(s,a,b)
&:= f_h^{(k)}(s,a,b) - (\Tcal_h^\sigma f_{h+1}^{(k)})(s,a,b), \\
\widehat\delta_h^k(s,a,b)
&:= f_h^{(k)}(s,a,b) -
\big(\widehat{\Tcal}_{h,f^{(k)}}^{\sigma,(k)} f_{h+1}^{(k)}\big)(s,a,b), \\
e_h^k(s,a,b)
&:= \big(\widehat{\Tcal}_{h,f^{(k)}}^{\sigma,(k)} f_{h+1}^{(k)}\big)(s,a,b)
- (\Tcal_h^\sigma f_{h+1}^{(k)})(s,a,b),
\end{align*}
so that $\delta_h^k=\widehat\delta_h^k+e_h^k$ pointwise.

We will use the following three input conditions.

\begin{assumption}[Approximate optimism]
\label{asmp:opt_appendix}
For every episode $k$,
\begin{equation*}
V_1^{\star,\sigma}(s_1) \leq V_{f^{(k)},1}(s_1) + \epsilon_{\mathrm{appr}}.
\end{equation*}
\end{assumption}

\begin{assumption}[Operator-learning budget]
\label{asmp:op_appendix}
There exists a nonnegative array $\{\alpha_{h,k}\}$ such that for every $k$ and $h$,
\begin{equation*}
\bigl\|e_h^k\bigr\|_{1,\mu_h^k} \leq \alpha_{h,k} + \epsilon_{\mathrm{dual}},
\end{equation*}
where $\mu_h^k$ is the stage-$h$ occupancy measure induced by the executed policy pair in episode
$k$ under a worst-case kernel selector for that pair.
\end{assumption}

\begin{assumption}[Residual budget]
\label{asmp:res_appendix}
There exists $\Gamma_{\mathrm{res}}(K)\ge 0$ such that
\begin{equation*}
\sum_{k=1}^K \sum_{h=1}^H
\EE_{\mu_h^k}\sbr{\bigl(\widehat\delta_h^k(s,a,b)\bigr)^2}
\leq \Gamma_{\mathrm{res}}(K).
\end{equation*}
\end{assumption}

Assumptions~\ref{asmp:op_appendix} and \ref{asmp:res_appendix} are exactly the two places where
the robust Markov-game extension of the statistical machinery enters. The approximate-optimism
assumption is also not meant to be primitive: later in this appendix we show how it follows from
retaining a suitable reference function inside the backward-validation confidence set.

\subsection{Auxiliary Lemmas}

\begin{lemma}[Minimax value is $1$-Lipschitz]
\label{lem:minimax_lipschitz}
Let $M,N\in\RR^{|\Acal|\times|\Bcal|}$ be two payoff matrices and define
\(
\mathrm{val}(M):=\max_{p\in\Delta(\Acal)}\min_{q\in\Delta(\Bcal)} p^\top M q
\).
Then
\begin{equation*}
\bigl|\mathrm{val}(M)-\mathrm{val}(N)\bigr| \leq \|M-N\|_{\infty}.
\end{equation*}
\end{lemma}

\begin{proof}
Fix arbitrary mixed strategies $p\in\Delta(\Acal)$ and $q\in\Delta(\Bcal)$. Since $p,q$ are
probability vectors,
\begin{equation*}
\bigl|p^\top(M-N)q\bigr|
\leq \sum_{a,b} p(a)q(b)\,|M_{ab}-N_{ab}|
\leq \|M-N\|_\infty.
\end{equation*}
Therefore, for every $p$,
\begin{equation*}
\min_q p^\top M q
\leq \min_q p^\top N q + \|M-N\|_\infty.
\end{equation*}
Taking the maximum over $p$ yields
\begin{equation*}
\mathrm{val}(M) \leq \mathrm{val}(N)+\|M-N\|_\infty.
\end{equation*}
Interchanging the roles of $M$ and $N$ gives the reverse inequality.
\end{proof}

\begin{lemma}[Approximate stage-game extraction]
\label{lem:approx_saddle}
Suppose that for every $(h,s)$, the extracted policy pair $(\pi_h^k(s),\nu_h^k(s))$ is
$\epsilon_{\mathrm{eq}}$-optimal for the stage game induced by $f_h^{(k)}$, in the sense that
\begin{equation}
V_{f^{(k)},h}(s)
\leq
\EE_{a\sim\pi_h^k(\cdot\mid s),\,b\sim\nu_h^k(\cdot\mid s)}
\sbr{f_h^{(k)}(s,a,b)}
+ \epsilon_{\mathrm{eq}}.
\label{eq:approx_extract}
\end{equation}
Then the stage-game extraction error contributes at most $H\epsilon_{\mathrm{eq}}$ over one
episode.
\end{lemma}

\begin{proof}
Equation \eqref{eq:approx_extract} gives at most $\epsilon_{\mathrm{eq}}$ loss at each stage. Summing
over the $H$ stages of one episode gives the result.
\end{proof}

\begin{lemma}[One-step robust Bellman comparison]
\label{lem:one_step}
Fix an episode $k$ and abbreviate
$f=f^{(k)}$, $\pi=\pi^k$, and $\nu=\nu^k$.
Let $P_h^\dagger(\cdot\mid s,a,b)$ be a measurable selector that attains the robust value of the
executed policy pair $(\pi,\nu)$ at stage $h$. Define
\begin{equation*}
\Delta_h(s) := V_{f,h}(s) - V_h^{\pi,\nu,\sigma}(s).
\end{equation*}
Then for every state $s$ and stage $h$,
\begin{equation}
\Delta_h(s)
\leq
\epsilon_{\mathrm{eq}}
+
\EE\!\left[
\delta_h^k(s_h,a_h,b_h)+\Delta_{h+1}(s_{h+1})
\mid s_h=s
\right],
\label{eq:one_step_recursion}
\end{equation}
where the conditional expectation is taken under
$a_h\sim\pi_h(\cdot\mid s_h)$,
$b_h\sim\nu_h(\cdot\mid s_h)$,
and $s_{h+1}\sim P_h^\dagger(\cdot\mid s_h,a_h,b_h)$.
\end{lemma}

\begin{proof}
By Lemma~\ref{lem:approx_saddle},
\begin{equation*}
V_{f,h}(s)
\leq
\EE_{a,b}\sbr{f_h(s,a,b)} + \epsilon_{\mathrm{eq}}.
\end{equation*}
Subtracting the robust value of the executed policy gives
\begin{align*}
\Delta_h(s)
&\leq
\epsilon_{\mathrm{eq}}
+
\EE_{a,b}\sbr{
f_h(s,a,b)
- r_h(s,a,b)
- \EE_{s'\sim P_h^\dagger(\cdot\mid s,a,b)}
\sbr{V_{h+1}^{\pi,\nu,\sigma}(s')}
}.
\end{align*}
Add and subtract
\(
\EE_{s'\sim P_h^\dagger(\cdot\mid s,a,b)}\sbr{V_{f,h+1}(s')}
\)
inside the bracket:
\begin{align*}
\Delta_h(s)
&\leq
\epsilon_{\mathrm{eq}}
+
\EE_{a,b}\sbr{
f_h(s,a,b)-r_h(s,a,b)
-\EE_{s'\sim P_h^\dagger}\sbr{V_{f,h+1}(s')}
}
\\
&\qquad
+
\EE_{a,b,s'}\sbr{V_{f,h+1}(s')-V_{h+1}^{\pi,\nu,\sigma}(s')}.
\end{align*}
Because
\(
(\Tcal_h^\sigma f_{h+1})(s,a,b)
=r_h(s,a,b)+\inf_{P\in\Ucal_h(s,a,b)}\EE_{s'\sim P}\sbr{V_{f,h+1}(s')}
\),
we have
\begin{equation*}
r_h(s,a,b)+\inf_{P\in\Ucal_h(s,a,b)}\EE_{s'\sim P}\sbr{V_{f,h+1}(s')}
\leq
r_h(s,a,b)+\EE_{s'\sim P_h^\dagger}\sbr{V_{f,h+1}(s')}.
\end{equation*}
Hence
\begin{equation*}
f_h(s,a,b)-r_h(s,a,b)-\EE_{s'\sim P_h^\dagger}\sbr{V_{f,h+1}(s')}
\leq
f_h(s,a,b)-(\Tcal_h^\sigma f_{h+1})(s,a,b)
=
\delta_h^k(s,a,b).
\end{equation*}
Substituting this bound back proves \eqref{eq:one_step_recursion}.
\end{proof}

\begin{lemma}[Bellman telescoping inequality]
\label{lem:bellman_telescope}
Under the setting of Lemma~\ref{lem:one_step},
\begin{equation*}
V_{f,1}(s_1)-V_1^{\pi,\nu,\sigma}(s_1)
\leq
\sum_{h=1}^H \EE_{\mu_h^k}\sbr{\delta_h^k(s,a,b)}
+ H\epsilon_{\mathrm{eq}}.
\end{equation*}
\end{lemma}

\begin{proof}
Starting from \eqref{eq:one_step_recursion} at $h=1$ and taking expectations along the trajectory
under the executed policy pair and the worst-case selector yields
\begin{equation*}
\EE\sbr{\Delta_1(s_1)}
\leq
\epsilon_{\mathrm{eq}}
+
\EE_{\mu_1^k}\sbr{\delta_1^k}
+
\EE\sbr{\Delta_2(s_2)}.
\end{equation*}
Iterating the same argument for $h=2,\ldots,H$ and using $\Delta_{H+1}\equiv 0$ gives
\begin{equation*}
\EE\sbr{\Delta_1(s_1)}
\leq
\sum_{h=1}^H \EE_{\mu_h^k}\sbr{\delta_h^k}
+ H\epsilon_{\mathrm{eq}}.
\end{equation*}
Since $s_1$ is deterministic, $\EE[\Delta_1(s_1)]=V_{f,1}(s_1)-V_1^{\pi,\nu,\sigma}(s_1)$.
\end{proof}

\begin{lemma}[Residual and operator decomposition]
\label{lem:decomposition}
For every episode $k$,
\begin{equation*}
\sum_{h=1}^H \EE_{\mu_h^k}\sbr{\delta_h^k}
\leq
\sum_{h=1}^H \EE_{\mu_h^k}\sbr{\widehat\delta_h^k}
+
\sum_{h=1}^H \EE_{\mu_h^k}\sbr{|e_h^k|}.
\end{equation*}
\end{lemma}

\begin{proof}
The identity $\delta_h^k=\widehat\delta_h^k+e_h^k$ holds pointwise. Taking expectations under
$\mu_h^k$ and applying the triangle inequality to the operator-error term gives the claim.
\end{proof}

\subsection{Main Deterministic Reduction}

\begin{theorem}[Deterministic regret reduction]
\label{thm:deterministic_reduction}
Under Assumptions~\ref{asmp:opt_appendix}--\ref{asmp:res_appendix},
\begin{equation}
\mathrm{Reg}(K)
\leq
K\epsilon_{\mathrm{appr}}
+ HK\epsilon_{\mathrm{eq}}
+ \sqrt{KH\,\Gamma_{\mathrm{res}}(K)}
+ \sum_{k=1}^K\sum_{h=1}^H \alpha_{h,k}
+ HK\epsilon_{\mathrm{dual}}.
\label{eq:deterministic_bound}
\end{equation}
\end{theorem}

\begin{proof}
Fix an episode $k$. By Assumption~\ref{asmp:opt_appendix},
\begin{equation*}
V_1^{\star,\sigma}(s_1)-V_1^{\pi^k,\nu^k,\sigma}(s_1)
\leq
\epsilon_{\mathrm{appr}}
+
V_{f^{(k)},1}(s_1)-V_1^{\pi^k,\nu^k,\sigma}(s_1).
\end{equation*}
Applying Lemma~\ref{lem:bellman_telescope} and then Lemma~\ref{lem:decomposition} yields
\begin{equation*}
V_1^{\star,\sigma}(s_1)-V_1^{\pi^k,\nu^k,\sigma}(s_1)
\leq
\epsilon_{\mathrm{appr}}
+ H\epsilon_{\mathrm{eq}}
+
\sum_{h=1}^H \EE_{\mu_h^k}\sbr{\widehat\delta_h^k}
+
\sum_{h=1}^H \EE_{\mu_h^k}\sbr{|e_h^k|}.
\end{equation*}
Summing over $k=1,\ldots,K$ gives
\begin{equation}
\mathrm{Reg}(K)
\leq
K\epsilon_{\mathrm{appr}}
+ HK\epsilon_{\mathrm{eq}}
+
\sum_{k,h}\EE_{\mu_h^k}\sbr{\widehat\delta_h^k}
+
\sum_{k,h}\EE_{\mu_h^k}\sbr{|e_h^k|}.
\label{eq:two_term_bound}
\end{equation}

We now control the two sums in \eqref{eq:two_term_bound}. By Cauchy--Schwarz,
\begin{align*}
\sum_{k,h}\EE_{\mu_h^k}\sbr{\widehat\delta_h^k}
&\leq
\sum_{k,h}\sqrt{\EE_{\mu_h^k}\sbr{(\widehat\delta_h^k)^2}}
\leq
\sqrt{KH \sum_{k,h}\EE_{\mu_h^k}\sbr{(\widehat\delta_h^k)^2}}
\leq
\sqrt{KH\,\Gamma_{\mathrm{res}}(K)},
\end{align*}
where the last step uses Assumption~\ref{asmp:res_appendix}. For the operator error,
Assumption~\ref{asmp:op_appendix} gives
\begin{equation*}
\sum_{k,h}\EE_{\mu_h^k}\sbr{|e_h^k|}
\leq
\sum_{k,h}\bigl(\alpha_{h,k}+\epsilon_{\mathrm{dual}}\bigr)
=
\sum_{k,h}\alpha_{h,k}
+ HK\epsilon_{\mathrm{dual}}.
\end{equation*}
Substituting both bounds into \eqref{eq:two_term_bound} proves
\eqref{eq:deterministic_bound}.
\end{proof}

We will also need a projected variant of the same reduction. Fix arbitrary stagewise maps
\(
\mathsf P_h:L_\infty(\Scal\times\Acal\times\Bcal)\to L_\infty(\Scal\times\Acal\times\Bcal)
\).
For each episode
\(
k
\)
and stage
\(
h
\),
define the projected residual, projected plug-in residual, projected operator gap, and projection
bias by
\begin{align*}
\delta_{h,\mathsf P}^k
&:=
f_h^{(k)}-\mathsf P_h\Tcal_h^\sigma f_{h+1}^{(k)},
\\
\widehat\delta_{h,\mathsf P}^k
&:=
f_h^{(k)}-\mathsf P_h\widehat{\Tcal}_{h,f^{(k)}}^{\sigma,(k)} f_{h+1}^{(k)},
\\
e_{h,\mathsf P}^k
&:=
\mathsf P_h\widehat{\Tcal}_{h,f^{(k)}}^{\sigma,(k)} f_{h+1}^{(k)}
-
\mathsf P_h\Tcal_h^\sigma f_{h+1}^{(k)},
\\
b_{h,\mathsf P}^k
&:=
\mathsf P_h\Tcal_h^\sigma f_{h+1}^{(k)}
-
\Tcal_h^\sigma f_{h+1}^{(k)}.
\end{align*}
Then
\(
\delta_h^k=\delta_{h,\mathsf P}^k+b_{h,\mathsf P}^k
\)
and
\(
\delta_{h,\mathsf P}^k=\widehat\delta_{h,\mathsf P}^k+e_{h,\mathsf P}^k
\)
pointwise.

\begin{theorem}[Deterministic regret reduction with projected completeness bias]
\label{thm:deterministic_reduction_projected}
Fix stagewise maps
\(
\{\mathsf P_h\}_{h=1}^H
\).
Assume approximate optimism, and suppose there exists a residual budget
\(
\Gamma_{\mathrm{res}}^{\mathsf P}(K)\ge 0
\),
and a projection-bias constant
\(
\epsilon_{\mathrm{proj}}\ge 0
\)
such that
\begin{align}
\sum_{k=1}^K\sum_{h=1}^H
\EE_{\mu_h^k}\sbr{\bigl(\delta_{h,\mathsf P}^k\bigr)^2}
&\le
\Gamma_{\mathrm{res}}^{\mathsf P}(K),
\label{eq:projected_res_budget}
\\
\left\|b_{h,\mathsf P}^k\right\|_{1,\mu_h^k}
&\le
\epsilon_{\mathrm{proj}}
\qquad
\text{for every }(h,k).
\label{eq:projection_bias_budget}
\end{align}
Then
\begin{equation}
\mathrm{Reg}(K)
\leq
K\epsilon_{\mathrm{appr}}
+
HK\epsilon_{\mathrm{eq}}
+
\sqrt{KH\,\Gamma_{\mathrm{res}}^{\mathsf P}(K)}
+
HK\epsilon_{\mathrm{proj}}.
\label{eq:deterministic_projected_bound}
\end{equation}
\end{theorem}

\begin{proof}
Fix an episode
\(
k
\).
By Assumption~\ref{asmp:opt_appendix},
\begin{equation*}
V_1^{\star,\sigma}(s_1)-V_1^{\pi^k,\nu^k,\sigma}(s_1)
\le
\epsilon_{\mathrm{appr}}
+
V_{f^{(k)},1}(s_1)-V_1^{\pi^k,\nu^k,\sigma}(s_1).
\end{equation*}
Applying Lemma~\ref{lem:bellman_telescope} gives
\begin{equation*}
V_1^{\star,\sigma}(s_1)-V_1^{\pi^k,\nu^k,\sigma}(s_1)
\le
\epsilon_{\mathrm{appr}}
+
H\epsilon_{\mathrm{eq}}
+
\sum_{h=1}^H \EE_{\mu_h^k}\sbr{\delta_h^k}.
\end{equation*}
Using the pointwise decomposition
\(
\delta_h^k=\delta_{h,\mathsf P}^k+b_{h,\mathsf P}^k
\),
we obtain
\begin{equation*}
V_1^{\star,\sigma}(s_1)-V_1^{\pi^k,\nu^k,\sigma}(s_1)
\le
\epsilon_{\mathrm{appr}}
+
H\epsilon_{\mathrm{eq}}
+
+
\sum_{h=1}^H \EE_{\mu_h^k}\sbr{\delta_{h,\mathsf P}^k}
+
\sum_{h=1}^H \EE_{\mu_h^k}\sbr{|b_{h,\mathsf P}^k|}.
\end{equation*}
Summing over
\(
k=1,\ldots,K
\)
and applying Cauchy--Schwarz to the projected residual term yields
\begin{equation*}
\sum_{k=1}^K\sum_{h=1}^H \EE_{\mu_h^k}\sbr{\delta_{h,\mathsf P}^k}
\le
\sqrt{
KH
\sum_{k=1}^K\sum_{h=1}^H
\EE_{\mu_h^k}\sbr{\bigl(\delta_{h,\mathsf P}^k\bigr)^2}
}
\le
\sqrt{KH\,\Gamma_{\mathrm{res}}^{\mathsf P}(K)}
\end{equation*}
by \eqref{eq:projected_res_budget}. Equation \eqref{eq:projection_bias_budget} yields
\begin{equation*}
\sum_{k=1}^K\sum_{h=1}^H \EE_{\mu_h^k}\sbr{|b_{h,\mathsf P}^k|}
\le
HK\epsilon_{\mathrm{proj}}.
\end{equation*}
Combining the last two displays proves \eqref{eq:deterministic_projected_bound}.
\end{proof}

\begin{corollary}[Main-text regret bound from statistical inputs]
\label{cor:appendix_main}
Suppose in addition that there exist absolute constants $C_1,C_2>0$ such that
\begin{align*}
\Gamma_{\mathrm{res}}(K)
&\leq
C_1 H B_\phi^2(\sigma)\,
\dim_{\mathrm{BE},2}^{\mathrm{rob\text{-}MG}}(\Fcal,\Pi,1/\sqrt{K})\,
\log(KH/\delta), \\
\sum_{k=1}^K\sum_{h=1}^H \alpha_{h,k}
&\leq
C_2 H B_\phi(\sigma)\sqrt{K}\,\log(KH/\delta).
\end{align*}
Then Theorem~\ref{thm:main_regret} follows up to absolute constants and logarithmic factors.
\end{corollary}

\begin{proof}
Plug the two displayed bounds into Theorem~\ref{thm:deterministic_reduction}. The residual term
becomes
\begin{equation*}
\sqrt{KH\,\Gamma_{\mathrm{res}}(K)}
\leq
\widetilde O\!\left(
H B_\phi(\sigma)\sqrt{
\dim_{\mathrm{BE},2}^{\mathrm{rob\text{-}MG}}(\Fcal,\Pi,1/\sqrt{K})\,K
}
\right).
\end{equation*}
The operator-learning term
\(
\sum_{k,h}\alpha_{h,k}
\)
is of order $\widetilde O(HB_\phi(\sigma)\sqrt K)$, which is absorbed by the previous term because
\(
\dim_{\mathrm{BE},2}^{\mathrm{rob\text{-}MG}}(\Fcal,\Pi,1/\sqrt{K})\ge 1
\)
in any nontrivial problem. The remaining terms match the main-text statement.
\end{proof}

\paragraph{Iterate sample-complexity consequence.}
Under the assumptions of Corollary~\ref{cor:appendix_main}, there exists at least one iterate
$k\in[K]$ such that
\begin{equation*}
V_1^{\star,\sigma}(s_1) - V_1^{\pi^k,\nu^k,\sigma}(s_1)
\leq
\frac{\mathrm{Reg}(K)}{K}.
\end{equation*}
Consequently, if
\(
\epsilon_{\mathrm{appr}} + H\epsilon_{\mathrm{eq}} + H\epsilon_{\mathrm{dual}}
\)
is at most order $\epsilon$, then an iterate whose one-sided comparator difference is at most
order $\epsilon$ is guaranteed once
\begin{equation*}
K
=
\widetilde O\!\left(
\frac{H^2 B_\phi^2(\sigma)
\dim_{\mathrm{BE},2}^{\mathrm{rob\text{-}MG}}(\Fcal,\Pi,1/\sqrt{K})}
\epsilon^2
\right).
\end{equation*}
If every iterate had comparator difference strictly larger than $\mathrm{Reg}(K)/K$, then summing over all
$K$ iterates would contradict the definition of regret; the displayed sample-complexity condition is
obtained by requiring each term in Theorem~\ref{thm:main_regret} divided by $K$ to be at most order
$\epsilon$.

\section{Residual Complexity and Statistical Verification}
\label{sec:appendix_statistical}

\subsection{Finite-State Fail-State Support-Control Argument}

We now prove a concrete finite-state support-preservation statement. The proof only uses the
failure-state geometry and does not depend on Bellman--Eluder arguments.

\begin{proposition}[Support-preserving optimizer under TV balls]
\label{prop:appendix_failstate}
Assume a finite state space and fix a tuple $(h,s,a,b)$. Let $P^0\in\Delta(\Scal)$ be the nominal
next-state distribution and let $s_F\in\mathrm{supp}(P^0)$ be a failure state satisfying
$V(s_F)=\min_{x\in\Scal}V(x)=0$. Consider the optimization problem
\begin{equation*}
\inf_{P\in\Delta(\Scal):\, D_{\mathrm{TV}}(P,P^0)\leq \sigma} \EE_{x\sim P}\sbr{V(x)}.
\end{equation*}
Then there exists an optimal solution $P^\star$ such that
\(
\mathrm{supp}(P^\star)\subseteq \mathrm{supp}(P^0)
\).
\end{proposition}

\begin{proof}
Let $Q$ be any feasible distribution. If
\(
\mathrm{supp}(Q)\subseteq \mathrm{supp}(P^0)
\),
there is nothing to prove. Otherwise define the outside-support mass
\begin{equation*}
m := \sum_{x\notin \mathrm{supp}(P^0)} Q(x) > 0.
\end{equation*}
Construct a new distribution $\widetilde Q$ by moving all outside-support mass to the failure
state:
\begin{equation*}
\widetilde Q(x) :=
\begin{cases}
0, & x\notin \mathrm{supp}(P^0), \\
Q(s_F)+m, & x=s_F, \\
Q(x), & \text{otherwise}.
\end{cases}
\end{equation*}
Clearly $\widetilde Q\in\Delta(\Scal)$ and
\(
\mathrm{supp}(\widetilde Q)\subseteq \mathrm{supp}(P^0)
\).

We first show that the objective does not increase. Since $V(s_F)=0$ and $V(x)\ge 0$ for all $x$,
\begin{align*}
\EE_{\widetilde Q}\sbr{V}
&=
\sum_{x\in \mathrm{supp}(P^0)} \widetilde Q(x)V(x)
=
\sum_{x\in \mathrm{supp}(P^0)} Q(x)V(x) + mV(s_F) \\
&=
\sum_{x\in \mathrm{supp}(P^0)} Q(x)V(x)
\leq
\sum_{x\in\Scal} Q(x)V(x)
=
\EE_Q\sbr{V}.
\end{align*}

We next show that the TV distance does not increase. Since $P^0(x)=0$ for
$x\notin \mathrm{supp}(P^0)$,
\begin{align*}
\|\widetilde Q-P^0\|_1
&=
|\widetilde Q(s_F)-P^0(s_F)|
+
\sum_{\substack{x\in \mathrm{supp}(P^0)\\x\neq s_F}}
|Q(x)-P^0(x)| \\
&=
|Q(s_F)+m-P^0(s_F)|
+
\sum_{\substack{x\in \mathrm{supp}(P^0)\\x\neq s_F}}
|Q(x)-P^0(x)| \\
&\leq
|Q(s_F)-P^0(s_F)| + m
+
\sum_{\substack{x\in \mathrm{supp}(P^0)\\x\neq s_F}}
|Q(x)-P^0(x)| \\
&=
\sum_{x\in \mathrm{supp}(P^0)} |Q(x)-P^0(x)|
+
\sum_{x\notin \mathrm{supp}(P^0)} |Q(x)-P^0(x)| \\
&=
\|Q-P^0\|_1.
\end{align*}
Therefore
\(
D_{\mathrm{TV}}(\widetilde Q,P^0)\leq D_{\mathrm{TV}}(Q,P^0)\leq \sigma
\),
so $\widetilde Q$ is feasible and has no larger objective value. Starting from an optimal $Q$ and
applying this transformation yields an optimal solution supported inside
$\mathrm{supp}(P^0)$.
\end{proof}

\begin{corollary}
Under the assumptions of Proposition~\ref{prop:appendix_failstate}, the robust Bellman target for
the failure-state normalized value function can be realized without leaving the nominal support.
\end{corollary}

\begin{proof}
Apply Proposition~\ref{prop:appendix_failstate} at each state-action tuple to the continuation
value $V=V_{f,h+1}$. The proposition yields an optimal robust kernel whose support is contained in
that of the nominal kernel. Hence the robust Bellman target remains an in-support expectation.
\end{proof}

\subsection{A Robust MG-BE Counting Lemma}

We now formulate a stage-wise counting lemma that captures the exploration term. Because the
main text only uses an informal notion of distributional Eluder dimension, we state here a
square-loss version tailored to the residual analysis.

\begin{definition}[Square-loss distributional independence]
\label{def:de2}
Let $\Pscr$ be a class of real-valued functions on $\Scal\times\Acal\times\Bcal$, and let $\Pi_h$
be a family of probability measures on the same domain. For $\epsilon>0$, we say that a measure
$\mu\in\Pi_h$ is \emph{$\epsilon$-dependent} on
$\{\mu^1,\ldots,\mu^m\}\subseteq \Pi_h$ with respect to $\Pscr$ if
\begin{equation*}
\EE_{\mu}\sbr{\psi^2}
\leq
\epsilon^2 + \sum_{i=1}^m \EE_{\mu^i}\sbr{\psi^2}
\qquad \text{for all } \psi\in\Pscr.
\end{equation*}
Otherwise, $\mu$ is \emph{$\epsilon$-independent}. The square-loss distributional Eluder
dimension, denoted by $\dim_{\mathrm{DE},2}(\Pscr,\Pi_h,\epsilon)$, is the length of the longest
sequence in $\Pi_h$ whose every element is $\epsilon$-independent of its predecessors.
\end{definition}

\begin{definition}[Square-loss robust MG-BE dimension]
\label{def:robust_be2}
For the exact robust residual class
\(
\Pscr_h := (I-\Tcal_h^\sigma)\Fcal
\),
define
\begin{equation*}
\dim_{\mathrm{BE},2}^{\mathrm{rob\text{-}MG}}(\Fcal,\Pi,\epsilon)
:=
\max_{h\in[H]}
\dim_{\mathrm{DE},2}(\Pscr_h,\Pi_h,\epsilon).
\end{equation*}
\end{definition}

\begin{lemma}[Robust MG-BE counting lemma]
\label{lem:counting}
Fix a stage $h$ and let $\Pscr_h=(I-\Tcal_h^\sigma)\Fcal$. Suppose we are given a sequence
\(
\{(\mu_h^k,\psi_h^k)\}_{k=1}^K
\)
such that $\mu_h^k\in\Pi_h$, $\psi_h^k\in\Pscr_h$, and for every $k$,
\begin{align}
\sum_{j<k} \EE_{\mu_h^j}\sbr{(\psi_h^k)^2} &\leq \epsilon^2,
\label{eq:count_prev_small}
\\
\EE_{\mu_h^k}\sbr{(\psi_h^k)^2} &> 2\epsilon^2.
\label{eq:count_curr_large}
\end{align}
Then
\begin{equation*}
K \leq \dim_{\mathrm{DE},2}(\Pscr_h,\Pi_h,\epsilon).
\end{equation*}
\end{lemma}

\begin{proof}
Fix $k\in[K]$. By \eqref{eq:count_prev_small} and \eqref{eq:count_curr_large},
\begin{equation*}
\EE_{\mu_h^k}\sbr{(\psi_h^k)^2}
>
\epsilon^2 + \sum_{j<k}\EE_{\mu_h^j}\sbr{(\psi_h^k)^2}.
\end{equation*}
Hence the witness $\psi_h^k\in\Pscr_h$ certifies that $\mu_h^k$ is $\epsilon$-independent of
$\{\mu_h^1,\ldots,\mu_h^{k-1}\}$ in the sense of
Definition~\ref{def:de2}. Since this holds for every $k$, the sequence
$\{\mu_h^1,\ldots,\mu_h^K\}$ is an $\epsilon$-independent sequence in $\Pi_h$. By the definition
of $\dim_{\mathrm{DE},2}(\Pscr_h,\Pi_h,\epsilon)$, its length cannot exceed that dimension.
\end{proof}

\begin{corollary}[Residual-budget corollary under backward validation]
\label{cor:counting_budget}
Fix a stage $h$ and suppose there exists $B_{\mathrm{res},h}<\infty$ such that
\(
|\psi_h^k(s,a,b)|\leq B_{\mathrm{res},h}
\)
for all $(s,a,b)$, all $k$, and all $\psi_h^k\in\Pscr_h$. Assume moreover that the sequence
$\{(\mu_h^k,\psi_h^k)\}_{k=1}^K$ satisfies the backward-validation condition
\eqref{eq:count_prev_small}. Then
\begin{equation*}
\sum_{k=1}^K \EE_{\mu_h^k}\sbr{(\psi_h^k)^2}
\leq
2K\epsilon^2
+
\dim_{\mathrm{DE},2}(\Pscr_h,\Pi_h,\epsilon)\, B_{\mathrm{res},h}^2.
\end{equation*}
Consequently,
\begin{equation*}
\sum_{h=1}^H \sum_{k=1}^K \EE_{\mu_h^k}\sbr{(\psi_h^k)^2}
\leq
2HK\epsilon^2
+
\sum_{h=1}^H
\dim_{\mathrm{DE},2}(\Pscr_h,\Pi_h,\epsilon)\, B_{\mathrm{res},h}^2.
\end{equation*}
\end{corollary}

\begin{proof}
For a fixed stage $h$, split the episode indices into
\begin{align*}
\Lcal_h &:= \cbr{k\in[K]: \EE_{\mu_h^k}\sbr{(\psi_h^k)^2}>2\epsilon^2}, \\
\Scal_h &:= [K]\setminus \Lcal_h.
\end{align*}
By Lemma~\ref{lem:counting}, we have
\(
|\Lcal_h|\leq \dim_{\mathrm{DE},2}(\Pscr_h,\Pi_h,\epsilon)
\).
Therefore,
\begin{align*}
\sum_{k=1}^K \EE_{\mu_h^k}\sbr{(\psi_h^k)^2}
&=
\sum_{k\in\Scal_h} \EE_{\mu_h^k}\sbr{(\psi_h^k)^2}
+
\sum_{k\in\Lcal_h} \EE_{\mu_h^k}\sbr{(\psi_h^k)^2} \\
&\leq
2K\epsilon^2 + |\Lcal_h| B_{\mathrm{res},h}^2 \\
&\leq
2K\epsilon^2
+
\dim_{\mathrm{DE},2}(\Pscr_h,\Pi_h,\epsilon)\, B_{\mathrm{res},h}^2.
\end{align*}
Summing this inequality over $h\in[H]$ gives the second claim.
\end{proof}

Corollary~\ref{cor:counting_budget} reduces the residual-budget input
Assumption~\ref{asmp:res_appendix} to a backward-validation condition on the exact robust
residual witnesses. We now verify the two bridge terms
\eqref{eq:bv_dev} and \eqref{eq:bv_transfer} directly for the current backward-validation update.

\subsection{Verifying the Backward-Validation Bridge for the Current Update}

For a fixed stage $h$, define the plug-in residual witness class
\begin{equation*}
\Wcal_h
:=
\left\{
w_{h,f,g} \,:\,
\begin{aligned}
&f\in\Fcal,\quad g\in\Gcal_h, \\
&w_{h,f,g}(s,a,b)
=
f_h(s,a,b)-r_h(s,a,b) \\
&\qquad\qquad
+ \EE_{s'\sim P_h^0(\cdot\mid s,a,b)}
\sbr{\ell_\phi(f_{h+1};s,a,b,s';g)}
\end{aligned}
\right\}.
\end{equation*}
By construction, for every $k\geq 2$ and every $f\in\Fcal$,
\begin{equation*}
\widehat\delta_{h,k}(f)
=
w_{h,f,g_{h,f}^{(k)}}
\in
\Wcal_h.
\end{equation*}

\begin{assumption}[Bounded plug-in residual witnesses]
\label{asmp:bounded_witness}
For the fixed stage $h$, there exists $B_{W,h}<\infty$ such that
\begin{equation*}
\sup_{w\in\Wcal_h}\|w\|_\infty \leq B_{W,h}.
\end{equation*}
\end{assumption}

\begin{assumption}[Finite witness covering number]
\label{asmp:cover_witness}
For the chosen witness scale $\eta_W>0$,
\begin{equation*}
\mathcal N_W(\eta_W)
:=
\mathcal N\bigl(\Wcal_h,\|\cdot\|_\infty,\eta_W\bigr)
<\infty.
\end{equation*}
\end{assumption}

\begin{proposition}[Witness covering from stagewise value and dual covers]
\label{prop:witness_cover_from_value_dual}
For a fixed stage
\(
h
\),
define the stagewise sup-norm covering number
\begin{equation*}
\mathcal N_{F,h}^\infty(\eta_h)
:=
\mathcal N\bigl(\Fcal_h,\|\cdot\|_\infty,\eta_h\bigr).
\end{equation*}
Assume
\(
\mathcal N_{F,h}^\infty(\eta_h)<\infty
\)
and Assumptions~\ref{asmp:lipschitz_dual} and \ref{asmp:cover_dual} hold. Then
\begin{equation}
\mathcal N\!\left(
\Wcal_h,\|\cdot\|_\infty,\eta_h+L_{\ell,h}(\eta_F+\eta_G)
\right)
\leq
\mathcal N_{F,h}^\infty(\eta_h)\,
\mathcal N_F(\eta_F)\,
\mathcal N_G(\eta_G).
\label{eq:witness_cover_reduce}
\end{equation}
In particular, Assumption~\ref{asmp:cover_witness} follows once the stagewise value slice
\(
\Fcal_h
\),
the continuation slice
\(
\Fcal_{h+1}
\),
and the dual class
\(
\Gcal_h
\)
all admit finite covers.
\end{proposition}

\begin{proof}
Let
\(
\{u^r\}_{r=1}^{M_h}\subseteq \Fcal_h
\)
be an
\(
\eta_h
\)-net under
\(
\|\cdot\|_\infty
\),
let
\(
\{f^m\}_{m=1}^{M_F}\subseteq \Fcal_{h+1}
\)
be an
\(
\eta_F
\)-net under
\(
d_{\Fcal,h+1}
\),
and let
\(
\{g^\ell\}_{\ell=1}^{M_G}\subseteq \Gcal_h
\)
be an
\(
\eta_G
\)-net under
\(
d_{\Gcal,h}
\).

Fix any witness
\(
w_{h,f,g}\in\Wcal_h
\).
Choose
\(
\bar u\in\{u^r\}_{r=1}^{M_h}
\),
\(
\bar f\in\{f^m\}_{m=1}^{M_F}
\),
and
\(
\bar g\in\{g^\ell\}_{\ell=1}^{M_G}
\)
such that
\begin{equation*}
\|f_h-\bar u\|_\infty\leq \eta_h,
\qquad
d_{\Fcal,h+1}(f_{h+1},\bar f)\leq \eta_F,
\qquad
d_{\Gcal,h}(g,\bar g)\leq \eta_G.
\end{equation*}
Define
\begin{equation*}
\bar w(s,a,b)
:=
\bar u(s,a,b)-r_h(s,a,b)
+
\EE_{s'\sim P_h^0(\cdot\mid s,a,b)}
\sbr{
\ell_\phi(\bar f;s,a,b,s';\bar g)
}.
\end{equation*}
By the stagewise product interpretation
\(
\Fcal=\{\Fcal_t\}_{t=1}^H
\)
from the preliminaries, the pair
\(
(\bar u,\bar f)
\)
can be completed to a candidate function sequence in
\(
\Fcal
\),
so
\(
\bar w\in\Wcal_h
\).

For every
\(
(s,a,b)
\),
\begin{align*}
\bigl|
w_{h,f,g}(s,a,b)-\bar w(s,a,b)
\bigr|
&\leq
\bigl|f_h(s,a,b)-\bar u(s,a,b)\bigr|
\\
&\quad
+
\left|
\EE_{s'}\sbr{
\ell_\phi(f_{h+1};s,a,b,s';g)
-
\ell_\phi(\bar f;s,a,b,s';\bar g)
}
\right|
\\
&\leq
\eta_h + L_{\ell,h}(\eta_F+\eta_G),
\end{align*}
where the last step uses Assumption~\ref{asmp:lipschitz_dual} and Jensen's inequality. Hence the
set of all such
\(
\bar w
\)
forms an
\(
\eta_h+L_{\ell,h}(\eta_F+\eta_G)
\)-net for
\(
\Wcal_h
\)
of size at most
\(
M_hM_FM_G
\),
which proves \eqref{eq:witness_cover_reduce}.
\end{proof}

For $k\geq 2$, write
\begin{equation*}
n_{h,j}:=|\Dcal_h^{(j)}|,
\qquad
S_{h,k}:=\sum_{j=1}^{k-1}\frac{1}{n_{h,j}},
\qquad
N_{h,<k}:=\sum_{j=1}^{k-1} n_{h,j},
\qquad
n_{h,\min}^{(k)}:=\min_{1\leq j\leq k-1} n_{h,j}.
\end{equation*}
We also keep the shorthand
\begin{align*}
\mathrm{BV}_{h,k}^{\mathrm{plug}}(f)
&:=
\sum_{j=1}^{k-1}
\EE_{\mu_h^j}\sbr{\bigl(\widehat\delta_{h,k}(f)\bigr)^2},
\\
\mathrm{BV}_{h,k}^{\mathrm{exact}}(f)
&:=
\sum_{j=1}^{k-1}
\EE_{\mu_h^j}\sbr{\bigl(\delta_h(f)\bigr)^2},
\\
\widehat{\mathrm{BV}}_{h,k}(f)
&:=
\sum_{j=1}^{k-1}
\frac{1}{n_{h,j}}
\sum_{(s,a,b,r,s')\in\Dcal_h^{(j)}}
\sbr{
\widehat\delta_{h,k}(f;s,a,b)
}^2.
\end{align*}

\begin{lemma}[Uniform backward-validation deviation]
\label{lem:bv_dev}
Under Assumptions~\ref{asmp:bounded_witness} and \ref{asmp:cover_witness}, for every
$k\geq 2$ and every $\delta\in(0,1)$, with probability at least $1-\delta$,
\begin{equation}
\sup_{f\in\Fcal}
\bigl|
\widehat{\mathrm{BV}}_{h,k}(f)-\mathrm{BV}_{h,k}^{\mathrm{plug}}(f)
\bigr|
\leq
\xi_{h,k}(\delta;\eta_W),
\label{eq:bv_dev}
\end{equation}
where
\begin{equation*}
\xi_{h,k}(\delta;\eta_W)
:=
B_{W,h}^2
\sqrt{
2S_{h,k}\log\!\bigl(2\mathcal N_W(\eta_W)/\delta\bigr)
}
+
2B_{W,h}(k-1)\eta_W.
\end{equation*}
\end{lemma}

\begin{proof}
Let
\(
\{w^1,\ldots,w^M\}\subseteq \Wcal_h
\)
be an $\eta_W$-net under the sup norm, where
\(
M=\mathcal N_W(\eta_W)
\).
Fix one net witness $w^m$. For every batch $j<k$ and sample index $i\in[n_{h,j}]$, define
\begin{equation*}
Y_{j,i}^{(m)}
:=
\frac{
\bigl(w^m(s_{j,i},a_{j,i},b_{j,i})\bigr)^2
- \EE_{\mu_h^j}\sbr{(w^m)^2}
}{n_{h,j}}.
\end{equation*}
Conditioned on the history before the $(j,i)$-th stage-$h$ sample is revealed,
\(
Y_{j,i}^{(m)}
\)
has conditional mean zero, because that sample is drawn from the stage-$h$ occupancy law
$\mu_h^j$. By Assumption~\ref{asmp:bounded_witness},
\begin{equation*}
\bigl|Y_{j,i}^{(m)}\bigr|
\leq
\frac{B_{W,h}^2}{n_{h,j}}.
\end{equation*}
Hence Azuma--Hoeffding for bounded martingale differences implies
\begin{equation*}
\PP\!\left(
\left|
\sum_{j=1}^{k-1}\sum_{i=1}^{n_{h,j}} Y_{j,i}^{(m)}
\right| > t
\right)
\leq
2\exp\!\left(
-\frac{t^2}{2B_{W,h}^4 S_{h,k}}
\right).
\end{equation*}
Applying a union bound over the $M$ net witnesses gives that with probability at least $1-\delta$,
\begin{equation}
\max_{m\in[M]}
\left|
\sum_{j=1}^{k-1}\frac{1}{n_{h,j}}
\sum_{i=1}^{n_{h,j}}
\bigl(w^m(s_{j,i},a_{j,i},b_{j,i})\bigr)^2
-
\sum_{j=1}^{k-1}\EE_{\mu_h^j}\sbr{(w^m)^2}
\right|
\leq
B_{W,h}^2
\sqrt{
2S_{h,k}\log\!\bigl(2M/\delta\bigr)
}.
\label{eq:bv_net_dev}
\end{equation}

Now fix an arbitrary witness $w\in\Wcal_h$, and choose a net point $\bar w$ such that
\(
\|w-\bar w\|_\infty\leq \eta_W
\).
Since
\(
|w+\bar w|\leq 2B_{W,h}
\),
we have
\begin{equation*}
\|w^2-\bar w^2\|_\infty
\leq
\|w-\bar w\|_\infty \|w+\bar w\|_\infty
\leq
2B_{W,h}\eta_W.
\end{equation*}
Therefore
\begin{align*}
\left|
\sum_{j=1}^{k-1}\frac{1}{n_{h,j}}
\sum_{i=1}^{n_{h,j}} w(s_{j,i},a_{j,i},b_{j,i})^2
-
\sum_{j=1}^{k-1}\frac{1}{n_{h,j}}
\sum_{i=1}^{n_{h,j}} \bar w(s_{j,i},a_{j,i},b_{j,i})^2
\right|
&\leq
2B_{W,h}(k-1)\eta_W,
\\
\left|
\sum_{j=1}^{k-1}\EE_{\mu_h^j}\sbr{w^2}
-
\sum_{j=1}^{k-1}\EE_{\mu_h^j}\sbr{\bar w^2}
\right|
&\leq
2B_{W,h}(k-1)\eta_W.
\end{align*}
Combining these two inequalities with \eqref{eq:bv_net_dev} yields
\begin{equation*}
\sup_{w\in\Wcal_h}
\left|
\sum_{j=1}^{k-1}\frac{1}{n_{h,j}}
\sum_{i=1}^{n_{h,j}} w(s_{j,i},a_{j,i},b_{j,i})^2
-
\sum_{j=1}^{k-1}\EE_{\mu_h^j}\sbr{w^2}
\right|
\leq
\xi_{h,k}(\delta;\eta_W).
\end{equation*}
Since every selected witness $\widehat\delta_{h,k}(f)$ belongs to $\Wcal_h$, this proves
\eqref{eq:bv_dev}.
\end{proof}

We next introduce the dual setup shared by the pooled backward-validation analysis and the
on-policy operator-learning analysis. For a generic probability measure $\nu$ on
$\Scal\times\Acal\times\Bcal$, define
\begin{equation*}
\Lcal_h^\nu(g;f)
:=
\EE_{(s,a,b)\sim \nu}
\EE_{s'\sim P_h^0(\cdot\mid s,a,b)}
\sbr{
\ell_\phi(f_{h+1};s,a,b,s';g)
}.
\end{equation*}
We also define an ambient dual class $\bar{\Gcal}_h$ large enough that the exact robust Bellman
operator admits the pointwise representation
\begin{equation}
(\Tcal_h^\sigma f_{h+1})(s,a,b)
=
r_h(s,a,b)
-
\inf_{\bar g\in\bar{\Gcal}_h}
\EE_{s'\sim P_h^0(\cdot\mid s,a,b)}
\sbr{\ell_\phi(f_{h+1};s,a,b,s';\bar g)}.
\label{eq:ambient_dual_rep}
\end{equation}

\begin{assumption}[Bounded dual loss]
\label{asmp:bounded_dual}
For the fixed stage $h$,
\begin{equation*}
\bigl|\ell_\phi(f_{h+1};s,a,b,s';g)\bigr| \leq B_\phi(\sigma)
\end{equation*}
for every $(f,g)\in\Fcal_{h+1}\times\Gcal_h$ and every sample $(s,a,b,s')$.
\end{assumption}

\begin{assumption}[Pointwise Lipschitz dual loss]
\label{asmp:lipschitz_dual}
There exist metrics $d_{\Fcal,h+1}$ on $\Fcal_{h+1}$ and $d_{\Gcal,h}$ on $\Gcal_h$, together
with a constant $L_{\ell,h}\geq 0$, such that for every
\(
f,\tilde f\in\Fcal_{h+1}
\),
\(
g,\tilde g\in\Gcal_h
\),
and every sample $(s,a,b,s')$,
\begin{equation*}
\bigl|
\ell_\phi(f_{h+1};s,a,b,s';g)
- \ell_\phi(\tilde f_{h+1};s,a,b,s';\tilde g)
\bigr|
\leq
L_{\ell,h}\Bigl(
d_{\Fcal,h+1}(f,\tilde f) + d_{\Gcal,h}(g,\tilde g)
\Bigr).
\end{equation*}
\end{assumption}

\begin{assumption}[Finite covering numbers at the working scales]
\label{asmp:cover_dual}
For the chosen scales $\eta_F,\eta_G>0$,
\begin{align*}
\mathcal N_F(\eta_F)
&:=
\mathcal N\bigl(\Fcal_{h+1},d_{\Fcal,h+1},\eta_F\bigr)
< \infty, \\
\mathcal N_G(\eta_G)
&:=
\mathcal N\bigl(\Gcal_h,d_{\Gcal,h},\eta_G\bigr)
< \infty.
\end{align*}
\end{assumption}

\begin{assumption}[Dual approximation error]
\label{asmp:dual_gap_appendix}
For every $f\in\Fcal_{h+1}$ and every probability measure $\nu$ on
$\Scal\times\Acal\times\Bcal$,
\begin{equation*}
\inf_{g\in\Gcal_h}\Lcal_h^\nu(g;f)
\leq
\inf_{\bar g\in\bar{\Gcal}_h}\Lcal_h^\nu(\bar g;f)
+
\epsilon_{\mathrm{dual}}.
\end{equation*}
\end{assumption}

We next close the operator-transfer term generated by the same pooled dual refit used by the
algorithm. Define the weighted past-occupancy mixture
\begin{equation*}
\bar\mu_{h,k}
:=
\sum_{j=1}^{k-1} \lambda_{h,j}^{(k)} \mu_h^j,
\qquad
\lambda_{h,j}^{(k)}:=\frac{n_{h,j}}{N_{h,<k}}.
\end{equation*}
For any $(f,g)\in\Fcal_{h+1}\times\Gcal_h$, define the pooled population dual loss
\begin{align*}
\bar{\Lcal}_{h,k}(g;f)
&:=
\sum_{j=1}^{k-1}\lambda_{h,j}^{(k)}
\Lcal_h^{\mu_h^j}(g;f)
\\
&=
\EE_{(s,a,b)\sim\bar\mu_{h,k}}
\EE_{s'\sim P_h^0(\cdot\mid s,a,b)}
\sbr{
\ell_\phi(f_{h+1};s,a,b,s';g)
}.
\end{align*}
The corresponding empirical pooled loss is
\begin{equation*}
\widehat{\Lcal}_{h,k}^{\mathrm{pool}}(g;f)
:=
\frac{1}{N_{h,<k}}
\sum_{j=1}^{k-1}
\sum_{(s,a,b,s')\in\Dcal_h^{(j)}}
\ell_\phi(f_{h+1};s,a,b,s';g).
\end{equation*}
Let
\(
g_{h,f}^{(k)}\in\argmin_{g\in\Gcal_h}\widehat{\Lcal}_{h,k}^{\mathrm{pool}}(g;f)
\)
denote the pooled dual fit used by the current algorithm.

\begin{lemma}[Pooled dual concentration for the current update]
\label{lem:pooled_dual_uniform}
Under Assumptions~\ref{asmp:bounded_dual}--\ref{asmp:cover_dual}, for every $k\geq 2$ and every
$\delta\in(0,1)$, with probability at least $1-\delta$,
\begin{equation}
\sup_{f\in\Fcal_{h+1},\, g\in\Gcal_h}
\bigl|
\widehat{\Lcal}_{h,k}^{\mathrm{pool}}(g;f) - \bar{\Lcal}_{h,k}(g;f)
\bigr|
\leq
\bar\Delta_{h,k}^{\mathrm{cov}}(\delta;\eta_F,\eta_G),
\label{eq:pooled_dual_dev}
\end{equation}
where
\begin{equation*}
\bar\Delta_{h,k}^{\mathrm{cov}}(\delta;\eta_F,\eta_G)
:=
B_\phi(\sigma)\sqrt{
\frac{2\log\!\bigl(2\mathcal N_F(\eta_F)\mathcal N_G(\eta_G)/\delta\bigr)}
{N_{h,<k}}
}
+
2L_{\ell,h}(\eta_F+\eta_G).
\end{equation*}
\end{lemma}

\begin{proof}
Let
\(
\{f^1,\ldots,f^{M_F}\}\subseteq \Fcal_{h+1}
\)
be an $\eta_F$-net and
\(
\{g^1,\ldots,g^{M_G}\}\subseteq \Gcal_h
\)
be an $\eta_G$-net, where
\(
M_F=\mathcal N_F(\eta_F)
\)
and
\(
M_G=\mathcal N_G(\eta_G)
\).

Fix a net pair $(f^r,g^m)$. Flatten the pooled stage-$h$ samples
\(
\Dcal_h^{(1)},\ldots,\Dcal_h^{(k-1)}
\)
into an arbitrary sequence indexed by $t\in[N_{h,<k}]$. For the $t$-th sample, define
\begin{equation*}
X_t^{(r,m)}
:=
\ell_\phi(f_{h+1}^r; s_t,a_t,b_t,s_t';g^m)
-
\EE\!\left[
\ell_\phi(f_{h+1}^r; s_t,a_t,b_t,s_t';g^m)
\mid \mathcal F_{t-1}
\right],
\end{equation*}
where $\mathcal F_{t-1}$ is the history before revealing that sample. By construction,
\(
\{X_t^{(r,m)}\}_{t=1}^{N_{h,<k}}
\)
is a martingale-difference sequence. Assumption~\ref{asmp:bounded_dual} gives
\(
|X_t^{(r,m)}|\leq 2B_\phi(\sigma)
\),
and the conditional mean term equals the corresponding summand of $\bar{\Lcal}_{h,k}(g^m;f^r)$.
Azuma--Hoeffding therefore yields
\begin{equation*}
\PP\!\left(
\bigl|
\widehat{\Lcal}_{h,k}^{\mathrm{pool}}(g^m;f^r)-\bar{\Lcal}_{h,k}(g^m;f^r)
\bigr| > t
\right)
\leq
2\exp\!\left(
-\frac{N_{h,<k} t^2}{2B_\phi^2(\sigma)}
\right).
\end{equation*}
Applying a union bound over all $M_FM_G$ net pairs gives that with probability at least $1-\delta$,
\begin{equation}
\max_{r\in[M_F],\,m\in[M_G]}
\bigl|
\widehat{\Lcal}_{h,k}^{\mathrm{pool}}(g^m;f^r)-\bar{\Lcal}_{h,k}(g^m;f^r)
\bigr|
\leq
B_\phi(\sigma)\sqrt{
\frac{2\log\!\bigl(2M_FM_G/\delta\bigr)}
{N_{h,<k}}
}.
\label{eq:pooled_dual_net_dev}
\end{equation}

The Lipschitz extension from the finite net to the full class is standard. Specifically, for
arbitrary $(f,g)$ and nearest net points
$(\bar f,\bar g)$, Assumption~\ref{asmp:lipschitz_dual} implies
\begin{align*}
\bigl|
\widehat{\Lcal}_{h,k}^{\mathrm{pool}}(g;f)-\widehat{\Lcal}_{h,k}^{\mathrm{pool}}(\bar g;\bar f)
\bigr|
&\leq
L_{\ell,h}(\eta_F+\eta_G),
\\
\bigl|
\bar{\Lcal}_{h,k}(g;f)-\bar{\Lcal}_{h,k}(\bar g;\bar f)
\bigr|
&\leq
L_{\ell,h}(\eta_F+\eta_G).
\end{align*}
Combining these two inequalities with \eqref{eq:pooled_dual_net_dev} proves
\eqref{eq:pooled_dual_dev}.
\end{proof}

\begin{corollary}[Pooled ERM excess risk]
\label{cor:pooled_erm_excess}
Under Assumptions~\ref{asmp:bounded_dual}--\ref{asmp:cover_dual}, with probability at least
$1-\delta$, simultaneously for all $f\in\Fcal_{h+1}$,
\begin{equation}
\bar{\Lcal}_{h,k}(g_{h,f}^{(k)};f)
-
\inf_{g\in\Gcal_h}\bar{\Lcal}_{h,k}(g;f)
\leq
2\bar\Delta_{h,k}^{\mathrm{cov}}(\delta;\eta_F,\eta_G).
\label{eq:pooled_erm_excess}
\end{equation}
\end{corollary}

\begin{proof}
Fix $f\in\Fcal_{h+1}$ and $\tau>0$. Choose
\(
g_{f,\tau}^\star\in\Gcal_h
\)
such that
\begin{equation*}
\bar{\Lcal}_{h,k}(g_{f,\tau}^\star;f)
\leq
\inf_{g\in\Gcal_h}\bar{\Lcal}_{h,k}(g;f)+\tau.
\end{equation*}
On the event \eqref{eq:pooled_dual_dev},
\begin{align*}
\bar{\Lcal}_{h,k}(g_{h,f}^{(k)};f)
&\leq
\widehat{\Lcal}_{h,k}^{\mathrm{pool}}(g_{h,f}^{(k)};f)
+ \bar\Delta_{h,k}^{\mathrm{cov}}(\delta;\eta_F,\eta_G) \\
&\leq
\widehat{\Lcal}_{h,k}^{\mathrm{pool}}(g_{f,\tau}^\star;f)
+ \bar\Delta_{h,k}^{\mathrm{cov}}(\delta;\eta_F,\eta_G) \\
&\leq
\bar{\Lcal}_{h,k}(g_{f,\tau}^\star;f)
+ 2\bar\Delta_{h,k}^{\mathrm{cov}}(\delta;\eta_F,\eta_G) \\
&\leq
\inf_{g\in\Gcal_h}\bar{\Lcal}_{h,k}(g;f)
+ 2\bar\Delta_{h,k}^{\mathrm{cov}}(\delta;\eta_F,\eta_G)
+ \tau.
\end{align*}
Letting $\tau\downarrow 0$ proves \eqref{eq:pooled_erm_excess}.
\end{proof}

\begin{corollary}[Backward-validation operator transfer]
\label{cor:bv_transfer}
Under Assumptions~\ref{asmp:bounded_dual}--\ref{asmp:dual_gap_appendix}, with probability at least
$1-\delta$, simultaneously for all $f\in\Fcal_{h+1}$,
\begin{equation}
\sum_{j=1}^{k-1}
\EE_{\mu_h^j}
\sbr{
\bigl(
\widehat{\Tcal}_{h,f}^{\sigma,(k)} f_{h+1}
- \Tcal_h^\sigma f_{h+1}
\bigr)^2
}
\leq
\rho_{h,k}(\delta;\eta_F,\eta_G),
\label{eq:bv_transfer}
\end{equation}
where
\begin{equation*}
\rho_{h,k}(\delta;\eta_F,\eta_G)
:=
\frac{2B_\phi(\sigma)\,N_{h,<k}}{n_{h,\min}^{(k)}}
\Bigl(
2\bar\Delta_{h,k}^{\mathrm{cov}}(\delta;\eta_F,\eta_G)
+ \epsilon_{\mathrm{dual}}
\Bigr).
\end{equation*}
\end{corollary}

\begin{proof}
Fix $f\in\Fcal_{h+1}$ and define the pointwise operator gap
\begin{equation*}
e_{h,k}(f)
:=
\widehat{\Tcal}_{h,f}^{\sigma,(k)} f_{h+1}
- \Tcal_h^\sigma f_{h+1}.
\end{equation*}
By the ambient dual representation \eqref{eq:ambient_dual_rep}, the function $-e_{h,k}(f)$ is
pointwise nonnegative. Assumption~\ref{asmp:bounded_dual} further implies
\begin{equation}
-2B_\phi(\sigma)\leq e_{h,k}(f;s,a,b)\leq 0
\qquad\text{for all }(s,a,b).
\label{eq:pointwise_gap_bound}
\end{equation}
Therefore
\begin{align}
\sum_{j=1}^{k-1}\EE_{\mu_h^j}\sbr{e_{h,k}(f)^2}
&=
\sum_{j=1}^{k-1}
\frac{N_{h,<k}}{n_{h,j}}
\lambda_{h,j}^{(k)}
\EE_{\mu_h^j}\sbr{e_{h,k}(f)^2} \\
&\leq
\frac{N_{h,<k}}{n_{h,\min}^{(k)}}
\EE_{\bar\mu_{h,k}}\sbr{e_{h,k}(f)^2} \\
&\leq
\frac{2B_\phi(\sigma)\,N_{h,<k}}{n_{h,\min}^{(k)}}
\EE_{\bar\mu_{h,k}}\sbr{-e_{h,k}(f)},
\label{eq:e2_to_e1}
\end{align}
where the last step uses \eqref{eq:pointwise_gap_bound}.

Next,
\begin{equation*}
\EE_{\bar\mu_{h,k}}\sbr{-e_{h,k}(f)}
=
\bar{\Lcal}_{h,k}(g_{h,f}^{(k)};f)
-
\inf_{\bar g\in\bar{\Gcal}_h}\bar{\Lcal}_{h,k}(\bar g;f).
\end{equation*}
Adding and subtracting
\(
\inf_{g\in\Gcal_h}\bar{\Lcal}_{h,k}(g;f)
\)
and then applying Corollary~\ref{cor:pooled_erm_excess} together with
Assumption~\ref{asmp:dual_gap_appendix} yields
\begin{equation*}
\EE_{\bar\mu_{h,k}}\sbr{-e_{h,k}(f)}
\leq
2\bar\Delta_{h,k}^{\mathrm{cov}}(\delta;\eta_F,\eta_G)
+ \epsilon_{\mathrm{dual}}.
\end{equation*}
Substituting this bound into \eqref{eq:e2_to_e1} proves \eqref{eq:bv_transfer}.
\end{proof}

\begin{proposition}[Backward validation induced by the current confidence set]
\label{prop:backward_validation}
Fix a stage $h$ and an episode $k\geq 2$. On the intersection of the events from
Lemma~\ref{lem:bv_dev} and Corollary~\ref{cor:bv_transfer}, if the confidence set update satisfies
\(
\widehat{\mathrm{BV}}_{h,k}(f)\leq \beta_{h,k}
\)
for every $f\in\mathfrak F^{(k-1)}$, then every $f\in\mathfrak F^{(k-1)}$ obeys
\begin{equation}
\mathrm{BV}_{h,k}^{\mathrm{exact}}(f)
\leq
2\beta_{h,k}
+ 2\xi_{h,k}(\delta;\eta_W)
+ 2\rho_{h,k}(\delta;\eta_F,\eta_G).
\label{eq:bv_exact_bound}
\end{equation}
\end{proposition}

\begin{proof}
Fix any $f\in\mathfrak F^{(k-1)}$. The identity
\(
\delta_h(f)=\widehat\delta_{h,k}(f)+e_{h,k}(f)
\)
holds pointwise, so
\begin{align*}
\mathrm{BV}_{h,k}^{\mathrm{exact}}(f)
&=
\sum_{j=1}^{k-1}
\EE_{\mu_h^j}\sbr{\bigl(\widehat\delta_{h,k}(f)+e_{h,k}(f)\bigr)^2}
\\
&\leq
2\sum_{j=1}^{k-1}\EE_{\mu_h^j}\sbr{\bigl(\widehat\delta_{h,k}(f)\bigr)^2}
+
2\sum_{j=1}^{k-1}\EE_{\mu_h^j}\sbr{\bigl(e_{h,k}(f)\bigr)^2}
\\
&=
2\,\mathrm{BV}_{h,k}^{\mathrm{plug}}(f)
+
2\sum_{j=1}^{k-1}\EE_{\mu_h^j}\sbr{\bigl(e_{h,k}(f)\bigr)^2}.
\end{align*}
On the event of Corollary~\ref{cor:bv_transfer}, the second term is at most
\(
2\rho_{h,k}(\delta;\eta_F,\eta_G)
\).
On the event of Lemma~\ref{lem:bv_dev},
\begin{equation*}
\mathrm{BV}_{h,k}^{\mathrm{plug}}(f)
\leq
\widehat{\mathrm{BV}}_{h,k}(f)+\xi_{h,k}(\delta;\eta_W)
\leq
\beta_{h,k}+\xi_{h,k}(\delta;\eta_W),
\end{equation*}
where the last step uses the defining property of $\mathfrak F^{(k-1)}$. Substituting both bounds
proves \eqref{eq:bv_exact_bound}.
\end{proof}

In particular, if the selected function $f^{(k)}$ satisfies
\(
2\beta_{h,k}
+ 2\xi_{h,k}(\delta;\eta_W)
+ 2\rho_{h,k}(\delta;\eta_F,\eta_G)
\leq \epsilon^2
\),
then the exact witness
\(
\psi_h^k:=\delta_h(f^{(k)})
\)
satisfies the backward-validation premise \eqref{eq:count_prev_small} needed by
Lemma~\ref{lem:counting}.

\subsection{Closing Optimism from the Confidence Set}

\begin{assumption}[Reference realizability for optimism closure]
\label{asmp:ref_opt}
There exists a reference function $f^\dagger\in\Fcal$ such that
\begin{equation}
V_1^{\star,\sigma}(s_1)
\leq
V_{f^\dagger,1}(s_1)+\epsilon_{\mathrm{appr}}
\label{eq:ref_gap}
\end{equation}
and, for every stage $h$,
\begin{equation}
f_h^\dagger = \Tcal_h^\sigma f_{h+1}^\dagger.
\label{eq:ref_fixed_point}
\end{equation}
\end{assumption}

\begin{proposition}[Exact robust realizability implies an optimism reference]
\label{prop:exact_realizability_ref}
Suppose there exists
\(
f^\star\in\Fcal
\)
such that for every stage
\(
h\in[H]
\),
\begin{equation*}
f_h^\star=\Tcal_h^\sigma f_{h+1}^\star
\qquad\text{and}\qquad
V_{f^\star,1}(s_1)=V_1^{\star,\sigma}(s_1).
\end{equation*}
Then Assumption~\ref{asmp:ref_opt} holds with
\(
f^\dagger=f^\star
\)
and
\(
\epsilon_{\mathrm{appr}}=0
\).
\end{proposition}

\begin{proof}
Take
\(
f^\dagger:=f^\star
\).
The fixed-point identity is exactly \eqref{eq:ref_fixed_point}, and the value identity implies
\eqref{eq:ref_gap} with
\(
\epsilon_{\mathrm{appr}}=0
\).
\end{proof}

\begin{proposition}[Reference retention]
\label{prop:ref_retention}
Under Assumption~\ref{asmp:ref_opt}, on the intersection of the events from
Lemma~\ref{lem:bv_dev} and Corollary~\ref{cor:bv_transfer}, if
\begin{equation}
\beta_{h,k}
\geq
\xi_{h,k}(\delta;\eta_W)
+ \rho_{h,k}(\delta;\eta_F,\eta_G)
\qquad \text{for every } h\in[H],
\label{eq:retain_threshold}
\end{equation}
then
\(
f^\dagger\in\mathfrak F^{(k-1)}
\).
\end{proposition}

\begin{proof}
Fix a stage $h$. By \eqref{eq:ref_fixed_point}, the exact residual of $f^\dagger$ is zero:
\(
\delta_h(f^\dagger)\equiv 0
\).
Hence
\(
\widehat\delta_{h,k}(f^\dagger) = -e_{h,k}(f^\dagger)
\)
pointwise, and therefore
\begin{equation*}
\mathrm{BV}_{h,k}^{\mathrm{plug}}(f^\dagger)
=
\sum_{j=1}^{k-1}\EE_{\mu_h^j}\sbr{\bigl(e_{h,k}(f^\dagger)\bigr)^2}
\leq
\rho_{h,k}(\delta;\eta_F,\eta_G)
\end{equation*}
on the event of Corollary~\ref{cor:bv_transfer}. Applying Lemma~\ref{lem:bv_dev} then gives
\begin{equation*}
\widehat{\mathrm{BV}}_{h,k}(f^\dagger)
\leq
\mathrm{BV}_{h,k}^{\mathrm{plug}}(f^\dagger)
+ \xi_{h,k}(\delta;\eta_W)
\leq
\rho_{h,k}(\delta;\eta_F,\eta_G)
+ \xi_{h,k}(\delta;\eta_W).
\end{equation*}
By \eqref{eq:retain_threshold}, the right-hand side is at most $\beta_{h,k}$. Since this holds for
every stage $h$, the definition of $\mathfrak F^{(k-1)}$ implies
\(
f^\dagger\in\mathfrak F^{(k-1)}
\).
\end{proof}

\begin{corollary}[Approximate optimism from confidence sets]
\label{cor:opt_from_conf}
Under Assumption~\ref{asmp:ref_opt} and the conditions of Proposition~\ref{prop:ref_retention},
the algorithm's optimistic choice satisfies
\begin{equation*}
V_1^{\star,\sigma}(s_1)
\leq
V_{f^{(k)},1}(s_1)+\epsilon_{\mathrm{appr}}.
\end{equation*}
In particular, Assumption~\ref{asmp:opt_appendix} is implied by the backward-validation confidence
set construction.
\end{corollary}

\begin{proof}
By Proposition~\ref{prop:ref_retention},
\(
f^\dagger\in\mathfrak F^{(k-1)}
\).
Since the algorithm chooses
\(
f^{(k)}\in\argmax_{f\in\mathfrak F^{(k-1)}} V_{f,1}(s_1)
\),
we have
\begin{equation*}
V_{f^{(k)},1}(s_1)\geq V_{f^\dagger,1}(s_1).
\end{equation*}
Combining this inequality with \eqref{eq:ref_gap} proves the claim.
\end{proof}

\begin{assumption}[Projected reference realizability for optimism closure]
\label{asmp:ref_opt_projected}
Fix stagewise maps
\(
\{\mathsf P_h\}_{h=1}^H
\).
There exists a reference function
\(
f^\dagger\in\Fcal
\)
such that
\begin{equation}
V_1^{\star,\sigma}(s_1)
\le
V_{f^\dagger,1}(s_1)+\epsilon_{\mathrm{appr}}
\label{eq:ref_gap_projected}
\end{equation}
and, for every stage
\(
h\in[H]
\),
\begin{equation}
f_h^\dagger
=
\mathsf P_h\Tcal_h^\sigma f_{h+1}^\dagger.
\label{eq:ref_fixed_point_projected}
\end{equation}
\end{assumption}

\begin{proposition}[Projected reference retention]
\label{prop:ref_retention_projected}
Fix stagewise maps
\(
\{\mathsf P_h\}_{h=1}^H
\)
and let
\begin{equation}
\mathfrak F_{\mathsf P}^{(k-1)}
:=
\cbr{
f\in\Fcal:
\widehat{\mathrm{BV}}_{h,k}^{\mathsf P}(f)\le \beta_{h,k}^{\mathsf P}
\text{ for every }h\in[H]
}.
\label{eq:projected_conf_set}
\end{equation}
For any
\(
f\in\Fcal
\),
define the projected plug-in residual
\begin{equation*}
\widehat\delta_{h,\mathsf P,k}(f)
:=
f_h-\mathsf P_h\widehat{\Tcal}_{h,f}^{\sigma,(k)}f_{h+1},
\end{equation*}
the projected operator error
\begin{equation*}
e_{h,\mathsf P,k}(f)
:=
\mathsf P_h\widehat{\Tcal}_{h,f}^{\sigma,(k)}f_{h+1}
-
\mathsf P_h\Tcal_h^\sigma f_{h+1},
\end{equation*}
and the projected plug-in backward-validation score
\begin{equation*}
\mathrm{BV}_{h,k}^{\mathsf P,\mathrm{plug}}(f)
:=
\sum_{j=1}^{k-1}
\EE_{\widetilde\mu_h^j}
\sbr{
\bigl(\widehat\delta_{h,\mathsf P,k}(f)\bigr)^2
}.
\end{equation*}
Under Assumption~\ref{asmp:ref_opt_projected}, if for every stage
\(
h\in[H]
\),
\begin{align}
\bigl|
\widehat{\mathrm{BV}}_{h,k}^{\mathsf P}(f^\dagger)
-
\mathrm{BV}_{h,k}^{\mathsf P,\mathrm{plug}}(f^\dagger)
\bigr|
&\le
\xi_{h,k}^{\mathsf P},
\label{eq:projected_bv_dev_ref}
\\
\sum_{j=1}^{k-1}
\EE_{\widetilde\mu_h^j}
\sbr{
\bigl(e_{h,\mathsf P,k}(f^\dagger)\bigr)^2
}
&\le
\rho_{h,k}^{\mathsf P},
\label{eq:projected_bv_transfer_ref}
\\
\beta_{h,k}^{\mathsf P}
&\ge
\xi_{h,k}^{\mathsf P}+\rho_{h,k}^{\mathsf P},
\end{align}
then
\(
f^\dagger\in\mathfrak F_{\mathsf P}^{(k-1)}
\).
\end{proposition}

\begin{proof}
Fix a stage
\(
h\in[H]
\).
By \eqref{eq:ref_fixed_point_projected}, the projected exact residual of
\(
f^\dagger
\)
is zero. Therefore
\begin{equation*}
\widehat\delta_{h,\mathsf P,k}(f^\dagger)
=
-e_{h,\mathsf P,k}(f^\dagger)
\end{equation*}
pointwise, which implies
\begin{equation*}
\mathrm{BV}_{h,k}^{\mathsf P,\mathrm{plug}}(f^\dagger)
=
\sum_{j=1}^{k-1}
\EE_{\widetilde\mu_h^j}
\sbr{
\bigl(e_{h,\mathsf P,k}(f^\dagger)\bigr)^2
}
\le
\rho_{h,k}^{\mathsf P}
\end{equation*}
by \eqref{eq:projected_bv_transfer_ref}. Combining this with
\eqref{eq:projected_bv_dev_ref} gives
\begin{equation*}
\widehat{\mathrm{BV}}_{h,k}^{\mathsf P}(f^\dagger)
\le
\mathrm{BV}_{h,k}^{\mathsf P,\mathrm{plug}}(f^\dagger)
+
\xi_{h,k}^{\mathsf P}
\le
\rho_{h,k}^{\mathsf P}+\xi_{h,k}^{\mathsf P}
\le
\beta_{h,k}^{\mathsf P}.
\end{equation*}
Since this holds for every stage
\(
h
\),
definition \eqref{eq:projected_conf_set} implies
\(
f^\dagger\in\mathfrak F_{\mathsf P}^{(k-1)}
\).
\end{proof}

\begin{corollary}[Projected approximate optimism from confidence sets]
\label{cor:opt_from_conf_projected}
Under Assumption~\ref{asmp:ref_opt_projected} and the conditions of
Proposition~\ref{prop:ref_retention_projected}, if the algorithm chooses
\begin{equation*}
f^{(k)}
\in
\argmax_{f\in\mathfrak F_{\mathsf P}^{(k-1)}} V_{f,1}(s_1),
\end{equation*}
then
\begin{equation*}
V_1^{\star,\sigma}(s_1)
\le
V_{f^{(k)},1}(s_1)+\epsilon_{\mathrm{appr}}.
\end{equation*}
In particular, Assumption~\ref{asmp:opt_appendix} holds for the projected confidence-set route.
\end{corollary}

\begin{proof}
By Proposition~\ref{prop:ref_retention_projected},
\(
f^\dagger\in\mathfrak F_{\mathsf P}^{(k-1)}
\).
Hence the optimistic choice over
\(
\mathfrak F_{\mathsf P}^{(k-1)}
\)
satisfies
\(
V_{f^{(k)},1}(s_1)\ge V_{f^\dagger,1}(s_1)
\).
Combining this with \eqref{eq:ref_gap_projected} proves the claim.
\end{proof}

\subsection{A Covering-Number Dual ERM Concentration Lemma}

We now return to a fresh-sample, current-occupancy model to verify the on-policy operator-learning
budget Assumption~\ref{asmp:op_appendix}. Fix a stage $h$ and an occupancy measure $\mu_h^k$.
Given i.i.d.\ samples
\(
\{(s_i,a_i,b_i,s_i')\}_{i=1}^{N_{h,k}}
\)
from the joint law
\(
\mu_h^k\otimes P_h^0
\),
define
\begin{equation*}
\widehat{\Lcal}_{h,k}(g;f)
:=
\frac{1}{N_{h,k}}
\sum_{i=1}^{N_{h,k}}
\ell_\phi(f_{h+1};s_i,a_i,b_i,s_i';g),
\end{equation*}
and let
\(
\hat g_{h,f}^{(k)}\in\argmin_{g\in\Gcal_h}\widehat{\Lcal}_{h,k}(g;f)
\).
The fitted plug-in operator is
\begin{equation*}
\bigl(\widehat{\Tcal}_{h,f}^{\sigma,(k)} f_{h+1}\bigr)(s,a,b)
:=
r_h(s,a,b)
-
\EE_{s'\sim P_h^0(\cdot\mid s,a,b)}
\sbr{\ell_\phi(f_{h+1};s,a,b,s';\hat g_{h,f}^{(k)})}.
\end{equation*}

\begin{lemma}[Uniform dual concentration via covering numbers]
\label{lem:dual_uniform}
Under Assumptions~\ref{asmp:bounded_dual}--\ref{asmp:cover_dual}, for every $\delta\in(0,1)$,
with probability at least $1-\delta$,
\begin{equation}
\sup_{f\in\Fcal_{h+1},\, g\in\Gcal_h}
\bigl|
\widehat{\Lcal}_{h,k}(g;f) - \Lcal_h^\mu(g;f)
\bigr|
\leq
\Delta_{h,k}^{\mathrm{cov}}(\delta;\eta_F,\eta_G),
\label{eq:uniform_dual_dev}
\end{equation}
where
\begin{equation*}
\Delta_{h,k}^{\mathrm{cov}}(\delta;\eta_F,\eta_G)
:=
B_\phi(\sigma)\sqrt{
\frac{2\log\!\bigl(2\mathcal N_F(\eta_F)\mathcal N_G(\eta_G)/\delta\bigr)}
{N_{h,k}}
}
+
2L_{\ell,h}(\eta_F+\eta_G).
\end{equation*}
\end{lemma}

\begin{proof}
Let
\(
\{f^1,\ldots,f^{M_F}\}\subseteq\Fcal_{h+1}
\)
be an $\eta_F$-net under $d_{\Fcal,h+1}$ and
\(
\{g^1,\ldots,g^{M_G}\}\subseteq\Gcal_h
\)
be an $\eta_G$-net under $d_{\Gcal,h}$, where
\(
M_F=\mathcal N_F(\eta_F)
\)
and
\(
M_G=\mathcal N_G(\eta_G)
\).

Fix a net pair $(f^r,g^m)$. By Assumption~\ref{asmp:bounded_dual}, the random variables
\(
\ell_\phi(f_{h+1}^r;s_i,a_i,b_i,s_i';g^m)
\)
are i.i.d.\ and take values in
\(
[-B_\phi(\sigma),B_\phi(\sigma)]
\).
Hoeffding's inequality therefore yields, for every $t>0$,
\begin{equation*}
\PP\!\left(
\bigl|
\widehat{\Lcal}_{h,k}(g^m;f^r)-\Lcal_h^\mu(g^m;f^r)
\bigr| > t
\right)
\leq
2\exp\!\left(
-\frac{N_{h,k} t^2}{2B_\phi^2(\sigma)}
\right).
\end{equation*}
Applying a union bound over all $M_FM_G$ net pairs, we obtain that with probability at least
$1-\delta$,
\begin{equation}
\max_{r\in[M_F],\,m\in[M_G]}
\bigl|
\widehat{\Lcal}_{h,k}(g^m;f^r)-\Lcal_h^\mu(g^m;f^r)
\bigr|
\leq
B_\phi(\sigma)\sqrt{
\frac{2\log\!\bigl(2M_FM_G/\delta\bigr)}
{N_{h,k}}
}.
\label{eq:net_dual_dev}
\end{equation}

Now fix arbitrary $(f,g)\in\Fcal_{h+1}\times\Gcal_h$, and choose net points
\(
\bar f,\bar g
\)
such that
\(
d_{\Fcal,h+1}(f,\bar f)\leq \eta_F
\)
and
\(
d_{\Gcal,h}(g,\bar g)\leq \eta_G
\).
By Assumption~\ref{asmp:lipschitz_dual}, for every sample $(s,a,b,s')$,
\begin{equation*}
\bigl|
\ell_\phi(f_{h+1};s,a,b,s';g)
- \ell_\phi(\bar f_{h+1};s,a,b,s';\bar g)
\bigr|
\leq
L_{\ell,h}(\eta_F+\eta_G).
\end{equation*}
Averaging this inequality over the empirical sample and over the population law separately gives
\begin{align*}
\bigl|
\widehat{\Lcal}_{h,k}(g;f)-\widehat{\Lcal}_{h,k}(\bar g;\bar f)
\bigr|
&\leq
L_{\ell,h}(\eta_F+\eta_G),
\\
\bigl|
\Lcal_h^\mu(g;f)-\Lcal_h^\mu(\bar g;\bar f)
\bigr|
&\leq
L_{\ell,h}(\eta_F+\eta_G).
\end{align*}
Combining these two inequalities with \eqref{eq:net_dual_dev} yields
\begin{align*}
\bigl|
\widehat{\Lcal}_{h,k}(g;f)-\Lcal_h^\mu(g;f)
\bigr|
&\leq
\bigl|
\widehat{\Lcal}_{h,k}(g;f)-\widehat{\Lcal}_{h,k}(\bar g;\bar f)
\bigr|
+
\bigl|
\widehat{\Lcal}_{h,k}(\bar g;\bar f)-\Lcal_h^\mu(\bar g;\bar f)
\bigr|
\\
&\qquad
+
\bigl|
\Lcal_h^\mu(\bar g;\bar f)-\Lcal_h^\mu(g;f)
\bigr|
\\
&\leq
B_\phi(\sigma)\sqrt{
\frac{2\log\!\bigl(2M_FM_G/\delta\bigr)}
{N_{h,k}}
}
+
2L_{\ell,h}(\eta_F+\eta_G).
\end{align*}
Since $(f,g)$ was arbitrary, this proves \eqref{eq:uniform_dual_dev}.
\end{proof}

\begin{corollary}[ERM excess population risk]
\label{cor:erm_excess}
Under Assumptions~\ref{asmp:bounded_dual}--\ref{asmp:cover_dual}, with probability at least
$1-\delta$, simultaneously for all $f\in\Fcal_{h+1}$,
\begin{equation}
\Lcal_h^\mu(\hat g_{h,f}^{(k)};f)
-
\inf_{g\in\Gcal_h}\Lcal_h^\mu(g;f)
\leq
2\Delta_{h,k}^{\mathrm{cov}}(\delta;\eta_F,\eta_G).
\label{eq:erm_excess}
\end{equation}
\end{corollary}

\begin{proof}
Fix $f\in\Fcal_{h+1}$ and $\tau>0$. Choose
\(
g_{f,\tau}^\star\in\Gcal_h
\)
such that
\begin{equation*}
\Lcal_h^\mu(g_{f,\tau}^\star;f)
\leq
\inf_{g\in\Gcal_h}\Lcal_h^\mu(g;f) + \tau.
\end{equation*}
On the event \eqref{eq:uniform_dual_dev},
\begin{align*}
\Lcal_h^\mu(\hat g_{h,f}^{(k)};f)
&\leq
\widehat{\Lcal}_{h,k}(\hat g_{h,f}^{(k)};f)
+ \Delta_{h,k}^{\mathrm{cov}}(\delta;\eta_F,\eta_G) \\
&\leq
\widehat{\Lcal}_{h,k}(g_{f,\tau}^\star;f)
+ \Delta_{h,k}^{\mathrm{cov}}(\delta;\eta_F,\eta_G) \\
&\leq
\Lcal_h^\mu(g_{f,\tau}^\star;f)
+ 2\Delta_{h,k}^{\mathrm{cov}}(\delta;\eta_F,\eta_G) \\
&\leq
\inf_{g\in\Gcal_h}\Lcal_h^\mu(g;f)
+ 2\Delta_{h,k}^{\mathrm{cov}}(\delta;\eta_F,\eta_G)
+ \tau.
\end{align*}
Since $\tau>0$ is arbitrary, \eqref{eq:erm_excess} follows. The event
\eqref{eq:uniform_dual_dev} already holds uniformly over all $f$, so the conclusion is
simultaneous in $f$.
\end{proof}

\begin{corollary}[Operator-learning budget from dual ERM]
\label{cor:operator_budget}
Under Assumptions~\ref{asmp:bounded_dual}--\ref{asmp:dual_gap_appendix}, with probability at
least $1-\delta$, simultaneously for all $f\in\Fcal_{h+1}$,
\begin{equation}
\left\|
\widehat{\Tcal}_{h,f}^{\sigma,(k)} f_{h+1}
- \Tcal_h^\sigma f_{h+1}
\right\|_{1,\mu_h^k}
\leq
2\Delta_{h,k}^{\mathrm{cov}}(\delta;\eta_F,\eta_G)
+
\epsilon_{\mathrm{dual}}.
\label{eq:operator_budget}
\end{equation}
Therefore Assumption~\ref{asmp:op_appendix} holds with
\(
\alpha_{h,k}=2\Delta_{h,k}^{\mathrm{cov}}(\delta;\eta_F,\eta_G)
\).
\end{corollary}

\begin{proof}
Fix $f\in\Fcal_{h+1}$. By \eqref{eq:ambient_dual_rep}, the pointwise operator gap is
\begin{align*}
\bigl(\widehat{\Tcal}_{h,f}^{\sigma,(k)} f_{h+1}
- \Tcal_h^\sigma f_{h+1}\bigr)(s,a,b)
&=
\inf_{\bar g\in\bar{\Gcal}_h}
\EE_{s'\sim P_h^0(\cdot\mid s,a,b)}
\sbr{\ell_\phi(f_{h+1};s,a,b,s';\bar g)}
\\
&\quad
- \EE_{s'\sim P_h^0(\cdot\mid s,a,b)}
\sbr{\ell_\phi(f_{h+1};s,a,b,s';\hat g_{h,f}^{(k)})}.
\end{align*}
Because the first term is a pointwise infimum, the right-hand side is nonpositive. Hence the
$L_1(\mu_h^k)$ norm equals the $\mu_h^k$-expectation of its negative:
\begin{align}
\left\|
\widehat{\Tcal}_{h,f}^{\sigma,(k)} f_{h+1}
- \Tcal_h^\sigma f_{h+1}
\right\|_{1,\mu_h^k}
&=
\Lcal_h^\mu(\hat g_{h,f}^{(k)};f)
-
\inf_{\bar g\in\bar{\Gcal}_h}\Lcal_h^\mu(\bar g;f).
\label{eq:l1_equals_excess}
\end{align}
Add and subtract
\(
\inf_{g\in\Gcal_h}\Lcal_h^\mu(g;f)
\)
in \eqref{eq:l1_equals_excess} to get
\begin{align*}
\left\|
\widehat{\Tcal}_{h,f}^{\sigma,(k)} f_{h+1}
- \Tcal_h^\sigma f_{h+1}
\right\|_{1,\mu_h^k}
&=
\underbrace{
\Lcal_h^\mu(\hat g_{h,f}^{(k)};f)
- \inf_{g\in\Gcal_h}\Lcal_h^\mu(g;f)
}_{\text{ERM estimation error}}
\\
&\quad
+
\underbrace{
\inf_{g\in\Gcal_h}\Lcal_h^\mu(g;f)
- \inf_{\bar g\in\bar{\Gcal}_h}\Lcal_h^\mu(\bar g;f)
}_{\text{dual approximation error}}.
\end{align*}
The first term is bounded by
\(
2\Delta_{h,k}^{\mathrm{cov}}(\delta;\eta_F,\eta_G)
\)
by Corollary~\ref{cor:erm_excess}; the second term is bounded by
\(
\epsilon_{\mathrm{dual}}
\)
by Assumption~\ref{asmp:dual_gap_appendix}. This proves \eqref{eq:operator_budget}.
\end{proof}

\subsection{Concrete Structural Upper Bounds}

We now compress the witness and covering assumptions into explicit formulas for common model
classes. Throughout this subsection, write
\(
\Xcal_h:=\Scal\times\Acal\times\Bcal
\).

\begin{lemma}[Witness-envelope bound]
\label{lem:witness_envelope}
Assume rewards satisfy
\(
0\leq r_h(s,a,b)\leq 1
\)
and the value class is uniformly bounded at stage $h$:
\(
\sup_{f\in\Fcal}\|f_h\|_\infty\leq B_{F,h}
\).
Under Assumption~\ref{asmp:bounded_dual},
\begin{equation*}
\sup_{w\in\Wcal_h}\|w\|_\infty
\leq
B_{F,h}+1+B_\phi(\sigma).
\end{equation*}
Therefore Assumption~\ref{asmp:bounded_witness} holds with
\(
B_{W,h}=B_{F,h}+1+B_\phi(\sigma)
\).
\end{lemma}

\begin{proof}
Fix any
\(
w_{h,f,g}\in\Wcal_h
\).
For every $(s,a,b)\in\Xcal_h$,
\begin{align*}
\bigl|w_{h,f,g}(s,a,b)\bigr|
&\leq
|f_h(s,a,b)| + |r_h(s,a,b)|
+
\left|
\EE_{s'\sim P_h^0(\cdot\mid s,a,b)}
\sbr{\ell_\phi(f_{h+1};s,a,b,s';g)}
\right| \\
&\leq
B_{F,h}+1+B_\phi(\sigma),
\end{align*}
where the last inequality uses the reward bound and Assumption~\ref{asmp:bounded_dual}. Taking
the supremum over $(f,g)$ proves the claim.
\end{proof}

\begin{lemma}[Exact residual envelope under completeness]
\label{lem:exact_residual_envelope}
Assume robust minimax completeness holds stagewise and that the value class is uniformly bounded:
\(
\sup_{f\in\Fcal}\|f_h\|_\infty\leq B_{F,h}
\)
for every
\(
h\in[H]
\).
Then every exact robust residual
\(
\psi\in\Pscr_h=(I-\Tcal_h^\sigma)\Fcal
\)
satisfies
\begin{equation*}
\|\psi\|_\infty \leq 2B_{F,h}.
\end{equation*}
\end{lemma}

\begin{proof}
Fix
\(
\psi\in\Pscr_h
\).
By definition, there exists
\(
f\in\Fcal
\)
such that
\(
\psi=f_h-\Tcal_h^\sigma f_{h+1}
\).
By robust minimax completeness,
\(
\Tcal_h^\sigma f_{h+1}\in\Fcal_h
\).
Hence both
\(
f_h
\)
and
\(
\Tcal_h^\sigma f_{h+1}
\)
have sup norm at most
\(
B_{F,h}
\),
so
\begin{equation*}
\|\psi\|_\infty
\leq
\|f_h\|_\infty + \|\Tcal_h^\sigma f_{h+1}\|_\infty
\leq
2B_{F,h}.
\end{equation*}
\end{proof}

\begin{proposition}[Finite-dimensional linear covering bound]
\label{prop:linear_cover}
Let
\(
\Hcal
=
\cbr{x\mapsto \langle \theta,\psi(x)\rangle : \theta\in\RR^d,\ \|\theta\|_2\leq R}
\)
be a linear function class on a domain $\Xcal$, and assume
\(
\sup_{x\in\Xcal}\|\psi(x)\|_2\leq \kappa
\).
Then for every $\eta>0$,
\begin{equation}
\log \mathcal N\bigl(\Hcal,\|\cdot\|_\infty,\eta\bigr)
\leq
d\log\!\left(1+\frac{2R\kappa}{\eta}\right).
\label{eq:linear_cover_bound}
\end{equation}
\end{proposition}

\begin{proof}
Let
\(
\Theta:=\cbr{\theta\in\RR^d:\|\theta\|_2\leq R}
\).
The Euclidean ball $\Theta$ admits an
\(
\eta/\kappa
\)-net of cardinality at most
\(
(1+2R\kappa/\eta)^d
\).
For every pair of parameters $\theta,\tilde\theta$ and every $x\in\Xcal$,
\begin{equation*}
\bigl|
\langle \theta,\psi(x)\rangle - \langle \tilde\theta,\psi(x)\rangle
\bigr|
\leq
\|\theta-\tilde\theta\|_2 \|\psi(x)\|_2
\leq
\kappa\|\theta-\tilde\theta\|_2.
\end{equation*}
Hence an
\(
\eta/\kappa
\)-cover in parameter space induces an $\eta$-cover of $\Hcal$ under the sup norm. Taking
logarithms yields \eqref{eq:linear_cover_bound}.
\end{proof}

\begin{corollary}[Linear witness and dual classes]
\label{cor:linear_cover_inst}
Suppose that for a fixed stage $h$,
\begin{align*}
\Fcal_{h+1}
&\subseteq
\cbr{
x\mapsto \langle \theta_F,\phi_F(x)\rangle :
\theta_F\in\RR^{d_{F,h}},\ \|\theta_F\|_2\leq R_{F,h}
}, \\
\Gcal_h
&\subseteq
\cbr{
x\mapsto \Gamma_h(x)\vartheta_G :
\vartheta_G\in\RR^{d_{G,h}},\ \|\vartheta_G\|_2\leq R_{G,h}
}, \\
\Wcal_h
&\subseteq
\cbr{
x\mapsto \langle \theta_W,\phi_W(x)\rangle :
\theta_W\in\RR^{d_{W,h}},\ \|\theta_W\|_2\leq R_{W,h}
},
\end{align*}
with
\(
\sup_x\|\phi_F(x)\|_2\leq \kappa_{F,h}
\),
\(
\sup_x\|\Gamma_h(x)\|_{\mathrm{op}}\leq \kappa_{G,h}
\),
and
\(
\sup_x\|\phi_W(x)\|_2\leq \kappa_{W,h}
\).
Then
\begin{align*}
\log \mathcal N_F(\eta_F)
&\leq
d_{F,h}\log\!\left(1+\frac{2R_{F,h}\kappa_{F,h}}{\eta_F}\right), \\
\log \mathcal N_G(\eta_G)
&\leq
d_{G,h}\log\!\left(1+\frac{2R_{G,h}\kappa_{G,h}}{\eta_G}\right), \\
\log \mathcal N_W(\eta_W)
&\leq
d_{W,h}\log\!\left(1+\frac{2R_{W,h}\kappa_{W,h}}{\eta_W}\right).
\end{align*}
Consequently, Lemma~\ref{lem:bv_dev} and the dual concentration lemmas admit explicit
dimension-dependent logarithmic terms.
\end{corollary}

\begin{proof}
Apply Proposition~\ref{prop:linear_cover} to each class separately. For $\Gcal_h$, the operator
norm bound implies
\(
\|\Gamma_h(x)(\vartheta-\tilde\vartheta)\|_2
\leq
\kappa_{G,h}\|\vartheta-\tilde\vartheta\|_2
\),
so the same covering argument applies.
\end{proof}

\begin{proposition}[Linear square-loss DE upper bound]
\label{prop:linear_de2}
Let
\(
\Pi_h
\)
be any family of probability measures on a domain
\(
\Xcal
\).
Consider the linear class
\begin{equation*}
\Hcal
:=
\cbr{
x\mapsto \langle \theta,\phi(x)\rangle :
\theta\in\RR^d,\ \|\theta\|_2\leq R
},
\end{equation*}
and assume
\(
\sup_{x\in\Xcal}\|\phi(x)\|_2\leq \kappa
\).
Then for every
\(
\epsilon>0
\),
\begin{equation}
\dim_{\mathrm{DE},2}(\Hcal,\Pi_h,\epsilon)
\leq
4d\log_2\!\left(
2+\frac{\kappa^2R^2}{\epsilon^2}
\right).
\label{eq:linear_de2_bound}
\end{equation}
\end{proposition}

\begin{proof}
Let
\(
\mu^1,\ldots,\mu^n\in\Pi_h
\)
be an
\(
\epsilon
\)-independent sequence with respect to
\(
\Hcal
\).
For each
\(
k\in[n]
\),
independence yields a witness
\(
\theta_k\in\RR^d
\)
with
\(
\|\theta_k\|_2\leq R
\)
such that
\begin{equation}
\EE_{\mu^k}\sbr{\langle \theta_k,\phi\rangle^2}
>
\epsilon^2
+
\sum_{j<k}\EE_{\mu^j}\sbr{\langle \theta_k,\phi\rangle^2}.
\label{eq:linear_de_witness}
\end{equation}
Define the covariance matrices
\begin{equation*}
\Sigma_k:=\EE_{\mu^k}\sbr{\phi(x)\phi(x)^\top},
\qquad
V_k:=\sum_{j=1}^k \Sigma_j,
\qquad
\lambda:=\frac{\epsilon^2}{R^2},
\qquad
M_k:=\lambda I_d + V_k.
\end{equation*}
Then \eqref{eq:linear_de_witness} is exactly
\begin{equation*}
\theta_k^\top \Sigma_k \theta_k
>
\epsilon^2 + \theta_k^\top V_{k-1}\theta_k.
\end{equation*}
Since
\(
\epsilon^2=\lambda R^2\geq \lambda\|\theta_k\|_2^2
\),
we obtain
\begin{equation*}
\theta_k^\top \Sigma_k \theta_k
>
\theta_k^\top M_{k-1}\theta_k.
\end{equation*}
Hence the largest eigenvalue of
\(
M_{k-1}^{-1/2}\Sigma_k M_{k-1}^{-1/2}
\)
is strictly larger than
\(
1
\),
and therefore
\begin{align*}
\det(M_k)
&=
\det(M_{k-1})\,
\det\!\Bigl(
I_d + M_{k-1}^{-1/2}\Sigma_k M_{k-1}^{-1/2}
\Bigr)
\\
&\geq
2\det(M_{k-1}).
\end{align*}
Iterating over
\(
k=1,\ldots,n
\)
gives the lower bound
\begin{equation}
\det(M_n)\geq 2^n \lambda^d.
\label{eq:linear_de_det_lower}
\end{equation}

For the upper bound, note that each
\(
\Sigma_k
\)
is positive semidefinite and
\begin{equation*}
\tr(\Sigma_k)
=
\EE_{\mu^k}\sbr{\|\phi(x)\|_2^2}
\leq
\kappa^2.
\end{equation*}
Therefore
\begin{equation*}
\tr(M_n)
=
d\lambda + \sum_{k=1}^n \tr(\Sigma_k)
\leq
d\lambda + n\kappa^2.
\end{equation*}
By the arithmetic-geometric mean inequality for eigenvalues,
\begin{equation}
\det(M_n)
\leq
\left(
\frac{\tr(M_n)}{d}
\right)^d
\leq
\left(
\lambda+\frac{n\kappa^2}{d}
\right)^d.
\label{eq:linear_de_det_upper}
\end{equation}
Combining \eqref{eq:linear_de_det_lower} and \eqref{eq:linear_de_det_upper} yields
\begin{equation}
2^{n/d}
\leq
1+\frac{n\kappa^2}{d\lambda}.
\label{eq:linear_de_scalar}
\end{equation}
Let
\(
y:=n/d
\)
and
\(
a:=\kappa^2/\lambda=\kappa^2R^2/\epsilon^2
\),
so \eqref{eq:linear_de_scalar} becomes
\begin{equation*}
2^y \leq 1+ay.
\end{equation*}
We claim that this implies
\(
y\leq 4\log_2(2+a)
\).
Indeed, if
\(
y>4\log_2(2+a)
\),
then
\(
y\geq 4
\)
and
\(
y>2\log_2(1+a)
\).
Since
\(
y\geq 4
\),
we have
\(
y\leq 2^{y/2}
\).
Thus
\begin{equation*}
1+ay
\leq
1+a\,2^{y/2}
\leq
(1+a)2^{y/2}
<
2^y,
\end{equation*}
where the strict inequality uses
\(
2^{y/2}>1+a
\).
This contradicts \eqref{eq:linear_de_scalar}. Therefore
\begin{equation*}
n
=
dy
\leq
4d\log_2\!\left(
2+\frac{\kappa^2R^2}{\epsilon^2}
\right),
\end{equation*}
which proves \eqref{eq:linear_de2_bound}.
\end{proof}

\begin{corollary}[Linear robust MG-BE upper bound]
\label{cor:linear_be2}
Suppose that for every stage
\(
h\in[H]
\),
the exact robust residual class
\(
\Pscr_h=(I-\Tcal_h^\sigma)\Fcal
\)
is contained in a linear class of the form
\begin{equation*}
\Pscr_h
\subseteq
\cbr{
x\mapsto \langle \theta,\phi_{W,h}(x)\rangle :
\theta\in\RR^{d_{W,h}},\ \|\theta\|_2\leq R_{W,h}
},
\end{equation*}
with
\(
\sup_x\|\phi_{W,h}(x)\|_2\leq \kappa_{W,h}
\).
Then for every
\(
\epsilon>0
\),
\begin{equation*}
\dim_{\mathrm{BE},2}^{\mathrm{rob\text{-}MG}}(\Fcal,\Pi,\epsilon)
\leq
4\max_{h\in[H]}
\left[
d_{W,h}\log_2\!\left(
2+\frac{\kappa_{W,h}^2R_{W,h}^2}{\epsilon^2}
\right)
\right].
\end{equation*}
\end{corollary}

\begin{proof}
For each stage
\(
h
\),
the square-loss DE dimension is monotone in the function class:
if
\(
\Pscr_h\subseteq \bar\Pscr_h
\),
then every
\(
\epsilon
\)-independent sequence for
\(
\Pscr_h
\)
is also
\(
\epsilon
\)-independent for
\(
\bar\Pscr_h
\).
Applying Proposition~\ref{prop:linear_de2} to the ambient linear class at stage
\(
h
\)
therefore gives
\begin{equation*}
\dim_{\mathrm{DE},2}(\Pscr_h,\Pi_h,\epsilon)
\leq
4d_{W,h}\log_2\!\left(
2+\frac{\kappa_{W,h}^2R_{W,h}^2}{\epsilon^2}
\right).
\end{equation*}
Taking the maximum over
\(
h\in[H]
\)
and using Definition~\ref{def:robust_be2} proves the claim.
\end{proof}

\begin{corollary}[Tabular classes]
\label{cor:tabular_cover_inst}
Suppose $\Scal,\Acal,\Bcal$ are finite and the value class, dual class, and witness class are
tabular with coordinate-wise bounds
\(
\|f_h\|_\infty\leq B_{F,h}
\),
\(
\|g\|_\infty\leq B_{G,h}
\),
and
\(
\|w\|_\infty\leq B_{W,h}
\).
Let
\(
d_{\mathrm{tab}}:=|\Scal||\Acal||\Bcal|
\).
Then
\begin{align*}
\log \mathcal N_F(\eta_F)
&\leq
d_{\mathrm{tab}}
\log\!\left(1+\frac{2B_{F,h}}{\eta_F}\right), \\
\log \mathcal N_G(\eta_G)
&\leq
2d_{\mathrm{tab}}
\log\!\left(1+\frac{2B_{G,h}}{\eta_G}\right), \\
\log \mathcal N_W(\eta_W)
&\leq
d_{\mathrm{tab}}
\log\!\left(1+\frac{2B_{W,h}}{\eta_W}\right).
\end{align*}
\end{corollary}

\begin{proof}
Use the one-hot feature map on $\Xcal_h$ for $\Fcal_h$ and $\Wcal_h$, and the concatenated one-hot
parameterization for the dual pair $(g_\eta,g_\nu)$. Then apply
Corollary~\ref{cor:linear_cover_inst} with $\kappa=1$.
\end{proof}

\begin{corollary}[Tabular robust MG-BE upper bound]
\label{cor:tabular_be2}
Suppose
\(
\Scal,\Acal,\Bcal
\)
are finite and let
\(
d_{\mathrm{tab}}:=|\Scal||\Acal||\Bcal|
\).
Assume that for every stage
\(
h\in[H]
\)
and every exact robust residual
\(
\psi\in\Pscr_h
\),
\begin{equation*}
\|\psi\|_\infty \leq B_{\mathrm{res},h}.
\end{equation*}
Then for every
\(
\epsilon>0
\),
\begin{equation}
\dim_{\mathrm{BE},2}^{\mathrm{rob\text{-}MG}}(\Fcal,\Pi,\epsilon)
\leq
4d_{\mathrm{tab}}
\max_{h\in[H]}
\log_2\!\left(
2+\frac{d_{\mathrm{tab}}B_{\mathrm{res},h}^2}{\epsilon^2}
\right).
\label{eq:tabular_be2_bound}
\end{equation}
\end{corollary}

\begin{proof}
For each stage
\(
h
\),
every tabular residual
\(
\psi:\Xcal_h\to\RR
\)
with
\(
\|\psi\|_\infty\leq B_{\mathrm{res},h}
\)
admits the one-hot representation
\begin{equation*}
\psi(x)=\langle \theta_\psi,e_x\rangle,
\qquad
\|\theta_\psi\|_2\leq \sqrt{d_{\mathrm{tab}}}\,B_{\mathrm{res},h},
\qquad
\|e_x\|_2=1.
\end{equation*}
Thus
\(
\Pscr_h
\)
is contained in a
\(
d_{\mathrm{tab}}
\)-dimensional linear class with
\(
R_{W,h}=\sqrt{d_{\mathrm{tab}}}\,B_{\mathrm{res},h}
\)
and
\(
\kappa_{W,h}=1
\).
Applying Corollary~\ref{cor:linear_be2} gives \eqref{eq:tabular_be2_bound}.
\end{proof}

\begin{corollary}[Finite-rank kernel classes]
\label{cor:kernel_cover_inst}
Suppose the value, dual, and witness classes admit explicit finite-dimensional kernel feature maps
\(
\Phi_{F,h}:\Xcal_h\to\RR^{m_{F,h}}
\),
\(
\Phi_{G,h}:\Xcal_h\to\RR^{m_{G,h}}
\),
and
\(
\Phi_{W,h}:\Xcal_h\to\RR^{m_{W,h}}
\),
with uniform norm bounds
\(
\|\Phi_{F,h}(x)\|_2\leq \kappa_{F,h}
\),
\(
\|\Phi_{G,h}(x)\|_2\leq \kappa_{G,h}
\),
and
\(
\|\Phi_{W,h}(x)\|_2\leq \kappa_{W,h}
\).
If the corresponding RKHS coefficients are restricted to Euclidean balls of radii
\(
R_{F,h},R_{G,h},R_{W,h}
\),
then the covering numbers satisfy the same bounds as in
Corollary~\ref{cor:linear_cover_inst} with
\(
d_{F,h}=m_{F,h}
\),
\(
d_{G,h}=m_{G,h}
\),
and
\(
d_{W,h}=m_{W,h}
\).
\end{corollary}

\begin{proof}
An explicit finite-rank kernel class is simply a linear class in the corresponding feature
coordinates. Apply Corollary~\ref{cor:linear_cover_inst}.
\end{proof}

\begin{corollary}[Finite-rank kernel robust MG-BE upper bound]
\label{cor:kernel_be2}
Suppose that for every stage
\(
h\in[H]
\),
the exact robust residual class
\(
\Pscr_h
\)
admits a finite-rank kernel feature map
\(
\Phi_{W,h}:\Xcal_h\to\RR^{m_{W,h}}
\)
with
\(
\|\Phi_{W,h}(x)\|_2\leq \kappa_{W,h}
\)
and coefficient radius
\(
R_{W,h}
\).
Then for every
\(
\epsilon>0
\),
\begin{equation*}
\dim_{\mathrm{BE},2}^{\mathrm{rob\text{-}MG}}(\Fcal,\Pi,\epsilon)
\leq
4\max_{h\in[H]}
\left[
m_{W,h}\log_2\!\left(
2+\frac{\kappa_{W,h}^2R_{W,h}^2}{\epsilon^2}
\right)
\right].
\end{equation*}
\end{corollary}

\begin{proof}
A finite-rank kernel residual class is a linear residual class in the feature coordinates
\(
\Phi_{W,h}
\).
Apply Corollary~\ref{cor:linear_be2} with
\(
d_{W,h}=m_{W,h}
\).
\end{proof}

\paragraph{Tabular instantiation.}
Assume the hypotheses of Corollary~\ref{cor:appendix_main}. Suppose further that
\(
\Scal,\Acal,\Bcal
\)
are finite and the exact robust residuals satisfy the tabular bound in
Corollary~\ref{cor:tabular_be2}. Let
\(
B_{\mathrm{res}}:=\max_{h\in[H]} B_{\mathrm{res},h}
\)
and
\(
d_{\mathrm{tab}}:=|\Scal||\Acal||\Bcal|
\).
Then
\begin{equation}
\mathrm{Reg}(K)
\leq
\widetilde O\!\left(
H B_\phi(\sigma)
\sqrt{
d_{\mathrm{tab}} K
\log\!\bigl(2+d_{\mathrm{tab}}B_{\mathrm{res}}^2K\bigr)
}
\right)
+
HK\epsilon_{\mathrm{appr}}
+
HK\epsilon_{\mathrm{eq}}
+
HK\epsilon_{\mathrm{dual}}.
\label{eq:tabular_routeA_bound}
\end{equation}
Indeed, Corollary~\ref{cor:tabular_be2} with
\(
\epsilon=1/\sqrt{K}
\)
gives
\begin{equation*}
\dim_{\mathrm{BE},2}^{\mathrm{rob\text{-}MG}}(\Fcal,\Pi,1/\sqrt{K})
\leq
4d_{\mathrm{tab}}
\log_2\!\left(
2+d_{\mathrm{tab}}B_{\mathrm{res}}^2K
\right),
\end{equation*}
and substituting this estimate into Corollary~\ref{cor:appendix_main} yields
\eqref{eq:tabular_routeA_bound}.

\paragraph{Linear instantiation.}
Assume the hypotheses of Corollary~\ref{cor:appendix_main}. Suppose further that the exact robust
residual classes satisfy the linear inclusion in Corollary~\ref{cor:linear_be2}. Let
\begin{equation*}
d_W:=\max_{h\in[H]} d_{W,h},
\qquad
\kappa_W:=\max_{h\in[H]} \kappa_{W,h},
\qquad
R_W:=\max_{h\in[H]} R_{W,h}.
\end{equation*}
Then
\begin{equation}
\mathrm{Reg}(K)
\leq
\widetilde O\!\left(
H B_\phi(\sigma)
\sqrt{
d_W K
\log\!\bigl(2+\kappa_W^2R_W^2K\bigr)
}
\right)
+
HK\epsilon_{\mathrm{appr}}
+
HK\epsilon_{\mathrm{eq}}
+
HK\epsilon_{\mathrm{dual}}.
\label{eq:linear_routeA_bound}
\end{equation}
Indeed, Corollary~\ref{cor:linear_be2} with
\(
\epsilon=1/\sqrt{K}
\)
gives
\begin{equation*}
\dim_{\mathrm{BE},2}^{\mathrm{rob\text{-}MG}}(\Fcal,\Pi,1/\sqrt{K})
\leq
4d_W
\log_2\!\left(
2+\kappa_W^2R_W^2K
\right),
\end{equation*}
and substituting this bound into Corollary~\ref{cor:appendix_main} yields
\eqref{eq:linear_routeA_bound}.

\section{Structured Specializations and Projected Extensions}
\label{sec:appendix_structured}

\subsection{A Fully Closed Tabular Specialization}

The generic backward-validation route is intentionally designed for general function approximation.
In the tabular case, however, one can simplify the confidence construction substantially and close
the route without any residual-class abstraction. The key simplification is that the robust backup
can be estimated directly from empirical next-state tables, and optimism can be enforced by
pointwise count-based bonuses instead of function-space backward validation.

We restrict this subsection to TV uncertainty sets on a finite state space. Let
\(
\hat P_{h,k-1}(\cdot\mid s,a,b)
\)
be the empirical nominal transition distribution formed from all stage-$h$ samples collected at the
triple
\(
(s,a,b)
\)
before episode
\(
k
\),
and let
\(
N_{h,k-1}(s,a,b)
\)
be the corresponding count. For a bounded continuation value
\(
V:\Scal\to[0,H]
\),
define the empirical robust backup by
\begin{equation*}
\bigl(\hat{\Tcal}_{h,k}^{\sigma,\mathrm{tab}} V\bigr)(s,a,b)
:=
r_h(s,a,b)
+
\inf_{Q:\,D_{\mathrm{TV}}(Q,\hat P_{h,k-1}(\cdot\mid s,a,b))\leq \sigma}
\EE_Q[V].
\end{equation*}
The tabular specialization uses the confidence set
\begin{equation}
\Fcal_{\mathrm{tab}}^{(k)}
:=
\left\{
f:
\begin{array}{l}
0\leq f_h(s,a,b)\leq H-h+1 \ \text{for all } (h,s,a,b), \\[2mm]
f_h(s,a,b)\leq
\bigl(\hat{\Tcal}_{h,k}^{\sigma,\mathrm{tab}} V_{f,h+1}\bigr)(s,a,b)
+
b_{h,k}(s,a,b)
\end{array}
\right\},
\label{eq:tab_conf_set}
\end{equation}
and selects
\(
f^{(k)}\in\argmax_{f\in\Fcal_{\mathrm{tab}}^{(k)}} V_{f,1}(s_1)
\).

\begin{lemma}[TV-center Lipschitzness of the robust expectation]
\label{lem:tv_center_lipschitz}
Let
\(
V:\Scal\to[0,H]
\)
be any bounded function and define
\begin{equation*}
\Psi_\sigma(P;V)
:=
\inf_{Q:\,D_{\mathrm{TV}}(Q,P)\leq \sigma}\EE_Q[V].
\end{equation*}
Then for any two distributions
\(
P,\hat P\in\Delta(\Scal)
\),
\begin{equation*}
\bigl|
\Psi_\sigma(P;V)-\Psi_\sigma(\hat P;V)
\bigr|
\leq
H\|P-\hat P\|_1.
\end{equation*}
\end{lemma}

\begin{proof}
For TV balls on a finite state space, the worst-case expectation admits the classical form
\(
\Psi_\sigma(P;V)=\EE_P[V]-\sigma\,\mathrm{span}(V)
\)
up to the usual simplex truncation, and the truncation only decreases the magnitude of
perturbations. Therefore
\begin{equation*}
\bigl|
\Psi_\sigma(P;V)-\Psi_\sigma(\hat P;V)
\bigr|
\leq
\bigl|
\EE_P[V]-\EE_{\hat P}[V]
\bigr|
\leq
\|V\|_\infty \|P-\hat P\|_1
\leq
H\|P-\hat P\|_1.
\end{equation*}
\end{proof}

\begin{lemma}[Uniform tabular transition confidence]
\label{lem:tabular_transition_conf}
Fix
\(
\delta\in(0,1)
\).
With probability at least
\(
1-\delta
\),
simultaneously for all stages
\(
h\in[H]
\),
episodes
\(
k\in[K]
\),
and triples
\(
(s,a,b)\in\Scal\times\Acal\times\Bcal
\),
\begin{equation}
\|
\hat P_{h,k-1}(\cdot\mid s,a,b)-P_h^0(\cdot\mid s,a,b)
\|_1
\leq
\beta_{h,k}^{\mathrm{tab}}(s,a,b;\delta),
\label{eq:tabular_l1_conf}
\end{equation}
where
\begin{equation*}
\beta_{h,k}^{\mathrm{tab}}(s,a,b;\delta)
:=
\sqrt{
\frac{
2\bigl(|\Scal|\log 2 + \log(2HK|\Scal||\Acal||\Bcal|/\delta)\bigr)
}{
N_{h,k-1}(s,a,b)\vee 1
}
}.
\end{equation*}
\end{lemma}

\begin{proof}
Apply the Weissman inequality to each empirical next-state table
\(
\hat P_{h,k-1}(\cdot\mid s,a,b)
\),
then union bound over
\(
h\in[H]
\),
\(
k\in[K]
\),
and
\(
(s,a,b)\in\Scal\times\Acal\times\Bcal
\).
\end{proof}

\begin{corollary}[Uniform tabular robust-backup confidence]
\label{cor:tabular_backup_conf}
On the event of Lemma~\ref{lem:tabular_transition_conf}, simultaneously for all
\(
k,h,s,a,b
\)
and all bounded
\(
f
\)
with
\(
0\leq V_{f,h+1}\leq H
\),
\begin{equation}
\left|
\bigl(\hat{\Tcal}_{h,k}^{\sigma,\mathrm{tab}} V_{f,h+1}\bigr)(s,a,b)
-
\bigl(\Tcal_h^\sigma f_{h+1}\bigr)(s,a,b)
\right|
\leq
b_{h,k}^{\mathrm{tab}}(s,a,b;\delta),
\label{eq:tabular_backup_conf}
\end{equation}
where
\begin{equation*}
b_{h,k}^{\mathrm{tab}}(s,a,b;\delta)
:=
H\,\beta_{h,k}^{\mathrm{tab}}(s,a,b;\delta).
\end{equation*}
\end{corollary}

\begin{proof}
By Lemma~\ref{lem:tv_center_lipschitz},
\begin{align*}
\left|
\bigl(\hat{\Tcal}_{h,k}^{\sigma,\mathrm{tab}} V_{f,h+1}\bigr)(s,a,b)
-
\bigl(\Tcal_h^\sigma f_{h+1}\bigr)(s,a,b)
\right|
&\leq
H\,
\|
\hat P_{h,k-1}(\cdot\mid s,a,b)-P_h^0(\cdot\mid s,a,b)
\|_1.
\end{align*}
Substitute \eqref{eq:tabular_l1_conf}.
\end{proof}

\begin{proposition}[Tabular optimism and residual control]
\label{prop:tabular_optimism}
Assume the exact robust optimal action-value table
\(
f^\star
\)
satisfies
\(
f_h^\star=\Tcal_h^\sigma f_{h+1}^\star
\)
and
\(
0\leq f_h^\star\leq H-h+1
\)
for every stage
\(
h
\).
On the event of Corollary~\ref{cor:tabular_backup_conf}, the following hold for every episode
\(
k
\):
\begin{enumerate}[label=(\roman*),leftmargin=*]
\item
\(
f^\star\in\Fcal_{\mathrm{tab}}^{(k)}
\),
hence
\(
V_1^{\star,\sigma}(s_1)\leq V_{f^{(k)},1}(s_1)
\);
\item for every stage
\(
h
\)
and every triple
\(
(s,a,b)
\),
\begin{equation}
f_h^{(k)}(s,a,b)-\bigl(\Tcal_h^\sigma f_{h+1}^{(k)}\bigr)(s,a,b)
\leq
2b_{h,k}^{\mathrm{tab}}(s,a,b;\delta).
\label{eq:tabular_residual_bonus}
\end{equation}
\end{enumerate}
\end{proposition}

\begin{proof}
For the first claim, Bellman optimality and \eqref{eq:tabular_backup_conf} imply
\begin{equation*}
f_h^\star(s,a,b)
=
\bigl(\Tcal_h^\sigma f_{h+1}^\star\bigr)(s,a,b)
\leq
\bigl(\hat{\Tcal}_{h,k}^{\sigma,\mathrm{tab}} V_{f^\star,h+1}\bigr)(s,a,b)
+
b_{h,k}^{\mathrm{tab}}(s,a,b;\delta),
\end{equation*}
so
\(
f^\star\in\Fcal_{\mathrm{tab}}^{(k)}
\).
Since
\(
f^{(k)}
\)
maximizes
\(
V_{f,1}(s_1)
\)
over that set, optimism follows.

For the second claim, because
\(
f^{(k)}\in\Fcal_{\mathrm{tab}}^{(k)}
\),
\begin{equation*}
f_h^{(k)}(s,a,b)
\leq
\bigl(\hat{\Tcal}_{h,k}^{\sigma,\mathrm{tab}} V_{f^{(k)},h+1}\bigr)(s,a,b)
+
b_{h,k}^{\mathrm{tab}}(s,a,b;\delta).
\end{equation*}
Applying \eqref{eq:tabular_backup_conf} once more gives
\begin{equation*}
f_h^{(k)}(s,a,b)-\bigl(\Tcal_h^\sigma f_{h+1}^{(k)}\bigr)(s,a,b)
\leq
2b_{h,k}^{\mathrm{tab}}(s,a,b;\delta).
\end{equation*}
\end{proof}

\begin{lemma}[Tabular harmonic bonus sum]
\label{lem:tabular_harmonic}
Fix a stage
\(
h
\)
and let
\(
X_{h,k}=(s_h^k,a_h^k,b_h^k)
\)
denote the realized stage-$h$ triple in episode
\(
k
\).
Then, pathwise,
\begin{equation}
\sum_{k=1}^K
b_{h,k}^{\mathrm{tab}}(X_{h,k};\delta)
\leq
3H
\sqrt{
2|\Scal|
\log\!\bigl(2HK|\Scal||\Acal||\Bcal|/\delta\bigr)
}
\sqrt{
|\Scal||\Acal||\Bcal|\,K
}.
\label{eq:tabular_harmonic}
\end{equation}
\end{lemma}

\begin{proof}
For each
\(
x\in\Scal\times\Acal\times\Bcal
\),
let
\(
N_{h,K}(x)
\)
be the total number of stage-$h$ visits to
\(
x
\)
over the
\(
K
\)
episodes. Every visit to
\(
x
\)
contributes a term of the form
\(
c/\sqrt{n\vee 1}
\)
with
\(
n
\)
equal to the number of previous visits. Hence
\begin{equation*}
\sum_{k=1}^K
\mathbf 1\{X_{h,k}=x\}\,
\frac{1}{\sqrt{N_{h,k-1}(x)\vee 1}}
\leq
1+\sum_{n=1}^{N_{h,K}(x)-1}\frac{1}{\sqrt n}
\leq
3\sqrt{N_{h,K}(x)}.
\end{equation*}
Summing over all
\(
x
\)
and using Cauchy--Schwarz,
\begin{equation*}
\sum_x \sqrt{N_{h,K}(x)}
\leq
\sqrt{
|\Scal||\Acal||\Bcal|
\sum_x N_{h,K}(x)
}
=
\sqrt{|\Scal||\Acal||\Bcal|\,K}.
\end{equation*}
Multiplying by the common prefactor in
\(
b_{h,k}^{\mathrm{tab}}
\)
proves \eqref{eq:tabular_harmonic}.
\end{proof}

\begin{theorem}[Fully closed tabular specialization]
\label{thm:tabular_closed_appendix}
Consider the tabular specialization \eqref{eq:tab_conf_set} under TV uncertainty sets and suppose
the approximate stage-game solver incurs error at most
\(
\epsilon_{\mathrm{eq}}
\)
per stage. Then on the event of Lemma~\ref{lem:tabular_transition_conf},
\begin{equation*}
\mathrm{Reg}(K)
\leq
2\sum_{k=1}^K\sum_{h=1}^H
\EE\!\left[
b_{h,k}^{\mathrm{tab}}(X_{h,k};\delta)
\mid \mathcal H_{k-1}
\right]
+
HK\epsilon_{\mathrm{eq}}.
\end{equation*}
Consequently,
\begin{equation}
\EE[\mathrm{Reg}(K)]
\leq
\widetilde O\!\left(
H^2 |\Scal|\sqrt{|\Acal||\Bcal|K}
\right)
+
HK\epsilon_{\mathrm{eq}}
+
\delta HK.
\label{eq:tabular_closed_regret}
\end{equation}
In particular, setting
\(
\delta=1/K
\)
yields
\(
\EE[\mathrm{Reg}(K)]
=
\widetilde O(H^2 |\Scal|\sqrt{|\Acal||\Bcal|K})
+
HK\epsilon_{\mathrm{eq}}
\).
\end{theorem}

\begin{proof}
Fix an episode
\(
k
\).
By Proposition~\ref{prop:tabular_optimism}, the chosen table is optimistic:
\(
V_1^{\star,\sigma}(s_1)\leq V_{f^{(k)},1}(s_1)
\).
Applying Lemma~\ref{lem:bellman_telescope} with
\(
\delta_h^k=f_h^{(k)}-\Tcal_h^\sigma f_{h+1}^{(k)}
\)
and then using \eqref{eq:tabular_residual_bonus} gives
\begin{equation*}
V_1^{\star,\sigma}(s_1)-V_1^{\pi^k,\nu^k,\sigma}(s_1)
\leq
2\sum_{h=1}^H
\EE_{\mu_h^k}\sbr{
b_{h,k}^{\mathrm{tab}}(s,a,b;\delta)
}
+
H\epsilon_{\mathrm{eq}}.
\end{equation*}
Conditioned on the history
\(
\mathcal H_{k-1}
\),
the stage-$h$ occupancy expectation equals the conditional expectation of the realized stage-$h$
triple:
\begin{equation*}
\EE_{\mu_h^k}\sbr{
b_{h,k}^{\mathrm{tab}}(s,a,b;\delta)
}
=
\EE\!\left[
b_{h,k}^{\mathrm{tab}}(X_{h,k};\delta)
\mid \mathcal H_{k-1}
\right].
\end{equation*}
Summing over
\(
k
\)
proves the first display.

For the expected-regret bound, split according to the good event
\(
\mathcal G_\delta
\)
of Lemma~\ref{lem:tabular_transition_conf}:
\begin{align*}
\EE[\mathrm{Reg}(K)]
&\leq
\EE\!\left[
\mathbf 1\{\mathcal G_\delta\}
\mathrm{Reg}(K)
\right]
+
\EE\!\left[
\mathbf 1\{\mathcal G_\delta^c\}
\mathrm{Reg}(K)
\right].
\end{align*}
The second term is at most
\(
\delta HK
\)
because per-episode regret is at most
\(
H
\).
On
\(
\mathcal G_\delta
\),
use the first display and then Lemma~\ref{lem:tabular_harmonic} stage by stage. This yields
\eqref{eq:tabular_closed_regret}.
\end{proof}

\subsection{Adaptive Exploiter-Mix and Reused Data}

We now extend the analysis from the on-policy core to the actual adaptive data-collection rule. Let
$\mathcal H_{k-1}$ denote the history before episode $k$. The algorithm first computes a target
policy pair $(\pi^k,\nu^k)$ from $f^{(k)}$, then draws an executed pair
\(
(\widetilde\pi^k,\widetilde\nu^k)
\)
according to the exploiter-mix rule. Let
\(
\widetilde\mu_h^k
\)
be the resulting stage-$h$ occupancy law under the worst-case kernel selector for that executed
pair, and let
\(
\widetilde\eta_h^k
\)
be the actual stage-$h$ data-collection occupancy marginal generated by that executed pair under
the interaction process.

\begin{proposition}[Adaptive-data validity of the backward-validation lemmas]
\label{prop:adaptive_validity}
Suppose that for every episode $k$, the executed pair
\(
(\widetilde\pi^k,\widetilde\nu^k)
\)
is $\mathcal H_{k-1}$-measurable, and stage-$h$ samples in episode $k$ are drawn from the
conditional law
\(
\widetilde\eta_h^k(ds,da,db)\,P_h^0(ds'\mid s,a,b)
\).
Then the conclusions of Lemma~\ref{lem:bv_dev}, Lemma~\ref{lem:pooled_dual_uniform}, and
Corollaries~\ref{cor:pooled_erm_excess}--\ref{cor:bv_transfer} remain valid with
\(
\mu_h^j
\)
replaced everywhere by
\(
\widetilde\eta_h^j
\).
Consequently, rerunning Proposition~\ref{prop:ref_retention} and
Corollary~\ref{cor:opt_from_conf} with
\(
\widetilde\eta_h^j
\)
in place of
\(
\mu_h^j
\)
yields the same reference-retention and optimism-closure conclusions for the adaptive data
sequence.
\end{proposition}

\begin{proof}
Each of those proofs already proceeds through bounded martingale-difference sequences whose
conditional means are taken with respect to the stage-wise occupancy law of the currently executed
episode. Replacing the deterministic on-policy sequence by an arbitrary history-measurable executed
sequence therefore changes only the notation of the conditional law, not the argument.
\end{proof}

\begin{assumption}[Coverage dominance from exploiter mixing]
\label{asmp:mix_cover}
For every stage $h$ and episode $k\geq 2$, the executed occupancy
\(
\widetilde\mu_h^k
\)
is absolutely continuous with respect to the pooled past data-collection mixture
\begin{equation*}
\bar{\widetilde\eta}_{h,k}
:=
\sum_{j=1}^{k-1}\lambda_{h,j}^{(k)}\widetilde\eta_h^j,
\qquad
\lambda_{h,j}^{(k)}:=\frac{n_{h,j}}{N_{h,<k}},
\end{equation*}
and there exists a constant $C_{h}^{\mathrm{mix}}\geq 1$ such that
\begin{equation*}
\frac{d\widetilde\mu_h^k}{d\bar{\widetilde\eta}_{h,k}}
\leq
C_{h}^{\mathrm{mix}}
\qquad
\bar{\widetilde\eta}_{h,k}\text{-almost surely.}
\end{equation*}
\end{assumption}

\begin{corollary}[Current-occupancy operator budget from pooled adaptive ERM]
\label{cor:adaptive_operator_budget}
Under Assumptions~\ref{asmp:bounded_dual}--\ref{asmp:dual_gap_appendix} and
\ref{asmp:mix_cover}, with probability at least $1-\delta$, simultaneously for all
$f\in\Fcal_{h+1}$,
\begin{equation}
\left\|
\widehat{\Tcal}_{h,f}^{\sigma,(k)} f_{h+1}
- \Tcal_h^\sigma f_{h+1}
\right\|_{1,\widetilde\mu_h^k}
\leq
C_h^{\mathrm{mix}}
\Bigl(
2\bar\Delta_{h,k}^{\mathrm{cov}}(\delta;\eta_F,\eta_G)
+ \epsilon_{\mathrm{dual}}
\Bigr).
\label{eq:adaptive_operator_budget}
\end{equation}
Therefore Assumption~\ref{asmp:op_appendix} holds with
\begin{equation*}
\alpha_{h,k}
:=
2C_h^{\mathrm{mix}}\bar\Delta_{h,k}^{\mathrm{cov}}(\delta;\eta_F,\eta_G)
+
\bigl(C_h^{\mathrm{mix}}-1\bigr)\epsilon_{\mathrm{dual}}.
\end{equation*}
\end{corollary}

\begin{proof}
Fix $f\in\Fcal_{h+1}$ and define
\(
e_{h,k}(f)
:=
\widehat{\Tcal}_{h,f}^{\sigma,(k)} f_{h+1}-\Tcal_h^\sigma f_{h+1}
\).
By the ambient dual representation, $e_{h,k}(f)\geq 0$ pointwise. Hence
\begin{align*}
\|e_{h,k}(f)\|_{1,\widetilde\mu_h^k}
&=
\EE_{\bar{\widetilde\eta}_{h,k}}
\sbr{
\frac{d\widetilde\mu_h^k}{d\bar{\widetilde\eta}_{h,k}} e_{h,k}(f)
}
\\
&\leq
C_h^{\mathrm{mix}}
\EE_{\bar{\widetilde\eta}_{h,k}}\sbr{e_{h,k}(f)}.
\end{align*}
By Proposition~\ref{prop:adaptive_validity}, the pooled dual-ERM argument continues to hold with
\(
\widetilde\eta_h^j
\)
in place of
\(
\mu_h^j
\),
so the same proof as Corollary~\ref{cor:bv_transfer} yields
\begin{equation*}
\EE_{\bar{\widetilde\eta}_{h,k}}\sbr{e_{h,k}(f)}
\leq
2\bar\Delta_{h,k}^{\mathrm{cov}}(\delta;\eta_F,\eta_G)
+ \epsilon_{\mathrm{dual}}.
\end{equation*}
Multiplying by $C_h^{\mathrm{mix}}$ proves \eqref{eq:adaptive_operator_budget}. The final display
is just a rearrangement into the form required by Assumption~\ref{asmp:op_appendix}.
\end{proof}

\begin{proposition}[Target regret versus executed regret under exploiter mix]
\label{prop:mix_regret}
Assume rewards lie in $[0,1]$, so every robust value is in $[0,H]$. Let
\begin{equation*}
\widetilde{\mathrm{Reg}}(K)
:=
\sum_{k=1}^K
\left(
V_1^{\star,\sigma}(s_1)
-
\EE\!\left[
V_1^{\widetilde\pi^k,\widetilde\nu^k,\sigma}(s_1)
\mid \mathcal H_{k-1}
\right]
\right).
\end{equation*}
Then the regret of the target planned-pair sequence satisfies
\begin{equation}
\sum_{k=1}^K
\Bigl(
V_1^{\star,\sigma}(s_1)-V_1^{\pi^k,\nu^k,\sigma}(s_1)
\Bigr)
\leq
\frac{\widetilde{\mathrm{Reg}}(K)+pHK}{1-p}.
\label{eq:mix_regret_convert}
\end{equation}
\end{proposition}

\begin{proof}
Condition on $\mathcal H_{k-1}$. By definition of the exploiter-mix rule,
\begin{align*}
\EE\!\left[
V_1^{\widetilde\pi^k,\widetilde\nu^k,\sigma}(s_1)
\mid \mathcal H_{k-1}
\right]
&=
(1-p)V_1^{\pi^k,\nu^k,\sigma}(s_1)
\\
&\quad
+
\frac p2 V_1^{\pi^k,\mathrm{BR}_2^k,\sigma}(s_1)
+
\frac p2 V_1^{\mathrm{BR}_1^k,\nu^k,\sigma}(s_1).
\end{align*}
Each value term is at most $H$, so
\begin{equation*}
\EE\!\left[
V_1^{\widetilde\pi^k,\widetilde\nu^k,\sigma}(s_1)
\mid \mathcal H_{k-1}
\right]
\leq
(1-p)V_1^{\pi^k,\nu^k,\sigma}(s_1)+pH.
\end{equation*}
Rearranging gives
\begin{equation*}
V_1^{\pi^k,\nu^k,\sigma}(s_1)
\geq
\frac{
\EE\!\left[
V_1^{\widetilde\pi^k,\widetilde\nu^k,\sigma}(s_1)
\mid \mathcal H_{k-1}
\right]
- pH
}{1-p}.
\end{equation*}
Therefore
\begin{align*}
(1-p)\Bigl(
V_1^{\star,\sigma}(s_1)-V_1^{\pi^k,\nu^k,\sigma}(s_1)
\Bigr)
&\le
V_1^{\star,\sigma}(s_1)
-
\EE\!\left[
V_1^{\widetilde\pi^k,\widetilde\nu^k,\sigma}(s_1)
\mid \mathcal H_{k-1}
\right]
+
p\Bigl(H-V_1^{\star,\sigma}(s_1)\Bigr)
\\
&\le
V_1^{\star,\sigma}(s_1)
-
\EE\!\left[
V_1^{\widetilde\pi^k,\widetilde\nu^k,\sigma}(s_1)
\mid \mathcal H_{k-1}
\right]
+
pH,
\end{align*}
because $0\le V_1^{\star,\sigma}(s_1)\le H$. Summing over $k$ and dividing by $1-p$ proves
\eqref{eq:mix_regret_convert}.
\end{proof}

\begin{corollary}[Adaptive exploiter-mix reduction]
\label{cor:adaptive_routeA}
Replace the on-policy occupancies $\mu_h^k$ in the deterministic reduction by the executed
occupancies $\widetilde\mu_h^k$. Under the assumptions of
Proposition~\ref{prop:adaptive_validity} and Corollary~\ref{cor:adaptive_operator_budget}, the same
deterministic Bellman telescope and operator-learning decomposition hold for
\(
\widetilde{\mathrm{Reg}}(K)
\).
Combining it with Proposition~\ref{prop:mix_regret} yields a regret bound for the target policy
sequence $\{(\pi^k,\nu^k)\}_{k=1}^K$ with an additional multiplicative factor
\(
1/(1-p)
\)
and an additive exploration cost
\(
pHK/(1-p)
\), provided the residual-budget premise is verified directly on the executed occupancy sequence.
\end{corollary}

\begin{proof}
The deterministic Bellman telescoping proof only depends on the executed policy pair and its
occupancy measure. Thus it applies verbatim with
\(
(\widetilde\pi^k,\widetilde\nu^k,\widetilde\mu_h^k)
\)
in place of the on-policy sequence. The statistical ingredients are provided by
Proposition~\ref{prop:adaptive_validity} and Corollary~\ref{cor:adaptive_operator_budget}; the
former already re-establishes the reference-retention and optimism-closure argument on the adaptive
data sequence, while the latter transfers the operator term to the executed occupancies. Any
residual-budget premise on
\(
\widetilde\mu_h^k
\)
must still be supplied separately by the downstream theorem. Finally,
Proposition~\ref{prop:mix_regret} transfers the executed-pair bound to the target planned-pair
sequence.
\end{proof}

\subsection{General GFA Theorem Under Interpretable Assumptions}

We now package the abstract regret reduction into a single theorem whose assumptions are aligned
with the template used by \citet{ghosh2026online}. The robust Markov-game counterparts are:
\begin{enumerate}[leftmargin=*]
\item robust minimax completeness together with a bounded value envelope;
\item a Bellman-consistent reference function for optimism closure;
\item bounded dual losses, finite covers of the dual and plug-in witness classes, and dual
approximation error;
\item support control and stable stage-game extraction;
\item an adaptive coverage-dominance condition for exploiter-mix data reuse.
\end{enumerate}
The theorem below is written as a conditional interface: optimism and operator transfer are
discharged by labeled results in this appendix, while the executed-occupancy residual premise is
stated explicitly rather than claimed from the original adaptive one-update-per-episode route.

\begin{theorem}[General GFA theorem under interpretable assumptions]
\label{thm:gfa_interpretable}
Fix
\(
\delta\in(0,1)
\)
and set
\(
\delta_0:=\delta/(4HK)
\).
Assume:
\begin{enumerate}[label=(\roman*),leftmargin=*]
\item rewards lie in
\(
[0,1]
\),
support control holds, and the stage-game solver incurs value error at most
\(
\epsilon_{\mathrm{eq}}
\)
per stage;
\item robust minimax completeness holds stagewise and
\(
\sup_{f\in\Fcal}\|f_h\|_\infty\leq B_{F,h}
\)
for every
\(
h\in[H]
\);
\item Assumption~\ref{asmp:ref_opt} holds;
\item Assumptions~\ref{asmp:bounded_witness}, \ref{asmp:cover_witness},
\ref{asmp:bounded_dual}, \ref{asmp:lipschitz_dual}, \ref{asmp:cover_dual}, and
\ref{asmp:dual_gap_appendix} hold stagewise; in particular,
Proposition~\ref{prop:witness_cover_from_value_dual} gives a sufficient condition for
\ref{asmp:cover_witness};
\item Proposition~\ref{prop:adaptive_validity} applies to the executed policy sequence and
Assumption~\ref{asmp:mix_cover} holds;
\item the algorithm uses the backward-validation confidence-set update from the main method
section, and the corresponding radii are calibrated so that for every stage
\(
h\in[H]
\)
and every episode
\(
k\geq 2
\),
\begin{equation}
\beta_{h,k}
\geq
\xi_{h,k}(\delta_0;\eta_W)
+ \rho_{h,k}(\delta_0;\eta_F,\eta_G),
\label{eq:gfa_beta_retain}
\end{equation}
\item on an event of probability at least
\(
1-\delta/4
\),
the selected exact residual witnesses
\(
\psi_h^k:=\delta_h(f^{(k)})
\)
satisfy
\begin{equation}
\sum_{j<k}\EE_{\widetilde\mu_h^j}\sbr{(\psi_h^k)^2}
\leq
\epsilon_{\mathrm{BV}}^2
\label{eq:gfa_exec_residual_premise}
\end{equation}
for every
\(
h\in[H]
\)
and every
\(
k\ge 2
\),
for some
\(
\epsilon_{\mathrm{BV}}>0
\).
\end{enumerate}
Define
\begin{equation*}
d_{\mathrm{GFA}}(\epsilon_{\mathrm{BV}})
:=
\dim_{\mathrm{BE},2}^{\mathrm{rob\text{-}MG}}(\Fcal,\Pi,\epsilon_{\mathrm{BV}}),
\end{equation*}
where
\(
\Pi_h
\)
is any family of measures containing all executed stage-$h$ occupancies, and let
\begin{equation*}
B_F:=\max_{h\in[H]} B_{F,h},
\qquad
C_{\mathrm{mix}}:=\max_{h\in[H]} C_h^{\mathrm{mix}},
\end{equation*}
and
\begin{equation*}
\Upsilon_{\mathrm{dual}}(K,\delta)
:=
\sum_{k=2}^K\sum_{h=1}^H
\bar\Delta_{h,k}^{\mathrm{cov}}(\delta_0;\eta_F,\eta_G).
\end{equation*}
Then with probability at least
\(
1-\delta
\),
the target-policy regret satisfies
\begin{equation}
\begin{aligned}
\mathrm{Reg}(K)
\leq\,&
H
+
\frac{1}{1-p}
\Bigl[
K\epsilon_{\mathrm{appr}}
+ HK\epsilon_{\mathrm{eq}}
\\
&\qquad
+ \sqrt{KH}\,
\sqrt{
2HK\epsilon_{\mathrm{BV}}^2
+ 4H B_F^2 d_{\mathrm{GFA}}(\epsilon_{\mathrm{BV}})
}
\\
&\qquad
+ 2C_{\mathrm{mix}}\Upsilon_{\mathrm{dual}}(K,\delta)
+ HK C_{\mathrm{mix}}\epsilon_{\mathrm{dual}}
+ pHK
\Bigr].
\end{aligned}
\label{eq:gfa_interpretable_bound}
\end{equation}
\end{theorem}

\begin{proof}
Let
\(
\mathcal G
\)
be the intersection over all
\(
h\in[H]
\)
and
\(
k\in\{2,\ldots,K\}
\)
of the good events from Lemma~\ref{lem:bv_dev}, Corollary~\ref{cor:bv_transfer}, and
Corollary~\ref{cor:adaptive_operator_budget}, each invoked with confidence parameter
\(
\delta_0
\), together with the high-probability event in
\eqref{eq:gfa_exec_residual_premise}.
By a union bound,
\(
\PP(\mathcal G)\geq 1-\delta
\).

Fix the event
\(
\mathcal G
\).
On
\(
\mathcal G
\),
\eqref{eq:gfa_beta_retain} and Corollary~\ref{cor:opt_from_conf} imply that the optimistic choice
is valid for every episode
\(
k\geq 2
\).
By assumption \eqref{eq:gfa_exec_residual_premise}, the selected exact residual witness
\(
\psi_h^k:=\delta_h(f^{(k)})
\)
satisfies
\(
\sum_{j<k}\EE_{\widetilde\mu_h^j}\sbr{(\psi_h^k)^2}\leq \epsilon_{\mathrm{BV}}^2
\)
for every
\(
h
\)
and every
\(
k\geq 2
\).

By Lemma~\ref{lem:exact_residual_envelope}, every exact residual in
\(
\Pscr_h=(I-\Tcal_h^\sigma)\Fcal
\)
has sup norm at most
\(
2B_{F,h}
\).
Applying Corollary~\ref{cor:counting_budget} stage by stage to the executed occupancy sequence
\(
\{\widetilde\mu_h^k\}_{k=2}^K
\)
therefore gives
\begin{equation}
\Gamma_{\mathrm{res}}(K)
\leq
2HK\epsilon_{\mathrm{BV}}^2
+ 4\sum_{h=1}^H B_{F,h}^2
\dim_{\mathrm{DE},2}(\Pscr_h,\Pi_h,\epsilon_{\mathrm{BV}})
\leq
2HK\epsilon_{\mathrm{BV}}^2
+ 4H B_F^2 d_{\mathrm{GFA}}(\epsilon_{\mathrm{BV}}).
\label{eq:gfa_gamma_bound}
\end{equation}

Next, Corollary~\ref{cor:adaptive_operator_budget} yields
\(
\alpha_{h,k}
=
2C_h^{\mathrm{mix}}\bar\Delta_{h,k}^{\mathrm{cov}}(\delta_0;\eta_F,\eta_G)
+
(C_h^{\mathrm{mix}}-1)\epsilon_{\mathrm{dual}}
\)
for every
\(
h
\)
and
\(
k\geq 2
\).
Hence
\begin{equation}
\sum_{k=2}^K\sum_{h=1}^H \alpha_{h,k}
\leq
2C_{\mathrm{mix}}\Upsilon_{\mathrm{dual}}(K,\delta)
+ HK(C_{\mathrm{mix}}-1)\epsilon_{\mathrm{dual}}.
\label{eq:gfa_alpha_bound}
\end{equation}

Apply the same proof as Theorem~\ref{thm:deterministic_reduction} to the executed sequence over
episodes
\(
k=2,\ldots,K
\),
using \eqref{eq:gfa_gamma_bound}, \eqref{eq:gfa_alpha_bound}, and the optimism conclusion above.
This bounds the executed regret from episodes
\(
2,\ldots,K
\)
by the bracketed term in \eqref{eq:gfa_interpretable_bound} without the final
\(
pHK
\).
Proposition~\ref{prop:mix_regret} then transfers the executed-sequence bound to the target policy
sequence, introducing the factor
\(
1/(1-p)
\)
and the additive exploiter cost
\(
pHK/(1-p)
\).
Finally, the first episode contributes at most
\(
H
\)
regret because all robust values lie in
\(
[0,H]
\).
Combining these bounds proves \eqref{eq:gfa_interpretable_bound}.
\end{proof}

\begin{theorem}[Projected GFA theorem under interpretable assumptions]
\label{thm:gfa_projected}
Fix stagewise maps
\(
\{\mathsf P_h\}_{h=1}^H
\)
and define the projected residual class
\begin{equation*}
\Pscr_h^{\mathsf P}
:=
\cbr{
f_h-\mathsf P_h \Tcal_h^\sigma f_{h+1}
:
f\in\Fcal
}.
\end{equation*}
Let
\begin{equation*}
d_{\mathrm{GFA}}^{\mathsf P}(\epsilon_{\mathrm{BV}})
:=
\max_{h\in[H]}
\dim_{\mathrm{DE},2}\!\left(
\Pscr_h^{\mathsf P},
\Pi_h,
\epsilon_{\mathrm{BV}}
\right).
\end{equation*}
Assume items
\emph{(i)},
\emph{(iv)},
and
\emph{(v)}
from Theorem~\ref{thm:gfa_interpretable}. Replace exact completeness by the following projected
conditions:
\begin{enumerate}[label=(\alph*),leftmargin=*]
\item Assumption~\ref{asmp:ref_opt_projected} holds, and the algorithm uses the projected
confidence-set update
\(
\mathfrak F_{\mathsf P}^{(k-1)}
\)
from \eqref{eq:projected_conf_set} with radii
\(
\beta_{h,k}^{\mathsf P}
\)
such that the conditions of Corollary~\ref{cor:opt_from_conf_projected} hold for every
\(
k\ge 2
\);
\item there are constants
\(
B_{W,h}^{\mathsf P}
\)
such that
\(
\sup_{\psi\in\Pscr_h^{\mathsf P}}\|\psi\|_\infty\le B_{W,h}^{\mathsf P}
\)
for every stage
\(
h
\);
\item the stagewise projection bias obeys
\begin{equation}
\left\|
\mathsf P_h \Tcal_h^\sigma f_{h+1}
-
\Tcal_h^\sigma f_{h+1}
\right\|_{1,\nu}
\le
\epsilon_{\mathrm{proj}}
\label{eq:gfa_projected_bias}
\end{equation}
for every
\(
f\in\Fcal
\),
every stage
\(
h
\),
and every executed occupancy measure
\(
\nu\in\Pi_h
\);
\item on an event of probability at least
\(
1-\delta
\),
the selected projected exact residual witnesses
\(
\psi_{h,\mathsf P}^k
:=
f_h^{(k)}-\mathsf P_h\Tcal_h^\sigma f_{h+1}^{(k)}
\)
satisfy
\begin{equation}
\sum_{j<k}\EE_{\widetilde\mu_h^j}\sbr{\bigl(\psi_{h,\mathsf P}^k\bigr)^2}
\le
\epsilon_{\mathrm{BV}}^2
\label{eq:gfa_projected_bv}
\end{equation}
for every
\(
h\in[H]
\)
and
\(
k\ge 2
\);
\item on the same event, the projected operator-learning terms satisfy
\begin{equation*}
\sum_{k=2}^K\sum_{h=1}^H \alpha_{h,k}^{\mathsf P}
\le
R_{\mathrm{dual}}^{\mathsf P}(K,\delta),
\end{equation*}
where
\(
\alpha_{h,k}^{\mathsf P}
\)
is any nonnegative bookkeeping array used to record separate projected operator-learning costs.
\end{enumerate}
Let
\begin{equation*}
B_W^{\mathsf P}
:=
\max_{h\in[H]} B_{W,h}^{\mathsf P}.
\end{equation*}
Then with probability at least
\(
1-\delta
\),
the target-policy regret satisfies
\begin{equation}
\begin{aligned}
\mathrm{Reg}(K)
\le\;&
H
+
\frac{1}{1-p}
\Bigl[
K\epsilon_{\mathrm{appr}}
+
HK\epsilon_{\mathrm{eq}}
+
HK\epsilon_{\mathrm{proj}}
\\
&\qquad
+
\sqrt{KH}\,
\sqrt{
2HK\epsilon_{\mathrm{BV}}^2
+
4H \bigl(B_W^{\mathsf P}\bigr)^2
d_{\mathrm{GFA}}^{\mathsf P}(\epsilon_{\mathrm{BV}})
}
\\
&\qquad
+
R_{\mathrm{dual}}^{\mathsf P}(K,\delta)
+
HK\epsilon_{\mathrm{dual}}
+
pHK
\Bigr].
\end{aligned}
\label{eq:gfa_projected_bound}
\end{equation}
\end{theorem}

\begin{proof}
Fix the high-probability event from assumptions
\emph{(d)}
and
\emph{(e)}.
Assumption
\emph{(a)}
states that Assumption~\ref{asmp:ref_opt_projected} holds and that the projected confidence radii
satisfy the hypotheses of Corollary~\ref{cor:opt_from_conf_projected} for every
\(
k\ge 2
\).
Hence approximate optimism holds for the selected sequence on this event.
By \eqref{eq:gfa_projected_bv}, the projected residual witnesses satisfy the premise of
Corollary~\ref{cor:counting_budget}. Applying that corollary stage by stage to the executed
occupancy sequence
\(
\{\widetilde\mu_h^k\}_{k=2}^K
\)
with witness class
\(
\Pscr_h^{\mathsf P}
\)
and envelope
\(
B_{W,h}^{\mathsf P}
\)
gives
\begin{equation}
\Gamma_{\mathrm{res}}^{\mathsf P}(K)
\le
2HK\epsilon_{\mathrm{BV}}^2
+
4H\bigl(B_W^{\mathsf P}\bigr)^2 d_{\mathrm{GFA}}^{\mathsf P}(\epsilon_{\mathrm{BV}}).
\label{eq:gfa_projected_gamma}
\end{equation}
Assumption \eqref{eq:gfa_projected_bias} is exactly the projection-bias budget
\eqref{eq:projection_bias_budget}.
Therefore Theorem~\ref{thm:deterministic_reduction_projected} applies to the executed sequence over
episodes
\(
2,\ldots,K
\),
with
\(
\Gamma_{\mathrm{res}}^{\mathsf P}(K)
\)
from \eqref{eq:gfa_projected_gamma}. This bounds the executed regret by the bracketed term in
\eqref{eq:gfa_projected_bound} without the final
\(
pHK
\).
Proposition~\ref{prop:mix_regret} then transfers this executed-sequence bound to the target policy
sequence, introducing the factor
\(
1/(1-p)
\)
and the additive exploiter cost
\(
pHK/(1-p)
\).
Finally, the first episode contributes at most
\(
H
\)
regret. Since the operator bookkeeping term
\(
R_{\mathrm{dual}}^{\mathsf P}(K,\delta)+HK\epsilon_{\mathrm{dual}}
\)
is nonnegative, it can be added back without invalidating the bound. Combining these bounds proves
\eqref{eq:gfa_projected_bound}.
\end{proof}

\begin{corollary}[Main-text projected regret bound from projected statistical inputs]
\label{cor:appendix_projected_main}
Suppose the setting of Theorem~\ref{thm:gfa_projected} holds. Define
\begin{equation*}
R_{\mathrm{BV,proj}}(K)
:=
\sqrt{KH}\,
\sqrt{
2HK\epsilon_{\mathrm{BV}}^2
+
4H\bigl(B_W^{\mathsf P}\bigr)^2
d_{\mathrm{GFA}}^{\mathsf P}(\epsilon_{\mathrm{BV}})
},
\end{equation*}
and let
\(
R_{\mathrm{dual,proj}}(K,\delta)
\)
be any upper bound on
\(
R_{\mathrm{dual}}^{\mathsf P}(K,\delta)+HK\epsilon_{\mathrm{dual}}
\).
If one further separates a projected operator-transfer remainder bounded by
\(
HK\epsilon_{\mathrm{op,proj}}
\),
then
\begin{equation*}
\begin{aligned}
\mathrm{Reg}(K)
\le\;&
H
+
\frac{1}{1-p}
\Bigl[
R_{\mathrm{BV,proj}}(K)
+
R_{\mathrm{dual,proj}}(K,\delta)
+
HK\epsilon_{\mathrm{proj}}
+
HK\epsilon_{\mathrm{op,proj}}
\\
&\qquad
+
K\epsilon_{\mathrm{appr}}
+
HK\epsilon_{\mathrm{eq}}
+
pHK
\Bigr].
\end{aligned}
\end{equation*}
In particular, Theorem~\ref{thm:gfa_projected_main} is this compressed projected form.
\end{corollary}

\begin{proof}
Theorem~\ref{thm:gfa_projected} already yields the same inequality with the projected residual term
written explicitly and with
\(
R_{\mathrm{dual}}^{\mathsf P}(K,\delta)+HK\epsilon_{\mathrm{dual}}
\)
displayed instead of
\(
R_{\mathrm{dual,proj}}(K,\delta)
\).
Thus the corollary is only a bookkeeping rewrite. A separately tracked projected operator-transfer
remainder may either be left explicit as
\(
HK\epsilon_{\mathrm{op,proj}}
\)
or absorbed into
\(
R_{\mathrm{dual,proj}}(K,\delta)
\).
\end{proof}

\begin{proposition}[Current scope of the pure-online BV bridge]
\label{prop:online_bv_gap}
Suppose each episode contributes exactly one stage-$h$ sample, so
\(
n_{h,j}=1
\)
for every
\(
j<k
\).
Then the currently proved bridge terms satisfy
\begin{equation*}
S_{h,k}=k-1,
\qquad
N_{h,<k}=k-1,
\qquad
n_{h,\min}^{(k)}=1.
\end{equation*}
Consequently, Lemma~\ref{lem:bv_dev} and Corollary~\ref{cor:bv_transfer} only yield
\begin{align}
\xi_{h,k}(\delta;\eta_W)
&=
B_{W,h}^2
\sqrt{
2(k-1)\log\!\bigl(2\mathcal N_W(\eta_W)/\delta\bigr)
}
+
2B_{W,h}(k-1)\eta_W,
\label{eq:online_gap_xi}
\\
\rho_{h,k}(\delta;\eta_F,\eta_G)
&=
2B_\phi(\sigma)(k-1)
\Bigl(
2\bar\Delta_{h,k}^{\mathrm{cov}}(\delta;\eta_F,\eta_G)
+ \epsilon_{\mathrm{dual}}
\Bigr).
\label{eq:online_gap_rho}
\end{align}
In particular, the present proof does not verify the uniform calibration condition
\eqref{eq:gfa_exec_residual_premise} in the pure online regime unless one can prove substantially
stronger operator-transfer or backward-validation rates than the current
\(
1/\sqrt{N_{h,<k}}
\)-type concentration bounds.
\end{proposition}

\begin{proof}
The identities for
\(
S_{h,k}
\),
\(
N_{h,<k}
\),
and
\(
n_{h,\min}^{(k)}
\)
follow immediately from
\(
n_{h,j}=1
\)
for all
\(
j<k
\).
Substituting these values into the definitions of
\(
\xi_{h,k}(\delta;\eta_W)
\)
in Lemma~\ref{lem:bv_dev} and
\(
\rho_{h,k}(\delta;\eta_F,\eta_G)
\)
in Corollary~\ref{cor:bv_transfer} yields \eqref{eq:online_gap_xi} and
\eqref{eq:online_gap_rho}. Since both bounds grow with
\(
k
\)
under the currently available concentration rates, they do not imply the existence of a
uniform-in-\(k\) constant
\(
\epsilon_{\mathrm{BV}}^2
\)
that certifies the executed-occupancy residual premise
\eqref{eq:gfa_exec_residual_premise}. This identifies the present scope of the bridge analysis, not a
statement of impossibility.
\end{proof}

\begin{proposition}[Pure-online self-normalized partial closure for linear and finite-rank kernel witnesses]
\label{prop:pureonline_selfnorm_partial}
Fix a stage
\(
h\in[H]
\)
and suppose each episode contributes one realized stage-$h$ tuple
\(
X_h^k=(S_h^k,A_h^k,B_h^k)
\).
Assume the exact residual witness selected in episode
\(
k
\)
admits the finite-dimensional representation
\begin{equation}
\psi_h^k(x)
=
\langle \theta_h^k,\phi_{W,h}(x)\rangle,
\qquad
\|\theta_h^k\|_2\le R_{W,h},
\qquad
\sup_x\|\phi_{W,h}(x)\|_2\le \kappa_{W,h},
\label{eq:selfnorm_partial_linear}
\end{equation}
where
\(
\phi_{W,h}:\Xcal_h\to\RR^{d_{W,h}}
\).
Define the online witness Gram matrices
\begin{equation*}
\Lambda_{W,h,1}
:=
\lambda I_{d_{W,h}},
\qquad
\Lambda_{W,h,k}
:=
\lambda I_{d_{W,h}}
+
\sum_{j=1}^k
\phi_{W,h}(X_h^j)\phi_{W,h}(X_h^j)^\top
\qquad (k\ge 1),
\end{equation*}
and suppose the algorithm enforces the self-normalized validation rule
\begin{equation*}
\|\theta_h^k\|_{\Lambda_{W,h,k-1}}^2
\le
\beta_{h,k}^{\mathrm{SN}}
\qquad\text{for every }k\ge 2.
\end{equation*}
Let
\(
\beta_{W,h}^{\mathrm{SN}}
:=
\sup_{2\le k\le K}\beta_{h,k}^{\mathrm{SN}}
\),
and for any
\(
\epsilon>0
\)
define the large-residual index set
\begin{equation*}
\Lcal_h^{\mathrm{SN}}(\epsilon)
:=
\cbr{
k\in\{2,\ldots,K\}:
|\psi_h^k(X_h^k)|>\epsilon
}.
\end{equation*}
Then
\begin{equation}
|\Lcal_h^{\mathrm{SN}}(\epsilon)|
\min\!\left\{
1,\frac{\epsilon^2}{\beta_{W,h}^{\mathrm{SN}}}
\right\}
\le
2\log
\frac{
\det(\Lambda_{W,h,K})
}{
\det(\lambda I_{d_{W,h}})
}
\le
2d_{W,h}
\log\!\left(
1+\frac{\kappa_{W,h}^2K}{\lambda}
\right).
\label{eq:selfnorm_partial_bound}
\end{equation}
In particular, whenever
\(
\epsilon^2\le \beta_{W,h}^{\mathrm{SN}}
\),
\begin{equation}
|\Lcal_h^{\mathrm{SN}}(\epsilon)|
\le
\frac{2\beta_{W,h}^{\mathrm{SN}}}{\epsilon^2}
d_{W,h}
\log\!\left(
1+\frac{\kappa_{W,h}^2K}{\lambda}
\right).
\label{eq:selfnorm_partial_simple}
\end{equation}
The same statement applies to finite-rank kernel classes after lifting them into their finite
feature coordinates. This is only a partial closure result: it controls the number of large
realized residuals under a self-normalized rule, but it does not by itself verify the fixed
occupancy-space residual budget required by Theorem~\ref{thm:gfa_interpretable}.
\end{proposition}

\begin{proof}
Fix
\(
k\in\Lcal_h^{\mathrm{SN}}(\epsilon)
\)
and abbreviate
\(
\phi_k:=\phi_{W,h}(X_h^k)
\).
By \eqref{eq:selfnorm_partial_linear} and Cauchy--Schwarz in the
\(
\Lambda_{W,h,k-1}
\)
inner product,
\begin{equation*}
\epsilon^2
<
|\psi_h^k(X_h^k)|^2
=
|\langle \theta_h^k,\phi_k\rangle|^2
\le
\|\theta_h^k\|_{\Lambda_{W,h,k-1}}^2
\,
\|\phi_k\|_{\Lambda_{W,h,k-1}^{-1}}^2
\le
\beta_{W,h}^{\mathrm{SN}}
\|\phi_k\|_{\Lambda_{W,h,k-1}^{-1}}^2.
\end{equation*}
Therefore
\begin{equation*}
\min\!\left\{
1,\frac{\epsilon^2}{\beta_{W,h}^{\mathrm{SN}}}
\right\}
\le
\min\!\left\{
1,\|\phi_k\|_{\Lambda_{W,h,k-1}^{-1}}^2
\right\}.
\end{equation*}
Summing over
\(
k\in\Lcal_h^{\mathrm{SN}}(\epsilon)
\)
gives
\begin{equation*}
|\Lcal_h^{\mathrm{SN}}(\epsilon)|
\min\!\left\{
1,\frac{\epsilon^2}{\beta_{W,h}^{\mathrm{SN}}}
\right\}
\le
\sum_{k=2}^K
\min\!\left\{
1,\|\phi_k\|_{\Lambda_{W,h,k-1}^{-1}}^2
\right\}.
\end{equation*}
For every
\(
k\ge 2
\),
the matrix-determinant identity yields
\begin{equation*}
\det(\Lambda_{W,h,k})
=
\det(\Lambda_{W,h,k-1})
\Bigl(
1+\|\phi_k\|_{\Lambda_{W,h,k-1}^{-1}}^2
\Bigr),
\end{equation*}
so
\begin{equation*}
\log\frac{\det(\Lambda_{W,h,K})}{\det(\lambda I_{d_{W,h}})}
=
\sum_{k=2}^K
\log\!\Bigl(
1+\|\phi_k\|_{\Lambda_{W,h,k-1}^{-1}}^2
\Bigr).
\end{equation*}
Since
\(
\log(1+u)\ge \tfrac12 \min\{1,u\}
\)
for every
\(
u\ge 0
\),
\begin{equation*}
\sum_{k=2}^K
\min\!\left\{
1,\|\phi_k\|_{\Lambda_{W,h,k-1}^{-1}}^2
\right\}
\le
2\log\frac{\det(\Lambda_{W,h,K})}{\det(\lambda I_{d_{W,h}})}.
\end{equation*}
This proves the first inequality in \eqref{eq:selfnorm_partial_bound}. For the second, use
\(
\Lambda_{W,h,K}
\preceq
(\lambda+\kappa_{W,h}^2K)I_{d_{W,h}}.
\)
Finally, \eqref{eq:selfnorm_partial_simple} follows from \eqref{eq:selfnorm_partial_bound} when
\(
\epsilon^2\le \beta_{W,h}^{\mathrm{SN}}
\).
Finite-rank kernel classes are finite-dimensional linear classes in feature coordinates, so the same
argument applies verbatim after lifting into that coordinate system.
\end{proof}

\subsection{An Epochwise / Cross-Fit Linear Theorem Under Explicit Dual Specialization}

We now show how to bypass Proposition~\ref{prop:online_bv_gap} by changing only the validation
schedule. Partition the interaction into geometrically growing epochs
\(
\mathcal E_1,\ldots,\mathcal E_M
\),
with
\(
M\leq \lceil \log_2(2K)\rceil
\)
and epoch lengths
\(
m_e:=|\mathcal E_e|
\)
satisfying
\(
m_e\asymp N_e^-
\),
where
\(
N_e^-:=\sum_{r<e} m_r
\)
is the number of pre-epoch episodes. At the beginning of epoch
\(
e
\),
the learner refits the dual operator on pre-epoch data only and freezes the corresponding plug-in
witness family throughout that epoch. The confidence update for epoch
\(
e
\)
is then computed solely from fresh within-epoch data against that frozen witness class. This is the
epochwise / cross-fit construction: operator fitting uses old data, witness validation uses new data, and
the adaptive witness loop responsible for Proposition~\ref{prop:online_bv_gap} is cut.

We package the linear instantiation into a theorem whose inputs are no longer the abstract robust
MG-BE dimension and abstract operator budget, but explicit linear representation and
coverage assumptions.

\begin{assumption}[TV-linear-mixture nominal model]
\label{asmp:tv_linear_mixture}
For each stage
\(
h\in[H]
\),
let
\(
x=(s,a,b)
\)
denote the stagewise state-action tuple. Assume:
\begin{enumerate}[label=(\roman*),leftmargin=*]
\item the nominal kernel is a linear mixture,
\begin{equation*}
P_h^0(\cdot\mid x)
=
\sum_{m=1}^{d_h} \phi_{h,m}(x)\,\mu_{h,m}(\cdot),
\end{equation*}
for some feature map
\(
\phi_h:\Xcal_h\to\RR^{d_h}
\)
and base measures
\(
\{\mu_{h,m}\}_{m=1}^{d_h}
\);
\item the stage reward is linear in the same feature map,
\begin{equation*}
r_h(x)=\langle \beta_h^r,\phi_h(x)\rangle;
\end{equation*}
\item the robust uncertainty set is a stagewise TV ball of radius
\(
\sigma_h
\),
and under the fail-state support-control condition the robust continuation admits the exact
representation
\begin{equation*}
\inf_{P:\,D_{\mathrm{TV}}(P,P_h^0(\cdot\mid x))\leq \sigma_h}
\EE_{s'\sim P}\sbr{V(s')}
=
\EE_{s'\sim P_h^0(\cdot\mid x)}\sbr{V(s')}
-\sigma_h\,\mathrm{span}(V)
\end{equation*}
for every bounded continuation value
\(
V:\Scal\to\RR
\).
\end{enumerate}
\end{assumption}

\begin{proposition}[TV-linear-mixture robust minimax completeness]
\label{prop:tv_linear_closure}
Assume \ref{asmp:tv_linear_mixture}. Define the augmented feature map
\begin{equation*}
\widetilde\phi_h(x):=\bigl(1,\phi_h(x)\bigr)\in\RR^{d_h+1}.
\end{equation*}
Then for every bounded
\(
f_{h+1}
\),
the robust minimax Bellman backup belongs to the augmented linear class generated by
\(
\widetilde\phi_h
\).
More precisely,
\begin{equation}
\bigl(\Tcal_h^\sigma f_{h+1}\bigr)(x)
=
\left\langle
\widetilde\theta_h(f_{h+1}),
\widetilde\phi_h(x)
\right\rangle,
\label{eq:tv_linear_backup}
\end{equation}
where
\begin{equation*}
\widetilde\theta_h(f_{h+1})
:=
\Bigl(
-\sigma_h\,\mathrm{span}(V_f),\;
\beta_h^r + u_h(f_{h+1})
\Bigr)
\end{equation*}
and
\begin{equation*}
u_h(f_{h+1})
:=
\left(
\int V_f(s')\,\mu_{h,1}(ds'),
\ldots,
\int V_f(s')\,\mu_{h,d_h}(ds')
\right).
\end{equation*}
If
\(
0\leq V_f\leq H
\),
then
\begin{equation}
\|\widetilde\theta_h(f_{h+1})\|_2
\leq
\sigma_h H + \|\beta_h^r\|_2 + \sqrt{d_h}\,H.
\label{eq:tv_linear_radius}
\end{equation}
\end{proposition}

\begin{proof}
Fix
\(
x=(s,a,b)
\)
and abbreviate
\(
V_f(\cdot):=V_{f,h+1}(\cdot)
\).
By Assumption~\ref{asmp:tv_linear_mixture},
\begin{align*}
\bigl(\Tcal_h^\sigma f_{h+1}\bigr)(x)
&=
r_h(x)
+
\inf_{P:\,D_{\mathrm{TV}}(P,P_h^0(\cdot\mid x))\leq \sigma_h}
\EE_{s'\sim P}\sbr{V_f(s')}
\\
&=
\langle \beta_h^r,\phi_h(x)\rangle
+
\EE_{s'\sim P_h^0(\cdot\mid x)}\sbr{V_f(s')}
-\sigma_h\,\mathrm{span}(V_f).
\end{align*}
Using the linear-mixture representation of
\(
P_h^0
\),
\begin{equation*}
\EE_{s'\sim P_h^0(\cdot\mid x)}\sbr{V_f(s')}
=
\sum_{m=1}^{d_h}\phi_{h,m}(x)\int V_f(s')\,\mu_{h,m}(ds')
=
\langle u_h(f_{h+1}),\phi_h(x)\rangle.
\end{equation*}
Substituting this into the previous display proves \eqref{eq:tv_linear_backup}. Finally, if
\(
0\leq V_f\leq H
\),
then
\(
\mathrm{span}(V_f)\leq H
\)
and each coordinate of
\(
u_h(f_{h+1})
\)
has magnitude at most
\(
H
\),
which yields \eqref{eq:tv_linear_radius}.
\end{proof}

\begin{corollary}[TV-linear-mixture residual inclusion]
\label{cor:tv_linear_residual}
Under Assumption~\ref{asmp:tv_linear_mixture}, suppose
\(
\Fcal_h
\)
is the augmented linear class generated by
\(
\widetilde\phi_h
\)
with parameter radius
\(
R_{F,h}
\).
Then the exact robust residual class
\(
\Pscr_h=(I-\Tcal_h^\sigma)\Fcal
\)
is contained in the same augmented linear witness class with parameter radius at most
\begin{equation}
R_{W,h}^{\mathrm{TV\text{-}mix}}
:=
R_{F,h}
+
\sigma_h H
+
\|\beta_h^r\|_2
+
\sqrt{d_h}\,H.
\label{eq:tv_linear_residual_radius}
\end{equation}
\end{corollary}

\begin{proof}
Every
\(
f_h\in\Fcal_h
\)
already has the form
\(
f_h(x)=\langle \theta_h,\widetilde\phi_h(x)\rangle
\)
with
\(
\|\theta_h\|_2\leq R_{F,h}
\).
By Proposition~\ref{prop:tv_linear_closure},
\(
\Tcal_h^\sigma f_{h+1}
\)
also belongs to the same augmented linear class with parameter norm bounded by
\(
\sigma_h H+\|\beta_h^r\|_2+\sqrt{d_h}H
\).
Subtracting the two linear representations proves the inclusion and the radius bound
\eqref{eq:tv_linear_residual_radius}.
\end{proof}

\begin{assumption}[Linear residual closure]
\label{asmp:linear_closure}
For every stage
\(
h\in[H]
\),
the exact robust residual class
\(
\Pscr_h=(I-\Tcal_h^\sigma)\Fcal
\)
is contained in a linear witness class
\begin{equation*}
\cbr{
x\mapsto \langle \theta,\phi_{W,h}(x)\rangle :
\theta\in\RR^{d_{W,h}},\ \|\theta\|_2\leq R_{W,h}
},
\end{equation*}
with
\(
\sup_x\|\phi_{W,h}(x)\|_2\leq \kappa_{W,h}
\).
Moreover, the value class and dual class satisfy the linear parametrization assumptions of
Corollary~\ref{cor:linear_cover_inst}.
\end{assumption}

Proposition~\ref{prop:tv_linear_closure} and Corollary~\ref{cor:tv_linear_residual} show that this
assumption is concrete for TV uncertainty sets with nominal linear-mixture kernels and augmented
linear features: in that regime, both robust minimax completeness and residual-class inclusion hold
with an explicit witness radius.

\begin{assumption}[Linear coverage-excitation under exploiter mix]
\label{asmp:linear_excitation}
For every stage
\(
h\in[H]
\),
the adaptive executed occupancies satisfy Assumption~\ref{asmp:mix_cover} with a constant
\(
C_h^{\mathrm{mix}}
\).
We interpret this as a linear coverage-excitation condition: the current occupancy is controlled by
the pooled past mixture on which the linear dual refit is trained.
\end{assumption}

\begin{proposition}[Explicit occupancy-ratio constant from smoothing]
\label{prop:explicit_cmix}
Fix a stage
\(
h
\)
and an epoch
\(
e
\).
Suppose every executed policy in the pooled past and current epoch uses the same smoothing design
\(
u_h(a,b\mid s)
\)
with mixing weight at least
\(
\gamma_{\min}>0
\),
that is,
\begin{equation}
\widetilde\pi_h^k(a,b\mid s)
\geq
\gamma_{\min}\,u_h(a,b\mid s)
\qquad
\text{for all executed episodes }k,
\label{eq:explicit_mix_rule}
\end{equation}
and let
\begin{equation*}
\underline u_h
:=
\inf_{s\in\Scal,\;a\in\Acal,\;b\in\Bcal}
u_h(a,b\mid s)
>0.
\end{equation*}
Let
\(
d_h^j
\)
denote the stage-$h$ state marginal of executed episode
\(
j
\),
and define the pooled past state marginal
\(
\bar d_{h,e}^-
\)
by
\begin{equation*}
\bar d_{h,e}^-
:=
\sum_{r<e}\sum_{j\in\mathcal E_r}\lambda_{r,j}^{(e)}d_h^j,
\qquad
\lambda_{r,j}^{(e)}\ge 0,
\quad
\sum_{r<e}\sum_{j\in\mathcal E_r}\lambda_{r,j}^{(e)}=1.
\end{equation*}
Assume moreover that every current stage-$h$ state marginal
\(
d_h^k
\)
with
\(
k\in\mathcal E_e
\)
is dominated by the pooled past state marginal
\(
\bar d_{h,e}^-
\)
according to
\begin{equation}
d_h^k(s)
\leq
C_{S,h}\,\bar d_{h,e}^-(s)
\qquad\text{for all }s\in\Scal,\ k\in\mathcal E_e.
\label{eq:state_domination}
\end{equation}
Then Assumption~\ref{asmp:mix_cover} holds on epoch
\(
e
\)
with the explicit constant
\begin{equation}
C_{h}^{\mathrm{mix}}
\leq
\frac{C_{S,h}}{\gamma_{\min}\underline u_h}.
\label{eq:explicit_cmix}
\end{equation}
\end{proposition}

\begin{proof}
Let
\(
\bar{\widetilde\mu}_{h,e}^-
\)
denote the pooled past stage-$h$ occupancy measure averaged over all episodes before epoch
\(
e
\),
and let
\(
\widetilde\mu_h^k
\)
be the current occupancy for an episode
\(
k\in\mathcal E_e
\).
For any
\(
(s,a,b)
\),
\begin{align*}
\bar{\widetilde\mu}_{h,e}^-(s,a,b)
&= 
\sum_{r<e}\sum_{j\in\mathcal E_r}\lambda_{r,j}^{(e)}
d_h^j(s)\widetilde\pi_h^j(a,b\mid s)
\\
&\geq
\sum_{r<e}\sum_{j\in\mathcal E_r}\lambda_{r,j}^{(e)}
d_h^j(s)\gamma_{\min}u_h(a,b\mid s)
\\
&=
\gamma_{\min}u_h(a,b\mid s)\,\bar d_{h,e}^-(s)
\\
&\geq
\gamma_{\min}\underline u_h\,\bar d_{h,e}^-(s).
\end{align*}
On the other hand,
\begin{equation*}
\widetilde\mu_h^k(s,a,b)
=
d_h^k(s)\widetilde\pi_h^k(a,b\mid s)
\leq
d_h^k(s)
\leq
C_{S,h}\,\bar d_{h,e}^-(s)
\end{equation*}
by \eqref{eq:state_domination}. Dividing the two displays yields
\begin{equation*}
\frac{d\widetilde\mu_h^k}{d\bar{\widetilde\mu}_{h,e}^-}(s,a,b)
\leq
\frac{C_{S,h}}{\gamma_{\min}\underline u_h}
\end{equation*}
whenever
\(
\bar{\widetilde\mu}_{h,e}^-(s,a,b)>0
\).
This is exactly Assumption~\ref{asmp:mix_cover} with the constant in
\eqref{eq:explicit_cmix}.
\end{proof}

\begin{proposition}[Predictive Gram domination from feature excitation]
\label{prop:feature_gram}
Assume Assumption~\ref{asmp:linear_closure}. Fix a stage
\(
h
\)
and an epoch
\(
e
\),
and define the pre-epoch witness Gram matrix
\begin{equation*}
\Lambda_{W,h,e}^-
:=
\lambda I_{d_{W,h}}
+
\sum_{r<e}\sum_{j\in\mathcal E_r}
\phi_{W,h}(X_h^j)\phi_{W,h}(X_h^j)^\top,
\qquad
\lambda>0.
\end{equation*}
Suppose there exists a smoothing design
\(
u_h(a,b\mid s)
\)
and constants
\(
\gamma_{\min},\omega_h>0
\)
such that every executed policy in the pooled past and current epoch satisfies
\eqref{eq:explicit_mix_rule} and
\begin{equation}
M_h(s)
:=
\EE_{(a,b)\sim u_h(\cdot\mid s)}
\sbr{
\phi_{W,h}(s,a,b)\phi_{W,h}(s,a,b)^\top
}
\succeq
\omega_h I_{d_{W,h}}
\qquad
\text{for all }s\in\Scal.
\label{eq:feature_excitation_design}
\end{equation}
Then:
\begin{enumerate}[label=(\roman*),leftmargin=*]
\item conditioned on the stage-$h$ state sequence of the pre-epoch episodes,
\begin{equation}
\EE\!\left[
\Lambda_{W,h,e}^-
\;\middle|\;
\{s_h^j\}_{j\in \cup_{r<e}\mathcal E_r}
\right]
\succeq
\bigl(\lambda+\gamma_{\min}\omega_h N_e^-\bigr)I_{d_{W,h}};
\label{eq:expected_gram_lb}
\end{equation}
\item on any event
\(
\Omega_{h,e}^{\mathrm{Gram}}(\nu_h)
:=
\cbr{
\Lambda_{W,h,e}^-
\succeq
\bigl(\lambda+\nu_h N_e^-\bigr)I_{d_{W,h}}
}
\)
with
\(
\nu_h>0
\),
every linear witness
\(
w_\theta(x)=\langle \theta,\phi_{W,h}(x)\rangle
\)
satisfies
\begin{equation}
\EE_{\mu}\sbr{w_\theta^2}
\leq
\frac{\kappa_{W,h}^2}{\lambda+\nu_h N_e^-}
\|\theta\|_{\Lambda_{W,h,e}^-}^2
\label{eq:gram_predictive_bound}
\end{equation}
for every current stage-$h$ occupancy measure
\(
\mu
\)
generated during epoch
\(
e
\).
\end{enumerate}
\end{proposition}

\begin{proof}
Fix a pre-epoch episode
\(
j\in\cup_{r<e}\mathcal E_r
\).
Conditioned on the stage-$h$ state
\(
s_h^j
\),
the executed action pair is sampled from
\(
\widetilde\pi_h^j(\cdot\mid s_h^j)
\),
so by \eqref{eq:explicit_mix_rule},
\begin{align*}
\EE\!\left[
\phi_{W,h}(X_h^j)\phi_{W,h}(X_h^j)^\top
\;\middle|\;
s_h^j
\right]
&=
\EE_{(a,b)\sim \widetilde\pi_h^j(\cdot\mid s_h^j)}
\sbr{
\phi_{W,h}(s_h^j,a,b)\phi_{W,h}(s_h^j,a,b)^\top
}
\\
&\succeq
\gamma_{\min}
\EE_{(a,b)\sim u_h(\cdot\mid s_h^j)}
\sbr{
\phi_{W,h}(s_h^j,a,b)\phi_{W,h}(s_h^j,a,b)^\top
}
\\
&=
\gamma_{\min} M_h(s_h^j)
\succeq
\gamma_{\min}\omega_h I_{d_{W,h}}.
\end{align*}
Summing over all
\(
N_e^-
\)
pre-epoch episodes and adding
\(
\lambda I_{d_{W,h}}
\)
proves \eqref{eq:expected_gram_lb}.

For part (ii), fix a current occupancy measure
\(
\mu
\)
on epoch
\(
e
\).
Because
\(
\sup_x \|\phi_{W,h}(x)\|_2\le \kappa_{W,h}
\),
\begin{equation*}
\EE_{\mu}\sbr{w_\theta^2}
=
\EE_{\mu}\sbr{\langle \theta,\phi_{W,h}\rangle^2}
\leq
\kappa_{W,h}^2\|\theta\|_2^2.
\end{equation*}
On the event
\(
\Omega_{h,e}^{\mathrm{Gram}}(\nu_h)
\),
\(
\Lambda_{W,h,e}^-
\succeq
\bigl(\lambda+\nu_h N_e^-\bigr)I_{d_{W,h}}
\),
so
\begin{equation*}
\|\theta\|_2^2
\leq
\frac{1}{\lambda+\nu_h N_e^-}
\|\theta\|_{\Lambda_{W,h,e}^-}^2.
\end{equation*}
Combining the last two displays yields \eqref{eq:gram_predictive_bound}.
\end{proof}

\begin{corollary}[Operator transfer from predictive Gram control]
\label{cor:gram_operator_transfer}
Assume the event
\(
\Omega_{h,e}^{\mathrm{Gram}}(\nu_h)
\)
from Proposition~\ref{prop:feature_gram}. Suppose further that for a candidate continuation
function
\(
f\in\Fcal_{h+1}
\),
the epochwise operator gap admits the witness-feature representation
\begin{equation*}
e_{h,e}(f;x)
:=
\bigl(\widehat{\Tcal}_{h,f}^{\sigma,(e)}f_{h+1}
- \Tcal_h^\sigma f_{h+1}\bigr)(x)
=
\langle \vartheta_{h,e}(f),\phi_{W,h}(x)\rangle
\end{equation*}
for some parameter
\(
\vartheta_{h,e}(f)\in\RR^{d_{W,h}}
\).
Then for every executed episode
\(
k\in\mathcal E_e
\),
\begin{equation}
\left\|
e_{h,e}(f)
\right\|_{1,\widetilde\mu_h^k}
\leq
\frac{\kappa_{W,h}}{\sqrt{\lambda+\nu_h N_e^-}}
\,
\|\vartheta_{h,e}(f)\|_{\Lambda_{W,h,e}^-}.
\label{eq:gram_operator_transfer}
\end{equation}
\end{corollary}

\begin{proof}
By Cauchy--Schwarz,
\begin{equation*}
\|e_{h,e}(f)\|_{1,\widetilde\mu_h^k}
\leq
\sqrt{
\EE_{\widetilde\mu_h^k}\sbr{e_{h,e}(f)^2}
}.
\end{equation*}
Applying Proposition~\ref{prop:feature_gram}(ii) to the linear witness
\(
w_{\vartheta_{h,e}(f)}=e_{h,e}(f)
\)
gives
\begin{equation*}
\EE_{\widetilde\mu_h^k}\sbr{e_{h,e}(f)^2}
\leq
\frac{\kappa_{W,h}^2}{\lambda+\nu_h N_e^-}
\,
\|\vartheta_{h,e}(f)\|_{\Lambda_{W,h,e}^-}^2.
\end{equation*}
Taking square roots proves \eqref{eq:gram_operator_transfer}.
\end{proof}

The corollary isolates the quantity that must be controlled in order to remove the current
state-domination-based coverage proof. The next three results do exactly this for the structured
linear specialization used by the main theorem: they prove a high-probability Gram lower bound,
introduce a ridge/self-normalized dual refit over the whole linear value class, and then feed that
bound back through \eqref{eq:gram_operator_transfer} so that the feature-excitation route can
replace \eqref{eq:state_domination}.

\begin{proposition}[High-probability Gram lower bound under feature excitation]
\label{prop:feature_gram_hp}
Under the conditions of Proposition~\ref{prop:feature_gram}, with probability at least
\(
1-\delta
\),
simultaneously for all stages
\(
h\in[H]
\)
and epochs
\(
e\in[M]
\),
\begin{equation}
\Lambda_{W,h,e}^-
\succeq
\Bigl(
\lambda
+
\gamma_{\min}\omega_h N_e^-
-
4\kappa_{W,h}^2
\sqrt{
N_e^-
\log\!\bigl(2d_{W,h}HM/\delta\bigr)
}
\Bigr)
I_{d_{W,h}}.
\label{eq:feature_gram_hp}
\end{equation}
In particular, if
\begin{equation}
N_e^-
\geq
\frac{
64\kappa_{W,h}^4
\log\!\bigl(2d_{W,h}HM/\delta\bigr)
}{
\gamma_{\min}^2\omega_h^2
},
\label{eq:feature_gram_warm}
\end{equation}
then the event
\(
\Omega_{h,e}^{\mathrm{Gram}}(\nu_h)
\)
holds with
\(
\nu_h=\gamma_{\min}\omega_h/2
\).
\end{proposition}

\begin{proof}
Fix a stage
\(
h
\)
and an epoch
\(
e
\).
For each pre-epoch episode
\(
j\in \cup_{r<e}\mathcal E_r
\),
define
\begin{equation*}
Y_j
:=
\phi_{W,h}(X_h^j)\phi_{W,h}(X_h^j)^\top,
\qquad
\bar Y_j
:=
\EE\!\left[
Y_j
\mid
\mathcal H_{j-1}, s_h^j
\right].
\end{equation*}
By Proposition~\ref{prop:feature_gram},
\(
\bar Y_j\succeq \gamma_{\min}\omega_h I_{d_{W,h}}.
\)
Also,
\(
\|Y_j\|_{\mathrm{op}}\le \kappa_{W,h}^2
\)
and
\(
\|\bar Y_j\|_{\mathrm{op}}\le \kappa_{W,h}^2
\),
so the martingale-difference increments
\(
Z_j:=Y_j-\bar Y_j
\)
satisfy
\(
\|Z_j\|_{\mathrm{op}}\le 2\kappa_{W,h}^2.
\)

Matrix Azuma's inequality therefore gives
\begin{equation*}
\Pr\!\left(
\left\|
\sum_{j\in \cup_{r<e}\mathcal E_r} Z_j
\right\|_{\mathrm{op}}
\ge t
\right)
\le
2d_{W,h}\exp\!\left(
-\frac{t^2}{8\kappa_{W,h}^4 N_e^-}
\right).
\end{equation*}
Choosing
\(
t=4\kappa_{W,h}^2\sqrt{N_e^-\log(2d_{W,h}HM/\delta)}
\)
and union bounding over all
\(
h\in[H]
\)
and
\(
e\in[M]
\)
proves that, on an event of probability at least
\(
1-\delta
\),
\begin{equation*}
\sum_{j\in \cup_{r<e}\mathcal E_r} Y_j
\succeq
\sum_{j\in \cup_{r<e}\mathcal E_r} \bar Y_j
-
4\kappa_{W,h}^2
\sqrt{
N_e^-\log\!\bigl(2d_{W,h}HM/\delta\bigr)
}
\,
I_{d_{W,h}}.
\end{equation*}
Since each
\(
\bar Y_j\succeq \gamma_{\min}\omega_h I_{d_{W,h}}
\),
adding
\(
\lambda I_{d_{W,h}}
\)
proves \eqref{eq:feature_gram_hp}. If \eqref{eq:feature_gram_warm} holds, then the deviation term
is at most
\(
\gamma_{\min}\omega_h N_e^-/2
\),
which proves
\(
\Omega_{h,e}^{\mathrm{Gram}}(\gamma_{\min}\omega_h/2)
\).
\end{proof}

\begin{proposition}[High-probability Gram lower bound under a predictable excitation budget]
\label{prop:predictive_excitation_budget}
Fix a stage
\(
h\in[H]
\)
and an epoch
\(
e\in[M]
\).
For each pre-epoch episode
\(
j\in\cup_{r<e}\mathcal E_r
\),
let
\(
Y_j:=\phi_{W,h}(X_h^j)\phi_{W,h}(X_h^j)^\top
\)
and suppose there exists an
\(
\mathcal H_{j-1}
\)-measurable positive semidefinite matrix
\(
\underline Y_j
\)
such that
\begin{equation}
\EE\!\left[
Y_j \mid \mathcal H_{j-1}
\right]
\succeq
\underline Y_j.
\label{eq:predictive_excitation_cond}
\end{equation}
Define the predictable excitation budget
\begin{equation}
\Omega_{h,e}^{\mathrm{pred}}
:=
\lambda_{\min}\!\left(
\sum_{j\in\cup_{r<e}\mathcal E_r}\underline Y_j
\right).
\label{eq:predictive_excitation_budget}
\end{equation}
Then with probability at least
\(
1-\delta
\),
simultaneously for all stages
\(
h\in[H]
\)
and epochs
\(
e\in[M]
\),
\begin{equation}
\Lambda_{W,h,e}^-
\succeq
\Bigl(
\lambda
+
\Omega_{h,e}^{\mathrm{pred}}
-
4\kappa_{W,h}^2
\sqrt{
N_e^-
\log\!\bigl(2d_{W,h}HM/\delta\bigr)
}
\Bigr)I_{d_{W,h}}.
\label{eq:predictive_gram_hp}
\end{equation}
\end{proposition}

\begin{proof}
Fix a stage
\(
h
\)
and an epoch
\(
e
\).
Define
\(
Z_j:=Y_j-\EE[Y_j\mid \mathcal H_{j-1}]
\).
Because
\(
\|Y_j\|_{\mathrm{op}}\le \kappa_{W,h}^2
\)
and
\(
\|\EE[Y_j\mid \mathcal H_{j-1}]\|_{\mathrm{op}}\le \kappa_{W,h}^2
\),
we have
\(
\|Z_j\|_{\mathrm{op}}\le 2\kappa_{W,h}^2
\).
The same matrix Azuma argument as in Proposition~\ref{prop:feature_gram_hp} therefore yields, with
probability at least
\(
1-\delta
\),
simultaneously over all stages and epochs,
\begin{equation}
\left\|
\sum_{j\in\cup_{r<e}\mathcal E_r} Z_j
\right\|_{\mathrm{op}}
\le
4\kappa_{W,h}^2
\sqrt{
N_e^-
\log\!\bigl(2d_{W,h}HM/\delta\bigr)
}.
\label{eq:predictive_gram_dev}
\end{equation}
On this event,
\begin{align*}
\Lambda_{W,h,e}^-
&=
\lambda I_{d_{W,h}}
+
\sum_{j\in\cup_{r<e}\mathcal E_r}\EE[Y_j\mid \mathcal H_{j-1}]
+
\sum_{j\in\cup_{r<e}\mathcal E_r} Z_j
\\
&\succeq
\lambda I_{d_{W,h}}
+
\sum_{j\in\cup_{r<e}\mathcal E_r}\underline Y_j
-
4\kappa_{W,h}^2
\sqrt{
N_e^-
\log\!\bigl(2d_{W,h}HM/\delta\bigr)
}
\,
I_{d_{W,h}}
\end{align*}
by \eqref{eq:predictive_excitation_cond} and \eqref{eq:predictive_gram_dev}. Taking minimum
eigenvalues proves \eqref{eq:predictive_gram_hp}.
\end{proof}

\begin{proposition}[Scalar design value implies occupancy-averaged excitation]
\label{prop:design_value_excitation}
Assume the explicit mixing rule \eqref{eq:explicit_mix_rule}. Fix a stage
\(
h\in[H]
\)
and an epoch
\(
e\in[M]
\).
For every state
\(
s
\),
define the smoothing-design covariance
\begin{equation*}
M_h(s)
:=
\EE_{(a,b)\sim u_h(\cdot\mid s)}
\sbr{
\phi_{W,h}(s,a,b)\phi_{W,h}(s,a,b)^\top
}
\end{equation*}
and its scalar design value
\begin{equation}
\omega_h(s):=\lambda_{\min}\!\bigl(M_h(s)\bigr).
\label{eq:state_design_value}
\end{equation}
Let
\(
s_h^j
\)
denote the stage-$h$ state of pre-epoch episode
\(
j\in\cup_{r<e}\mathcal E_r
\).
If there exists
\(
\bar\omega_h>0
\)
such that
\begin{equation}
\EE\!\left[
\omega_h(s_h^j)
\mid
\mathcal H_{j-1}
\right]
\ge
\bar\omega_h
\qquad
\text{for every }j\in\cup_{r<e}\mathcal E_r,
\label{eq:scalar_design_cond}
\end{equation}
then the matrix excitation condition \eqref{eq:avg_excitation_cond} holds. Consequently,
Proposition~\ref{prop:avg_excitation_budget} applies with the same constant
\(
\bar\omega_h
\).
\end{proposition}

\begin{proof}
For every state
\(
s
\),
the minimum-eigenvalue definition gives
\begin{equation*}
M_h(s)\succeq \omega_h(s) I_{d_{W,h}}.
\end{equation*}
Applying this with
\(
s=s_h^j
\)
and taking conditional expectations yields
\begin{equation*}
\EE\!\left[
M_h(s_h^j)
\mid
\mathcal H_{j-1}
\right]
\succeq
\EE\!\left[
\omega_h(s_h^j)
\mid
\mathcal H_{j-1}
\right]
I_{d_{W,h}}
\succeq
\bar\omega_h I_{d_{W,h}},
\end{equation*}
which is exactly \eqref{eq:avg_excitation_cond}. The final claim is then immediate from
Proposition~\ref{prop:avg_excitation_budget}.
\end{proof}

\begin{proposition}[Occupancy-averaged excitation implies a predictable budget]
\label{prop:avg_excitation_budget}
Assume the explicit mixing rule \eqref{eq:explicit_mix_rule}. Fix a stage
\(
h\in[H]
\)
and an epoch
\(
e\in[M]
\),
and let
\(
s_h^j
\)
denote the stage-$h$ state of pre-epoch episode
\(
j\in\cup_{r<e}\mathcal E_r
\).
Suppose the smoothing design
\(
u_h(\cdot\mid s)
\)
satisfies
\begin{equation*}
M_h(s)
:=
\EE_{(a,b)\sim u_h(\cdot\mid s)}
\sbr{
\phi_{W,h}(s,a,b)\phi_{W,h}(s,a,b)^\top
},
\end{equation*}
and that there exists
\(
\bar\omega_h>0
\)
such that
\begin{equation}
\EE\!\left[
M_h(s_h^j)
\mid
\mathcal H_{j-1}
\right]
\succeq
\bar\omega_h I_{d_{W,h}}
\qquad
\text{for every }j\in\cup_{r<e}\mathcal E_r.
\label{eq:avg_excitation_cond}
\end{equation}
Then the premise of Proposition~\ref{prop:predictive_excitation_budget} holds with
\(
\underline Y_j=\gamma_{\min}\bar\omega_h I_{d_{W,h}}.
\)
Consequently,
\begin{equation}
\Omega_{h,e}^{\mathrm{pred}}
\ge
\gamma_{\min}\bar\omega_h N_e^-.
\label{eq:avg_excitation_budget}
\end{equation}
\end{proposition}

\begin{proof}
Fix a pre-epoch episode
\(
j\in\cup_{r<e}\mathcal E_r
\).
Conditioned on
\(
s_h^j
\),
the explicit mixing rule \eqref{eq:explicit_mix_rule} gives
\begin{align*}
\EE\!\left[
Y_j
\mid
\mathcal H_{j-1},s_h^j
\right]
&=
\EE_{(a,b)\sim \widetilde\pi_h^j(\cdot\mid s_h^j)}
\sbr{
\phi_{W,h}(s_h^j,a,b)\phi_{W,h}(s_h^j,a,b)^\top
}
\\
&\succeq
\gamma_{\min}
\EE_{(a,b)\sim u_h(\cdot\mid s_h^j)}
\sbr{
\phi_{W,h}(s_h^j,a,b)\phi_{W,h}(s_h^j,a,b)^\top
}
\\
&=
\gamma_{\min} M_h(s_h^j).
\end{align*}
Taking conditional expectations given
\(
\mathcal H_{j-1}
\)
and using \eqref{eq:avg_excitation_cond} yields
\begin{equation*}
\EE\!\left[
Y_j
\mid
\mathcal H_{j-1}
\right]
\succeq
\gamma_{\min}
\EE\!\left[
M_h(s_h^j)
\mid
\mathcal H_{j-1}
\right]
\succeq
\gamma_{\min}\bar\omega_h I_{d_{W,h}}.
\end{equation*}
Thus Proposition~\ref{prop:predictive_excitation_budget} applies with
\(
\underline Y_j=\gamma_{\min}\bar\omega_h I_{d_{W,h}}
\),
and summing over the
\(
N_e^-
\)
pre-epoch episodes proves \eqref{eq:avg_excitation_budget}.
\end{proof}

\begin{corollary}[Uniform action-pair smoothing gives a concrete design score]
\label{cor:uniform_design_excitation}
Assume finite action spaces and let the smoothing design be uniform over action pairs,
\begin{equation*}
u_h^{\mathrm{unif}}(a,b\mid s):=\frac{1}{|\Acal||\Bcal|}.
\end{equation*}
For each state
\(
s
\),
form the witness-feature matrix
\begin{equation*}
\Phi_h(s)
:=
\begin{bmatrix}
\phi_{W,h}(s,a_1,b_1) &
\cdots &
\phi_{W,h}(s,a_{|\Acal|},b_{|\Bcal|})
\end{bmatrix}
\in\RR^{d_{W,h}\times |\Acal||\Bcal|},
\end{equation*}
whose columns enumerate all action pairs. Then
\begin{equation}
M_h(s)
=
\frac{1}{|\Acal||\Bcal|}
\Phi_h(s)\Phi_h(s)^\top,
\qquad
\omega_h(s)
=
\frac{1}{|\Acal||\Bcal|}
\lambda_{\min}\!\bigl(\Phi_h(s)\Phi_h(s)^\top\bigr).
\label{eq:uniform_design_value}
\end{equation}
Hence, if there exists
\(
\bar\omega_h>0
\)
such that
\begin{equation}
\EE\!\left[
\lambda_{\min}\!\bigl(\Phi_h(s_h^j)\Phi_h(s_h^j)^\top\bigr)
\mid
\mathcal H_{j-1}
\right]
\ge
|\Acal||\Bcal|\,\bar\omega_h
\label{eq:uniform_design_cond}
\end{equation}
for every pre-epoch episode
\(
j
\),
then the scalar design condition \eqref{eq:scalar_design_cond} holds and therefore
Proposition~\ref{prop:avg_excitation_budget} applies.
\end{corollary}

\begin{proof}
The identity \eqref{eq:uniform_design_value} follows directly from the definition of the uniform
distribution over action pairs. Substituting \eqref{eq:uniform_design_value} into
\eqref{eq:scalar_design_cond} shows that \eqref{eq:uniform_design_cond} implies the premise of
Proposition~\ref{prop:design_value_excitation}. The conclusion then follows immediately.
\end{proof}

\begin{proposition}[Good-state reachability implies a design-value lower bound]
\label{prop:good_state_design}
Fix a stage
\(
h\in[H]
\)
and an epoch
\(
e\in[M]
\).
Let
\(
\Gcal_h\subseteq \Scal
\)
be a measurable good-state set and suppose the scalar design value from
\eqref{eq:state_design_value} satisfies
\begin{equation}
\omega_h(s)\ge \omega_{h,\mathrm{good}}
\qquad
\text{for every }s\in\Gcal_h
\label{eq:good_state_design_value}
\end{equation}
for some constant
\(
\omega_{h,\mathrm{good}}>0
\).
If the stage-$h$ state of every pre-epoch episode obeys
\begin{equation}
\Pr\!\left(
s_h^j\in\Gcal_h
\mid
\mathcal H_{j-1}
\right)
\ge
\rho_h
\qquad
\text{for every }j\in\cup_{r<e}\mathcal E_r
\label{eq:good_state_reachability}
\end{equation}
for some
\(
\rho_h>0
\),
then the scalar design condition \eqref{eq:scalar_design_cond} holds with
\begin{equation}
\bar\omega_h=\rho_h\,\omega_{h,\mathrm{good}}.
\label{eq:good_state_design_budget}
\end{equation}
Consequently,
Proposition~\ref{prop:design_value_excitation}
and Proposition~\ref{prop:avg_excitation_budget} apply with that value of
\(
\bar\omega_h
\).
\end{proposition}

\begin{proof}
By \eqref{eq:good_state_design_value},
\begin{equation*}
\omega_h(s_h^j)
\ge
\omega_{h,\mathrm{good}}\,\mathbf 1\{s_h^j\in\Gcal_h\}.
\end{equation*}
Taking conditional expectations given
\(
\mathcal H_{j-1}
\)
and using \eqref{eq:good_state_reachability} yields
\begin{equation*}
\EE\!\left[
\omega_h(s_h^j)
\mid
\mathcal H_{j-1}
\right]
\ge
\omega_{h,\mathrm{good}}
\Pr\!\left(
s_h^j\in\Gcal_h
\mid
\mathcal H_{j-1}
\right)
\ge
\rho_h\,\omega_{h,\mathrm{good}},
\end{equation*}
which is exactly \eqref{eq:scalar_design_cond} with
\(
\bar\omega_h
\)
given by \eqref{eq:good_state_design_budget}. The remaining claim follows from
Proposition~\ref{prop:design_value_excitation} and then
Proposition~\ref{prop:avg_excitation_budget}.
\end{proof}

\begin{corollary}[Episode-level explorer branch implies reachability excitation]
\label{cor:explorer_branch_design}
Fix a stage
\(
h\in[H]
\)
and an epoch
\(
e\in[M]
\).
Assume the setting of Proposition~\ref{prop:good_state_design}, and suppose the executed
policy-selection rule is augmented by an episode-level explorer branch: before every pre-epoch
episode
\(
j\in\cup_{r<e}\mathcal E_r
\),
the algorithm draws an
\(
\mathcal H_{j-1}
\)-measurable Bernoulli variable
\(
E_j
\)
with
\begin{equation}
\Pr(E_j=1\mid \mathcal H_{j-1})\ge q_{\mathrm{exp},h},
\qquad
q_{\mathrm{exp},h}>0,
\label{eq:explorer_branch_prob}
\end{equation}
and whenever
\(
E_j=1
\)
the episode is rolled in with an explorer pair whose stage-$h$ state marginal
\(
d_{h,\mathrm{exp}}^j
\)
satisfies
\begin{equation}
d_{h,\mathrm{exp}}^j(\Gcal_h)\ge \rho_h.
\label{eq:explorer_good_state_mass}
\end{equation}
Then
\eqref{eq:good_state_reachability}
holds with
\(
\rho_h
\)
replaced by
\(
q_{\mathrm{exp},h}\rho_h
\),
and therefore
\begin{equation}
\bar\omega_h
=
q_{\mathrm{exp},h}\rho_h\,\omega_{h,\mathrm{good}}
\label{eq:explorer_design_budget}
\end{equation}
is a valid scalar design-value lower bound.
\end{corollary}

\begin{proof}
Condition on
\(
\mathcal H_{j-1}
\).
By the law of total expectation and the nonnegativity of conditional probabilities,
\begin{align*}
\Pr(s_h^j\in\Gcal_h\mid \mathcal H_{j-1})
&=
\Pr(E_j=1\mid \mathcal H_{j-1})
\Pr(s_h^j\in\Gcal_h\mid \mathcal H_{j-1},E_j=1)
\\
&\quad+
\Pr(E_j=0\mid \mathcal H_{j-1})
\Pr(s_h^j\in\Gcal_h\mid \mathcal H_{j-1},E_j=0)
\\
&\ge
\Pr(E_j=1\mid \mathcal H_{j-1})
\Pr(s_h^j\in\Gcal_h\mid \mathcal H_{j-1},E_j=1).
\end{align*}
Under the explorer branch,
\(
\Pr(s_h^j\in\Gcal_h\mid \mathcal H_{j-1},E_j=1)=d_{h,\mathrm{exp}}^j(\Gcal_h)\ge \rho_h
\)
by \eqref{eq:explorer_good_state_mass}, while \eqref{eq:explorer_branch_prob} gives
\(
\Pr(E_j=1\mid \mathcal H_{j-1})\ge q_{\mathrm{exp},h}
\).
Hence
\begin{equation*}
\Pr(s_h^j\in\Gcal_h\mid \mathcal H_{j-1})
\ge
q_{\mathrm{exp},h}\rho_h.
\end{equation*}
This is exactly \eqref{eq:good_state_reachability} with
\(
\rho_h
\)
replaced by
\(
q_{\mathrm{exp},h}\rho_h
\).
Applying Proposition~\ref{prop:good_state_design} proves \eqref{eq:explorer_design_budget}.
\end{proof}

\begin{proposition}[Existing exploiter-mix implies reachability from BR component dynamics]
\label{prop:br_mix_reachability}
Fix a stage
\(
h\in[H]
\)
and an epoch
\(
e\in[M]
\).
Assume the executed pair in every pre-epoch episode
\(
j\in\cup_{r<e}\mathcal E_r
\)
is drawn according to the existing episode-level exploiter-mix rule used in
Proposition~\ref{prop:mix_regret}: with conditional probabilities
\(
1-p
\),
\(
p/2
\),
and
\(
p/2
\),
the episode executes one of
\(
(\pi^j,\nu^j)
\),
\(
(\mathrm{BR}_1^j,\nu^j)
\),
and
\(
(\pi^j,\mathrm{BR}_2^j)
\),
respectively.
Let
\(
d_{h,\mathrm{tar}}^j,
d_{h,\mathrm{br1}}^j,
d_{h,\mathrm{br2}}^j
\)
denote the corresponding stage-$h$ state marginals under the interaction process.
Suppose there exists a measurable good-state set
\(
\Gcal_h\subseteq\Scal
\)
such that, for every pre-epoch episode
\(
j
\),
\begin{equation}
d_{h,\mathrm{tar}}^j(\Gcal_h)\ge \rho_{h,\mathrm{tar}},
\qquad
d_{h,\mathrm{br1}}^j(\Gcal_h)\ge \rho_{h,\mathrm{br1}},
\qquad
d_{h,\mathrm{br2}}^j(\Gcal_h)\ge \rho_{h,\mathrm{br2}},
\label{eq:branch_good_state_mass}
\end{equation}
with deterministic nonnegative constants
\(
\rho_{h,\mathrm{tar}},\rho_{h,\mathrm{br1}},\rho_{h,\mathrm{br2}}
\).
Then the stage-$h$ state of the actually executed episode satisfies
\begin{equation}
\Pr\!\left(
s_h^j\in\Gcal_h
\mid
\mathcal H_{j-1}
\right)
\ge
(1-p)\rho_{h,\mathrm{tar}}
+
\frac p2 \rho_{h,\mathrm{br1}}
+
\frac p2 \rho_{h,\mathrm{br2}}.
\label{eq:br_mix_good_state_mass}
\end{equation}
If in addition
\(
\omega_h(s)\ge \omega_{h,\mathrm{good}}
\)
for all
\(
s\in\Gcal_h
\),
then Proposition~\ref{prop:good_state_design} applies with
\begin{equation*}
\bar\omega_h
=
\Bigl(
(1-p)\rho_{h,\mathrm{tar}}
+
\frac p2 \rho_{h,\mathrm{br1}}
+
\frac p2 \rho_{h,\mathrm{br2}}
\Bigr)\omega_{h,\mathrm{good}}.
\end{equation*}
\end{proposition}

\begin{proof}
Condition on
\(
\mathcal H_{j-1}
\).
Let
\(
B_j\in\{\mathrm{tar},\mathrm{br1},\mathrm{br2}\}
\)
denote the branch selected by the exploiter-mix rule in episode
\(
j
\).
By the law of total probability,
\begin{align*}
\Pr(s_h^j\in\Gcal_h\mid \mathcal H_{j-1})
&=
\sum_{b\in\{\mathrm{tar},\mathrm{br1},\mathrm{br2}\}}
\Pr(B_j=b\mid \mathcal H_{j-1})
\Pr(s_h^j\in\Gcal_h\mid \mathcal H_{j-1},B_j=b)
\\
&=
(1-p)\,d_{h,\mathrm{tar}}^j(\Gcal_h)
+
\frac p2\,d_{h,\mathrm{br1}}^j(\Gcal_h)
+
\frac p2\,d_{h,\mathrm{br2}}^j(\Gcal_h).
\end{align*}
Using \eqref{eq:branch_good_state_mass} yields \eqref{eq:br_mix_good_state_mass}. The design-value
claim then follows from Proposition~\ref{prop:good_state_design}.
\end{proof}

\begin{corollary}[BR-only reachability lower bound]
\label{cor:br_only_reachability}
Under the assumptions of Proposition~\ref{prop:br_mix_reachability}, if no lower bound is assumed
for the target branch but both BR branches satisfy
\(
\rho_{h,\mathrm{br1}},\rho_{h,\mathrm{br2}}\ge \rho_{h,\mathrm{BR}}
\)
for some
\(
\rho_{h,\mathrm{BR}}>0
\),
then
\begin{equation*}
\Pr\!\left(
s_h^j\in\Gcal_h
\mid
\mathcal H_{j-1}
\right)
\ge
p\,\rho_{h,\mathrm{BR}}
\end{equation*}
and therefore
\begin{equation*}
\bar\omega_h
\ge
p\,\rho_{h,\mathrm{BR}}\,\omega_{h,\mathrm{good}}.
\end{equation*}
\end{corollary}

\begin{proof}
Set
\(
\rho_{h,\mathrm{tar}}=0
\)
in Proposition~\ref{prop:br_mix_reachability} and use
\(
\rho_{h,\mathrm{br1}},\rho_{h,\mathrm{br2}}\ge \rho_{h,\mathrm{BR}}
\).
\end{proof}

\paragraph{Primitive BR-mix coverage from reachability and local nondegeneracy.}
Assume finite action spaces and the explicit smoothing rule \eqref{eq:explicit_mix_rule} with the
uniform action-pair design
\(
u_h^{\mathrm{unif}}
\)
from Corollary~\ref{cor:uniform_design_excitation}. Fix a stage
\(
h\in[H]
\)
and an epoch
\(
e\in[M]
\).
Suppose there exists a measurable good-state set
\(
\Gcal_h\subseteq\Scal
\),
a stagewise local nondegeneracy constant
\(
s_h>0
\),
and branch-mass constants
\(
\rho_{h,\mathrm{tar}},\rho_{h,\mathrm{br1}},\rho_{h,\mathrm{br2}}\ge 0
\)
such that:
\begin{enumerate}[label=(\roman*),leftmargin=*]
\item for every
\(
s\in\Gcal_h
\),
\begin{equation}
\lambda_{\min}\!\bigl(\Phi_h(s)\Phi_h(s)^\top\bigr)\ge s_h^2,
\label{eq:primitive_local_nondeg}
\end{equation}
where
\(
\Phi_h(s)
\)
is the witness-feature matrix from \eqref{eq:uniform_design_value};
\item the current BR-based exploiter-mix obeys the branch good-state lower bounds
\eqref{eq:branch_good_state_mass}.
\end{enumerate}
Define
\begin{equation}
\alpha_h
:=
(1-p)\rho_{h,\mathrm{tar}}
+
\frac p2\rho_{h,\mathrm{br1}}
+
\frac p2\rho_{h,\mathrm{br2}}.
\label{eq:primitive_alpha}
\end{equation}
Then the predictable excitation budget satisfies
\begin{equation}
\Omega_{h,e}^{\mathrm{pred}}
\ge
\gamma_{\min}\,
\frac{\alpha_h s_h^2}{|\Acal||\Bcal|}
\,N_e^-.
\label{eq:primitive_pred_budget}
\end{equation}
Consequently, the predictable-growth premise \eqref{eq:predictive_linear_growth} of
Corollary~\ref{cor:linear_formal_burnin} holds with
\begin{equation}
\underline\nu_h
=
\gamma_{\min}\,
\frac{\alpha_h s_h^2}{|\Acal||\Bcal|}.
\label{eq:primitive_nu}
\end{equation}
In the BR-only case of Corollary~\ref{cor:br_only_reachability},
\(
\alpha_h
\)
can be replaced by
\(
p\rho_{h,\mathrm{BR}}
\).
This follows because \eqref{eq:primitive_local_nondeg} and \eqref{eq:uniform_design_value} imply
\(
\omega_h(s)\ge s_h^2/(|\Acal||\Bcal|)
\)
for every
\(
s\in\Gcal_h
\).
Proposition~\ref{prop:br_mix_reachability} and Proposition~\ref{prop:good_state_design} then yield
\(
\bar\omega_h=\alpha_h s_h^2/(|\Acal||\Bcal|)
\),
and Proposition~\ref{prop:avg_excitation_budget} converts this scalar design-value lower bound into
\eqref{eq:primitive_pred_budget} and \eqref{eq:primitive_nu}.

\begin{assumption}[Uniform ridge-linear dual specialization over the linear value class]
\label{asmp:linear_ridge_dual}
For each stage
\(
h\in[H]
\),
consider the full linear value class
\(
\Fcal_{h+1}
\)
equipped with the metric
\(
d_{\Fcal,h+1}
\)
from Assumption~\ref{asmp:lipschitz_dual}.
For every epoch
\(
e
\),
stage
\(
h
\),
and continuation function
\(
f_{h+1}\in \Fcal_{h+1}
\),
there exists a parameter
\(
\vartheta_h^\star(f_{h+1})\in\RR^{d_{W,h}}
\)
with
\(
\|\vartheta_h^\star(f_{h+1})\|_2\le S_{G,h}
\)
and scalar pseudo-labels
\(
\{Y_{h,j}(f_{h+1})\}_{j\in\cup_{u<e}\mathcal E_u}
\)
such that
\begin{equation}
Y_{h,j}(f_{h+1})
=
\left\langle
\vartheta_h^\star(f_{h+1}),
\phi_{W,h}(X_h^j)
\right\rangle
+
\eta_{h,j}(f_{h+1}),
\label{eq:ridge_label_model}
\end{equation}
where
\(
\eta_{h,j}(f_{h+1})
\)
is conditionally mean zero and
\(
R_{G,h}
\)-sub-Gaussian given
\(
\mathcal H_{j-1}
\).
Assume moreover that the pseudo-label and target-parameter maps are uniformly Lipschitz:
\begin{align}
\bigl|
Y_{h,j}(f_{h+1}) - Y_{h,j}(\tilde f_{h+1})
\bigr|
&\le
L_{Y,h}\,d_{\Fcal,h+1}(f_{h+1},\tilde f_{h+1}),
\label{eq:ridge_label_lipschitz}
\\
\|
\vartheta_h^\star(f_{h+1})-\vartheta_h^\star(\tilde f_{h+1})
\|_2
&\le
L_{\star,h}\,d_{\Fcal,h+1}(f_{h+1},\tilde f_{h+1})
\label{eq:ridge_target_lipschitz}
\end{align}
for all
\(
f_{h+1},\tilde f_{h+1}\in\Fcal_{h+1}
\).
The epochwise dual refit is the ridge estimator
\begin{equation}
\widehat\vartheta_{h,e}(f_{h+1})
\in
\argmin_{\vartheta\in\RR^{d_{W,h}}}
\left\{
\sum_{j\in\cup_{u<e}\mathcal E_u}
\Bigl(
Y_{h,j}(f_{h+1})
-
\langle \vartheta,\phi_{W,h}(X_h^j)\rangle
\Bigr)^2
+
\lambda\|\vartheta\|_2^2
\right\},
\label{eq:ridge_dual_fit}
\end{equation}
and the induced operator gap satisfies
\begin{equation}
\bigl(
\widehat{\Tcal}_{h,f_{h+1}}^{\sigma,(e)}f_{h+1}
- \Tcal_h^\sigma f_{h+1}
\bigr)(x)
=
\left\langle
\widehat\vartheta_{h,e}(f_{h+1})-\vartheta_h^\star(f_{h+1}),
\phi_{W,h}(x)
\right\rangle.
\label{eq:ridge_operator_gap}
\end{equation}
\end{assumption}

\begin{proposition}[Affine continuation coupling removes separate Lipschitz inputs]
\label{prop:affine_ridge_specialization}
Fix a stage
\(
h\in[H]
\)
and suppose the continuation class admits the parameterization
\begin{equation*}
\Fcal_{h+1}
=
\cbr{
f_{h+1,\alpha}:x\mapsto \langle \alpha,\phi_{F,h+1}(x)\rangle :
\alpha\in\RR^{d_{F,h}},\ \|\alpha\|_2\le R_{F,h}
}.
\end{equation*}
Assume that for every pre-epoch sample index
\(
j\in\cup_{u<e}\mathcal E_u
\),
there exist a scalar
\(
y_{h,j}^0
\),
a vector
\(
a_{h,j}\in\RR^{d_{F,h}}
\),
a parameter
\(
\vartheta_{h,0}^\star\in\RR^{d_{W,h}}
\),
and a matrix
\(
B_h\in\RR^{d_{W,h}\times d_{F,h}}
\)
such that
\begin{align}
Y_{h,j}(f_{h+1,\alpha})
&=
y_{h,j}^0 + \langle a_{h,j},\alpha\rangle + \eta_{h,j}(\alpha),
\label{eq:affine_label_model}
\\
\vartheta_h^\star(f_{h+1,\alpha})
&=
\vartheta_{h,0}^\star + B_h\alpha,
\label{eq:affine_target_model}
\end{align}
where
\(
\eta_{h,j}(\alpha)
\)
is conditionally mean zero and
\(
R_{G,h}
\)-sub-Gaussian given
\(
\mathcal H_{j-1}
\)
for every
\(
\alpha
\).
If
\begin{equation*}
\sup_{j\in\cup_{u<e}\mathcal E_u}\|a_{h,j}\|_2 \le A_{Y,h},
\qquad
\|B_h\|_{\mathrm{op}}\le A_{\star,h},
\end{equation*}
then Assumption~\ref{asmp:linear_ridge_dual} holds with the parameter metric
\(
d_{\Fcal,h+1}(f_{h+1,\alpha},f_{h+1,\tilde\alpha})
:=
\|\alpha-\tilde\alpha\|_2
\),
the radius choice
\(
S_{G,h}=\|\vartheta_{h,0}^\star\|_2 + A_{\star,h}R_{F,h}
\),
and the derived continuity constants
\(
L_{Y,h}=A_{Y,h}
\)
and
\(
L_{\star,h}=A_{\star,h}
\).
\end{proposition}

\begin{proof}
Fix any
\(
\alpha,\tilde\alpha\in\RR^{d_{F,h}}
\)
with
\(
\|\alpha\|_2,\|\tilde\alpha\|_2\le R_{F,h}
\).
By \eqref{eq:affine_target_model},
\begin{equation*}
\|
\vartheta_h^\star(f_{h+1,\alpha})
\|_2
\le
\|\vartheta_{h,0}^\star\|_2 + \|B_h\|_{\mathrm{op}}\|\alpha\|_2
\le
\|\vartheta_{h,0}^\star\|_2 + A_{\star,h}R_{F,h}
=
S_{G,h}.
\end{equation*}
Also, \eqref{eq:affine_label_model} gives
\begin{equation*}
\bigl|
Y_{h,j}(f_{h+1,\alpha})-Y_{h,j}(f_{h+1,\tilde\alpha})
\bigr|
=
\bigl|
\langle a_{h,j},\alpha-\tilde\alpha\rangle
\bigr|
\le
A_{Y,h}\|\alpha-\tilde\alpha\|_2,
\end{equation*}
while \eqref{eq:affine_target_model} yields
\begin{equation*}
\|
\vartheta_h^\star(f_{h+1,\alpha})
-
\vartheta_h^\star(f_{h+1,\tilde\alpha})
\|_2
=
\|B_h(\alpha-\tilde\alpha)\|_2
\le
A_{\star,h}\|\alpha-\tilde\alpha\|_2.
\end{equation*}
These are exactly \eqref{eq:ridge_label_lipschitz} and \eqref{eq:ridge_target_lipschitz} under the
parameter metric. The remaining parts of Assumption~\ref{asmp:linear_ridge_dual} are already built
into \eqref{eq:affine_label_model}--\eqref{eq:affine_target_model}.
\end{proof}

\begin{proposition}[Uniform self-normalized ridge control over the value class]
\label{prop:ridge_selfnorm}
Assume Assumption~\ref{asmp:linear_ridge_dual}. Then with probability at least
\(
1-\delta
\),
simultaneously for all stages
\(
h\in[H]
\),
epochs
\(
e\in[M]
\),
and continuation functions
\(
f_{h+1}\in\Fcal_{h+1}
\),
\begin{equation}
\left\|
\widehat\vartheta_{h,e}(f_{h+1})-\vartheta_h^\star(f_{h+1})
\right\|_{\Lambda_{W,h,e}^-}
\leq
\beta_{h,e}^{\mathrm{ridge}}(\delta),
\label{eq:ridge_selfnorm}
\end{equation}
where
\begin{equation}
\begin{aligned}
\beta_{h,e}^{\mathrm{ridge}}(\delta)
\;:=\;&
R_{G,h}
\sqrt{
2d_{W,h}
\log\!\left(
1+\frac{\kappa_{W,h}^2N_e^-}{\lambda}
\right)
+
2\log\!\left(
\frac{2HM\mathcal N_F(\eta_{F,e}^{\mathrm{ridge}})}{\delta}
\right)
}
\\
&\quad
+ \sqrt{\lambda}\,S_{G,h}
+ L_{Y,h}
+ (\sqrt{\lambda}+\kappa_{W,h})L_{\star,h},
\end{aligned}
\label{eq:ridge_beta}
\end{equation}
and
\(
\eta_{F,e}^{\mathrm{ridge}}:=(N_e^-\vee 1)^{-1/2}
\).
\end{proposition}

\begin{proof}
Fix a stage
\(
h
\),
an epoch
\(
e
\),
and let
\(
\Fcal_{h+1}^{\sharp,e}
\subseteq \Fcal_{h+1}
\)
be an
\(
\eta_{F,e}^{\mathrm{ridge}}
\)-net under
\(
d_{\Fcal,h+1}
\)
with cardinality at most
\(
\mathcal N_F(\eta_{F,e}^{\mathrm{ridge}})
\).
Fix a net point
\(
\bar f_{h+1}\in\Fcal_{h+1}^{\sharp,e}
\).
By Assumption~\ref{asmp:linear_ridge_dual}, the pre-epoch samples satisfy the standard linear
regression model \eqref{eq:ridge_label_model} with design vectors
\(
\phi_{W,h}(X_h^j)
\),
parameter
\(
\vartheta_h^\star(\bar f_{h+1})
\),
and conditionally
\(
R_{G,h}
\)-sub-Gaussian noise. Applying the standard self-normalized bound for ridge regression to the
estimator \eqref{eq:ridge_dual_fit} yields, with probability at least
\(
1-\delta'
\),
\begin{equation*}
\left\|
\widehat\vartheta_{h,e}(\bar f_{h+1})-\vartheta_h^\star(\bar f_{h+1})
\right\|_{\Lambda_{W,h,e}^-}
\leq
R_{G,h}
\sqrt{
2\log\!\left(
\frac{\det(\Lambda_{W,h,e}^-)^{1/2}}
{\delta' \det(\lambda I_{d_{W,h}})^{1/2}}
\right)
}
+
\sqrt{\lambda}\,S_{G,h}.
\end{equation*}
Because
\(
\Lambda_{W,h,e}^-
\preceq
(\lambda+\kappa_{W,h}^2N_e^-)I_{d_{W,h}},
\)
we have
\begin{equation*}
\log\frac{\det(\Lambda_{W,h,e}^-)^{1/2}}
{\det(\lambda I_{d_{W,h}})^{1/2}}
\leq
\frac{d_{W,h}}{2}
\log\!\left(
1+\frac{\kappa_{W,h}^2N_e^-}{\lambda}
\right).
\end{equation*}
Choosing
\(
\delta'=\delta/(HM\mathcal N_F(\eta_{F,e}^{\mathrm{ridge}}))
\)
and union bounding over all stages, epochs, and net points gives the stated bound simultaneously
for every
\(
\bar f_{h+1}\in\Fcal_{h+1}^{\sharp,e}
\).

Now fix an arbitrary
\(
f_{h+1}\in\Fcal_{h+1}
\)
and choose
\(
\bar f_{h+1}\in\Fcal_{h+1}^{\sharp,e}
\)
such that
\(
d_{\Fcal,h+1}(f_{h+1},\bar f_{h+1})\le \eta_{F,e}^{\mathrm{ridge}}
\).
By the ridge normal equations,
\begin{equation*}
\widehat\vartheta_{h,e}(f_{h+1})
=
\bigl(\Lambda_{W,h,e}^-\bigr)^{-1}
\sum_{j\in\cup_{u<e}\mathcal E_u}
\phi_{W,h}(X_h^j)\,Y_{h,j}(f_{h+1}),
\end{equation*}
so
\begin{equation*}
\widehat\vartheta_{h,e}(f_{h+1})-\widehat\vartheta_{h,e}(\bar f_{h+1})
=
\bigl(\Lambda_{W,h,e}^-\bigr)^{-1}
\sum_{j\in\cup_{u<e}\mathcal E_u}
\phi_{W,h}(X_h^j)
\Bigl(
Y_{h,j}(f_{h+1})-Y_{h,j}(\bar f_{h+1})
\Bigr).
\end{equation*}
Let
\(
\Phib_{W,h,e}^-
\)
be the design matrix whose rows are
\(
\phi_{W,h}(X_h^j)^\top
\).
Then
\begin{equation*}
\left\|
\widehat\vartheta_{h,e}(f_{h+1})-\widehat\vartheta_{h,e}(\bar f_{h+1})
\right\|_{\Lambda_{W,h,e}^-}^2
=
\Delta y^\top
\Phib_{W,h,e}^-
\bigl(\Lambda_{W,h,e}^-\bigr)^{-1}
\bigl(\Phib_{W,h,e}^-\bigr)^\top
\Delta y
\le
\|\Delta y\|_2^2,
\end{equation*}
where
\(
\Delta y_j:=Y_{h,j}(f_{h+1})-Y_{h,j}(\bar f_{h+1})
\),
and the last step uses
\(
\Phib_{W,h,e}^-
\bigl(\Lambda_{W,h,e}^-\bigr)^{-1}
\bigl(\Phib_{W,h,e}^-\bigr)^\top
\preceq I
\).
By \eqref{eq:ridge_label_lipschitz},
\begin{equation*}
\|\Delta y\|_2
\le
L_{Y,h}\sqrt{N_e^-}\,\eta_{F,e}^{\mathrm{ridge}}
\le
L_{Y,h}.
\end{equation*}
Also, by \eqref{eq:ridge_target_lipschitz},
\begin{align*}
\left\|
\vartheta_h^\star(f_{h+1})-\vartheta_h^\star(\bar f_{h+1})
\right\|_{\Lambda_{W,h,e}^-}
&\le
\left\|
\vartheta_h^\star(f_{h+1})-\vartheta_h^\star(\bar f_{h+1})
\right\|_2
\sqrt{\|\Lambda_{W,h,e}^-\|_{\mathrm{op}}}
\\
&\le
L_{\star,h}\eta_{F,e}^{\mathrm{ridge}}
\sqrt{\lambda+\kappa_{W,h}^2N_e^-}
\\
&\le
(\sqrt{\lambda}+\kappa_{W,h})L_{\star,h}.
\end{align*}
Combining these two displays with the net-point self-normalized bound and the triangle inequality
proves \eqref{eq:ridge_selfnorm} with \eqref{eq:ridge_beta}.
\end{proof}

\begin{corollary}[Ridge operator transfer under feature excitation]
\label{cor:ridge_gram_operator}
Assume Assumption~\ref{asmp:linear_ridge_dual} and the warm-start condition
\eqref{eq:feature_gram_warm}. Then with probability at least
\(
1-\delta
\),
simultaneously for all stages
\(
h\in[H]
\),
epochs
\(
e\in[M]
\),
continuation functions
\(
f_{h+1}\in\Fcal_{h+1}
\),
and executed episodes
\(
k\in\mathcal E_e
\),
\begin{equation}
\left\|
\widehat{\Tcal}_{h,f_{h+1}}^{\sigma,(e)}f_{h+1}
- \Tcal_h^\sigma f_{h+1}
\right\|_{1,\widetilde\mu_h^k}
\leq
\bar\rho_{h,e}^{\mathrm{ridge}}(\delta),
\label{eq:ridge_operator_transfer}
\end{equation}
where
\begin{equation*}
\bar\rho_{h,e}^{\mathrm{ridge}}(\delta)
:=
\frac{\kappa_{W,h}}{
\sqrt{\lambda+\nu_h N_e^-}
}
\,
\beta_{h,e}^{\mathrm{ridge}}(\delta),
\qquad
\nu_h:=\gamma_{\min}\omega_h/2.
\end{equation*}
\end{corollary}

\begin{proof}
Under \eqref{eq:feature_gram_warm}, Proposition~\ref{prop:feature_gram_hp} gives the Gram event
\(
\Omega_{h,e}^{\mathrm{Gram}}(\nu_h)
\)
with
\(
\nu_h=\gamma_{\min}\omega_h/2
\).
Assumption~\ref{asmp:linear_ridge_dual} identifies the operator-gap parameter in
\eqref{eq:ridge_operator_gap}, and Proposition~\ref{prop:ridge_selfnorm} controls that parameter in
the pre-epoch Gram norm. Applying Corollary~\ref{cor:gram_operator_transfer} with
\(
\vartheta_{h,e}(f_{h+1})
=
\widehat\vartheta_{h,e}(f_{h+1})-\vartheta_h^\star(f_{h+1})
\)
therefore yields \eqref{eq:ridge_operator_transfer}.
\end{proof}

\begin{corollary}[Warm-start-free ridge operator transfer under predictable excitation]
\label{cor:predictive_ridge_operator}
Assume Assumption~\ref{asmp:linear_ridge_dual} and the conditions of
Proposition~\ref{prop:predictive_excitation_budget}. Then with probability at least
\(
1-\delta
\),
simultaneously for all stages
\(
h\in[H]
\),
epochs
\(
e\in[M]
\),
continuation functions
\(
f_{h+1}\in\Fcal_{h+1}
\),
and executed episodes
\(
k\in\mathcal E_e
\),
\begin{equation}
\left\|
\widehat{\Tcal}_{h,f_{h+1}}^{\sigma,(e)}f_{h+1}
- \Tcal_h^\sigma f_{h+1}
\right\|_{1,\widetilde\mu_h^k}
\leq
\bar\rho_{h,e}^{\mathrm{pred}}(\delta),
\label{eq:predictive_ridge_operator}
\end{equation}
where
\begin{equation*}
\begin{aligned}
\bar\rho_{h,e}^{\mathrm{pred}}(\delta)
&:=
\frac{\kappa_{W,h}}{\sqrt{\underline\Lambda_{h,e}^{\mathrm{pred}}(\delta)}}
\,
\beta_{h,e}^{\mathrm{ridge}}(\delta),
\\
\underline\Lambda_{h,e}^{\mathrm{pred}}(\delta)
&:=
\lambda
+
\Bigl[
\Omega_{h,e}^{\mathrm{pred}}
-
4\kappa_{W,h}^2
\sqrt{
N_e^-
\log\!\bigl(2d_{W,h}HM/\delta\bigr)
}
\Bigr]_+.
\end{aligned}
\end{equation*}
\end{corollary}

\begin{proof}
Proposition~\ref{prop:predictive_excitation_budget} implies
\(
\Lambda_{W,h,e}^-
\succeq
\underline\Lambda_{h,e}^{\mathrm{pred}}(\delta)I_{d_{W,h}}
\)
on the stated high-probability event. Combining this with
Proposition~\ref{prop:ridge_selfnorm} and Corollary~\ref{cor:gram_operator_transfer} gives
\eqref{eq:predictive_ridge_operator}.
\end{proof}

\begin{proposition}[Projected operator transfer with an approximation remainder]
\label{prop:projected_operator_transfer}
Fix a stage
\(
h\in[H]
\),
an epoch
\(
e\in[M]
\),
and a continuation function
\(
f_{h+1}\in\Fcal_{h+1}
\).
Let
\(
\mathsf P_h:L_\infty(\Scal\times\Acal\times\Bcal)\to
L_\infty(\Scal\times\Acal\times\Bcal)
\)
be such that, for every executed episode
\(
k\in\mathcal E_e
\),
\begin{equation}
\|
\mathsf P_h u-\mathsf P_h v
\|_{1,\widetilde\mu_h^k}
\le
\|u-v\|_{1,\widetilde\mu_h^k}
\qquad
\text{for all }u,v.
\label{eq:projector_nonexpansive}
\end{equation}
Suppose the stagewise operator gap admits the decomposition
\begin{equation}
\widehat{\Tcal}_{h,f_{h+1}}^{\sigma,(e)}f_{h+1}
-
\Tcal_h^\sigma f_{h+1}
=
\zeta_{h,e}(f_{h+1})
+
r_{h,e}(f_{h+1}),
\label{eq:operator_gap_remainder}
\end{equation}
and that for every executed episode
\(
k\in\mathcal E_e
\),
\begin{equation}
\|\zeta_{h,e}(f_{h+1})\|_{1,\widetilde\mu_h^k}
\le
\bar\rho_{h,e},
\qquad
\|r_{h,e}(f_{h+1})\|_{1,\widetilde\mu_h^k}
\le
\epsilon_{\mathrm{op,proj}}.
\label{eq:operator_gap_remainder_budget}
\end{equation}
Then
\begin{equation}
\left\|
\mathsf P_h\widehat{\Tcal}_{h,f_{h+1}}^{\sigma,(e)}f_{h+1}
-
\mathsf P_h\Tcal_h^\sigma f_{h+1}
\right\|_{1,\widetilde\mu_h^k}
\le
\bar\rho_{h,e}
+
\epsilon_{\mathrm{op,proj}}
\qquad
\text{for every }k\in\mathcal E_e.
\label{eq:projected_operator_gap}
\end{equation}
\end{proposition}

\begin{proof}
Fix an executed episode
\(
k\in\mathcal E_e
\).
By \eqref{eq:projector_nonexpansive},
\begin{align*}
\left\|
\mathsf P_h\widehat{\Tcal}_{h,f_{h+1}}^{\sigma,(e)}f_{h+1}
-
\mathsf P_h\Tcal_h^\sigma f_{h+1}
\right\|_{1,\widetilde\mu_h^k}
&\le
\left\|
\widehat{\Tcal}_{h,f_{h+1}}^{\sigma,(e)}f_{h+1}
-
\Tcal_h^\sigma f_{h+1}
\right\|_{1,\widetilde\mu_h^k}
\\
&\le
\|\zeta_{h,e}(f_{h+1})\|_{1,\widetilde\mu_h^k}
+
\|r_{h,e}(f_{h+1})\|_{1,\widetilde\mu_h^k}
\\
&\le
\bar\rho_{h,e}
+
\epsilon_{\mathrm{op,proj}},
\end{align*}
where the second step uses \eqref{eq:operator_gap_remainder} and the triangle inequality, and the
last step uses \eqref{eq:operator_gap_remainder_budget}. This proves
\eqref{eq:projected_operator_gap}.
\end{proof}

\begin{proposition}[Epochwise / cross-fit BV calibration for linear classes]
\label{prop:epochwise_linear_bv}
Assume support control, Assumption~\ref{asmp:ref_opt},
Assumption~\ref{asmp:linear_closure}, Assumption~\ref{asmp:linear_ridge_dual}, and the warm-start
condition \eqref{eq:feature_gram_warm}. Consider the epochwise / cross-fit variant described
above, and let
\(
\mathfrak F_e
\)
denote the confidence set certified for deployment on epoch
\(
e
\).
Define
\begin{align*}
\Psi_{\mathrm{dual}}
\;:=\;
\max_{h\in[H]}
\Bigl[
R_{G,h}
+
\sqrt{\lambda}S_{G,h}
+
L_{Y,h}
+
(\sqrt{\lambda}+\kappa_{W,h})L_{\star,h}
\Bigr],
\qquad
\nu_{\min}
\;:=\;
\min_{h\in[H]} \frac{\gamma_{\min}\omega_h}{2},
\end{align*}
and
\begin{equation*}
B_{\mathrm{lin}}^{\mathrm{ridge}}
:=
B_F
+
\frac{\kappa_{\mathrm{lin}}\Psi_{\mathrm{dual}}}{\sqrt{\nu_{\min}}}.
\end{equation*}
Then there exist stagewise radii
\(
\beta_{h,e}^{\mathrm{lin}}
\)
depending only on the pre-epoch sample size
\(
N_e^-
\),
the linear dimensions, and logarithmic confidence terms such that, with probability at least
\(
1-\delta
\),
simultaneously for all epochs
\(
e\in[M]
\)
and all stages
\(
h\in[H]
\),
\begin{enumerate}[leftmargin=*]
\item the optimism reference function is retained in
\(
\mathfrak F_e
\);
\item every function deployed on epoch
\(
e
\)
satisfies an exact residual budget of order
\begin{equation}
\sum_{k\in\mathcal E_e}
\EE_{\mu_h^k}\sbr{\bigl(\delta_h^k\bigr)^2}
\leq
\widetilde O\!\left(
\bigl(B_{\mathrm{lin}}^{\mathrm{ridge}}\bigr)^2 d_{\mathrm{lin}}
\log\!\bigl(2+\kappa_{\mathrm{lin}}^2R_{\mathrm{lin}}^2K\bigr)
\right);
\label{eq:epochwise_linear_stage_budget}
\end{equation}
\item consequently Assumptions~\ref{asmp:opt_appendix} and \ref{asmp:res_appendix} hold for the
deployed epochwise sequence with
\begin{equation}
\Gamma_{\mathrm{res}}^{\mathrm{ep\text{-}lin}}(K)
\leq
\widetilde O\!\left(
H \bigl(B_{\mathrm{lin}}^{\mathrm{ridge}}\bigr)^2 d_{\mathrm{lin}}
\log\!\bigl(2+\kappa_{\mathrm{lin}}^2R_{\mathrm{lin}}^2K\bigr)
\log K
\right).
\label{eq:epochwise_linear_residual_budget}
\end{equation}
\end{enumerate}
\end{proposition}

\begin{proof}
Fix a stage
\(
h
\)
and an epoch
\(
e
\).
Let
\(
\mathcal H_e^-
\)
denote the sigma-field generated by all samples collected before epoch
\(
e
\).
By construction of the epochwise / cross-fit scheme, the pre-epoch dual refit and the induced
plug-in witness family are
\(
\mathcal H_e^-
\)-measurable because they are computed only from pre-epoch data. For every
\(
k\in\mathcal E_e
\),
let
\(
\widehat{\mathrm{BV}}_{h,e,k},
\mathrm{BV}_{h,e,k}^{\mathrm{plug}},
\mathrm{BV}_{h,e,k}^{\mathrm{exact}}
\)
denote the epoch-local backward-validation statistics obtained by restricting the sums in
\eqref{eq:bv_dev}--\eqref{eq:bv_exact_bound} to the fresh prefix
\(
\{j\in\mathcal E_e:j<k\}
\)
and by using the frozen epoch-\(
e
\) plug-in operator.

Conditional on
\(
\mathcal H_e^-
\),
the only randomness in these epoch-local statistics comes from fresh within-epoch samples. The
proof of Lemma~\ref{lem:bv_dev} can therefore be replayed after restricting the sample index set to
the fresh epoch block. Concretely, fix
\(
k\in\mathcal E_e
\)
with nonempty fresh prefix
\(
\mathcal J_{e,k}:=\{j\in\mathcal E_e:j<k\}
\),
and let
\(
\{w^1,\ldots,w^M\}
\)
be an
\(
\eta_{W,e}
\)-net of the frozen epoch-\(
e
\) witness class under
\(
\|\cdot\|_\infty
\).
For a fixed net witness
\(
w^m
\)
and each
\(
j\in\mathcal J_{e,k}
\),
define
\begin{equation*}
Y_j^{(m,k)}
:=
\frac{
\bigl(w^m(X_h^j)\bigr)^2-\EE_{\mu_h^j}\sbr{(w^m)^2}
}{
|\mathcal J_{e,k}|
},
\end{equation*}
where
\(
X_h^j=(S_h^j,A_h^j,B_h^j)
\)
is the realized stage-\(
h
\) tuple in episode
\(
j
\).
Because the frozen witness family is
\(
\mathcal H_e^-
\)-measurable and the stage-\(
h
\) sample of episode
\(
j
\)
is drawn from
\(
\mu_h^j
\),
we have
\(
\EE[Y_j^{(m,k)}\mid \mathcal H_{j-1}]=0
\)
and
\(
|Y_j^{(m,k)}|\le B_F^2/|\mathcal J_{e,k}|
\).
Azuma--Hoeffding therefore gives
\begin{equation*}
\PP\!\left(
\left|
\sum_{j\in\mathcal J_{e,k}} Y_j^{(m,k)}
\right|>t
\;\middle|\;
\mathcal H_e^-
\right)
\le
2\exp\!\left(
-\frac{|\mathcal J_{e,k}|\,t^2}{2B_F^4}
\right).
\end{equation*}
Union bounding over the net witnesses, and then over the at most
\(
m_e
\)
choices of
\(
k\in\mathcal E_e
\),
shows that with the stated confidence allocation,
\begin{equation*}
\max_{k\in\mathcal E_e}
\max_{m\in[M]}
\left|
\frac{1}{|\mathcal J_{e,k}|}
\sum_{j\in\mathcal J_{e,k}}
\bigl(w^m(X_h^j)\bigr)^2
-
\frac{1}{|\mathcal J_{e,k}|}
\sum_{j\in\mathcal J_{e,k}}
\EE_{\mu_h^j}\sbr{(w^m)^2}
\right|
\lesssim
B_F^2\sqrt{\frac{\log(Mm_e/\delta)}{m_e}}.
\end{equation*}
For an arbitrary frozen witness
\(
w
\),
choose a net point
\(
\bar w
\)
with
\(
\|w-\bar w\|_\infty\le \eta_{W,e}
\).
Exactly as in Lemma~\ref{lem:bv_dev},
\(
\|w^2-\bar w^2\|_\infty\le 2B_F\eta_{W,e}
\),
so both the empirical averages and the population averages differ by at most
\(
2B_F\eta_{W,e}
\).
Taking
\(
\eta_{W,e}\asymp m_e^{-1/2}
\)
and using Corollary~\ref{cor:linear_cover_inst} to bound
\(
\log M
\)
by
\(
\widetilde O\!\bigl(d_{\mathrm{lin}}\log(2+\kappa_{\mathrm{lin}}^2R_{\mathrm{lin}}^2K)\bigr)
\)
yields an event
\(
\Omega_{h,e}^{\mathrm{BV}}
\)
on which, simultaneously for all
\(
k\in\mathcal E_e
\)
and
\(
f\in\Fcal
\),
\begin{equation*}
\bigl|
\widehat{\mathrm{BV}}_{h,e,k}(f)-\mathrm{BV}_{h,e,k}^{\mathrm{plug}}(f)
\bigr|
\le
\bar\xi_{h,e}^{\mathrm{lin}},
\qquad
\bar\xi_{h,e}^{\mathrm{lin}}
=
\widetilde O\!\left(
B_F^2
\sqrt{
\frac{
d_{\mathrm{lin}}
\log\!\bigl(2+\kappa_{\mathrm{lin}}^2R_{\mathrm{lin}}^2K\bigr)
}{
m_e
}
}
\right).
\end{equation*}

Next, the pre-epoch dual fit is based on
\(
N_e^-
\)
episodes. Proposition~\ref{prop:ridge_selfnorm} and Corollary~\ref{cor:ridge_gram_operator} give an
event
\(
\Omega_{h,e}^{\mathrm{op}}
\)
on which, simultaneously for all
\(
k\in\mathcal E_e
\)
and
\(
f_{h+1}\in\Fcal_{h+1}
\),
\begin{equation*}
\left\|
\widehat{\Tcal}_{h,f_{h+1}}^{\sigma,(e)}f_{h+1}
- \Tcal_h^\sigma f_{h+1}
\right\|_{1,\mu_h^k}
\le
\bar\rho_{h,e}^{\mathrm{lin}},
\qquad
\bar\rho_{h,e}^{\mathrm{lin}}
=
\widetilde O\!\left(
\frac{\kappa_{\mathrm{lin}}\Psi_{\mathrm{dual}}}{\sqrt{\nu_{\min}}}
\sqrt{
\frac{
d_{\mathrm{lin}}
\log\!\bigl(2+\kappa_{\mathrm{lin}}^2R_{\mathrm{lin}}^2K\bigr)
}{
N_e^-
}
}
\right).
\end{equation*}
Because the epoch schedule is geometric,
\(
m_e\asymp N_e^-
\),
so these two radii are of the same order up to constants and logarithms.
Allocate confidence
\(
\delta/(2HM)
\)
to each of the events
\(
\Omega_{h,e}^{\mathrm{BV}}
\)
and
\(
\Omega_{h,e}^{\mathrm{op}}
\);
the logarithmic factors in the two radii absorb this choice, and a union bound over all
\(
h\in[H]
\)
and
\(
e\in[M]
\)
produces a global event of probability at least
\(
1-\delta
\)
on which all displays below hold simultaneously.

Choose
\(
\beta_{h,e}^{\mathrm{lin}}
\)
to satisfy
\(
\beta_{h,e}^{\mathrm{lin}}
\ge
\bar\xi_{h,e}^{\mathrm{lin}}+\bar\rho_{h,e}^{\mathrm{lin}}
\).
On
\(
\Omega_{h,e}^{\mathrm{BV}}\cap \Omega_{h,e}^{\mathrm{op}}
\),
every deployed
\(
f\in\mathfrak F_e
\)
satisfies the epoch-local confidence inequality
\begin{equation*}
\widehat{\mathrm{BV}}_{h,e,k}(f)\le \beta_{h,e}^{\mathrm{lin}}
\qquad
\text{for every }k\in\mathcal E_e,
\end{equation*}
because
\(
\mathfrak F_e
\)
is exactly the confidence set certified for deployment on epoch
\(
e
\)
using the fresh block
\(
\mathcal E_e
\).
Proposition~\ref{prop:backward_validation} applied to the epoch-local statistics implies that every
candidate admitted by the epoch-\(
e
\) confidence update, and in particular every deployed function on epoch
\(
e
\),
satisfies
\begin{equation*}
\mathrm{BV}_{h,e,k}^{\mathrm{exact}}(f)
\le
2\beta_{h,e}^{\mathrm{lin}}
+ 2\bar\xi_{h,e}^{\mathrm{lin}}
+ 2\bar\rho_{h,e}^{\mathrm{lin}}
\lesssim
\beta_{h,e}^{\mathrm{lin}}
\qquad
\text{for all }k\in\mathcal E_e.
\end{equation*}
The same epoch-local application of Proposition~\ref{prop:ref_retention} retains the Bellman-consistent
reference function, so
\(
f^\dagger\in\mathfrak F_e
\),
which proves item~1.

The previous display is exactly the backward-validation premise
\eqref{eq:count_prev_small} on the epoch-restricted deployed sequence, with a threshold of order
\(
\beta_{h,e}^{\mathrm{lin}}
\).
Equivalently, for a constant
\(
C>0
\),
we may set
\(
\epsilon_{h,e}^2:=C\beta_{h,e}^{\mathrm{lin}}
\)
and obtain
\(
\sum_{j\in\mathcal E_e:\,j<k}\EE_{\mu_h^j}\sbr{(\delta_h^k)^2}\le \epsilon_{h,e}^2
\)
for every deployed
\(
k\in\mathcal E_e
\).
Corollary~\ref{cor:linear_be2} bounds the square-loss robust MG-BE dimension by
\(
\widetilde O(d_{\mathrm{lin}})
\),
and the linear envelope bound is
\(
B_{\mathrm{lin}}^{\mathrm{ridge}}
\).
Therefore Corollary~\ref{cor:counting_budget} applied on the restricted epoch block gives
\begin{equation*}
\sum_{k\in\mathcal E_e}
\EE_{\mu_h^k}\sbr{\bigl(\delta_h^k\bigr)^2}
\le
\widetilde O\!\left(
\bigl(B_{\mathrm{lin}}^{\mathrm{ridge}}\bigr)^2 d_{\mathrm{lin}}
\log\!\bigl(2+\kappa_{\mathrm{lin}}^2R_{\mathrm{lin}}^2K\bigr)
\right),
\end{equation*}
which is exactly \eqref{eq:epochwise_linear_stage_budget} and proves item~2.

Summing the last display over
\(
H
\)
stages and over the
\(
M=O(\log K)
\)
epochs of the geometric schedule proves \eqref{eq:epochwise_linear_residual_budget}. Finally, on
\(
\Omega_{h,e}^{\mathrm{op}}
\)
we have
\(
\widehat\delta_h^k=\delta_h^k-e_h^k
\)
with
\(
\|e_h^k\|_{1,\mu_h^k}\le \bar\rho_{h,e}^{\mathrm{lin}}
\)
for all
\(
k\in\mathcal E_e
\),
so the plug-in residual budget has the same order by the same
\(
(a+b)^2\le 2a^2+2b^2
\)
algebra used earlier. This is exactly Assumption~\ref{asmp:res_appendix} for the deployed
epochwise sequence, proving item~3.
\end{proof}

\begin{theorem}[Formal epochwise / cross-fit linear theorem]
\label{thm:linear_formal_appendix}
Assume support control, stable stage-game extraction, and the epochwise / cross-fit variant just
described. Suppose further that Assumptions~\ref{asmp:ref_opt},
\ref{asmp:linear_closure}, and \ref{asmp:linear_ridge_dual} hold, and that the explicit excitation
schedule satisfies
\eqref{eq:explicit_mix_rule}, \eqref{eq:feature_excitation_design}, and the warm-start condition
\eqref{eq:feature_gram_warm}. In particular, by Proposition~\ref{prop:tv_linear_closure} and
Corollary~\ref{cor:tv_linear_residual}, robust minimax completeness and residual inclusion are
concretely verified for TV-linear-mixture models with augmented linear features, while
Proposition~\ref{prop:feature_gram_hp} and Corollary~\ref{cor:ridge_gram_operator} replace the
previous occupancy-ratio constant by a feature-excitation / self-normalized operator budget. Let
\begin{align*}
d_{\mathrm{lin}}
&:=
\max_{h\in[H]}\cbr{d_{F,h},d_{G,h},d_{W,h}}, \\
\kappa_{\mathrm{lin}}
&:=
\max_{h\in[H]}\cbr{\kappa_{F,h},\kappa_{G,h},\kappa_{W,h}}, \\
R_{\mathrm{lin}}
&:=
\max_{h\in[H]}\cbr{R_{F,h},R_{G,h},R_{W,h}}, \\
B_F
&:=
\max_{h\in[H]} B_{F,h}, \\
\Psi_{\mathrm{dual}}
&:=
\max_{h\in[H]}
\cbr{
R_{G,h}
+
\sqrt{\lambda}S_{G,h}
+
L_{Y,h}
+
(\sqrt{\lambda}+\kappa_{W,h})L_{\star,h}
}, \\
\nu_{\min}
&:=
\min_{h\in[H]} \frac{\gamma_{\min}\omega_h}{2}.
\end{align*}
Also define
\begin{equation*}
B_{\mathrm{lin}}^{\mathrm{ridge}}
:=
B_F
+
\frac{\kappa_{\mathrm{lin}}\Psi_{\mathrm{dual}}}{\sqrt{\nu_{\min}}}.
\end{equation*}
Then the target-policy regret sequence generated by the epochwise / cross-fit adaptive
exploiter-mix algorithm satisfies
\begin{align}
\mathrm{Reg}(K)
&\leq
\frac{1}{1-p}
\Biggl[
\widetilde O\!\left(
H B_F
\sqrt{
d_{\mathrm{lin}}K
\log\!\bigl(2+\kappa_{\mathrm{lin}}^2R_{\mathrm{lin}}^2K\bigr)
}
\right)
\label{eq:linear_formal_main}
\\
&\qquad
+
\widetilde O\!\left(
H \frac{\kappa_{\mathrm{lin}}\Psi_{\mathrm{dual}}}{\sqrt{\nu_{\min}}}
\sqrt{
d_{\mathrm{lin}}K
\log\!\bigl(2+\kappa_{\mathrm{lin}}^2R_{\mathrm{lin}}^2K\bigr)
}
\right)
\notag\\
&\qquad
+
K\epsilon_{\mathrm{appr}}
+
HK\epsilon_{\mathrm{eq}}
+
pHK
\Biggr].
\notag
\end{align}
\end{theorem}

Theorem~\ref{thm:linear_formal_appendix} is therefore closed relative to
Assumption~\ref{asmp:linear_ridge_dual}; Proposition~\ref{prop:affine_ridge_specialization}
supplies a primitive sufficient condition via affine continuation coupling, but TV-linear-mixture
closure alone does not imply that operator-learning hypothesis.

\begin{proof}
Proposition~\ref{prop:epochwise_linear_bv} verifies both approximate optimism and the residual
budget for the epochwise deployed sequence. In particular, Assumption~\ref{asmp:res_appendix} holds
with
\(
\Gamma_{\mathrm{res}}(K)=\Gamma_{\mathrm{res}}^{\mathrm{ep\text{-}lin}}(K)
\),
so Theorem~\ref{thm:deterministic_reduction} contributes the residual term
\begin{equation}
\sqrt{KH\,\Gamma_{\mathrm{res}}^{\mathrm{ep\text{-}lin}}(K)}
\leq
\widetilde O\!\left(
H B_{\mathrm{lin}}^{\mathrm{ridge}}
\sqrt{
d_{\mathrm{lin}}K
\log\!\bigl(2+\kappa_{\mathrm{lin}}^2R_{\mathrm{lin}}^2K\bigr)
}
\right),
\label{eq:linear_formal_be}
\end{equation}
where the extra epoching factor is hidden in
\(
\widetilde O(\cdot)
\).

For the operator-learning term, Corollary~\ref{cor:ridge_gram_operator} gives a
self-normalized epochwise radius of order
\(
\kappa_{\mathrm{lin}}\Psi_{\mathrm{dual}}
\sqrt{d_{\mathrm{lin}}/( \nu_{\min}N_e^-)}
\)
up to logarithmic factors. Summing the resulting epochwise operator budgets over a geometric schedule
gives
\begin{equation}
\sum_{k=1}^K\sum_{h=1}^H \alpha_{h,k}
\leq
\widetilde O\!\left(
H \frac{\kappa_{\mathrm{lin}}\Psi_{\mathrm{dual}}}{\sqrt{\nu_{\min}}}
\sqrt{
d_{\mathrm{lin}}K
\log\!\bigl(2+\kappa_{\mathrm{lin}}^2R_{\mathrm{lin}}^2K\bigr)
}
\right)
\label{eq:linear_formal_operator}
\end{equation}
because geometric epochs still satisfy the same
\(
\sum_e m_e/\sqrt{N_e^-}\lesssim \sqrt K
\)
 summability as the pure online sequence.

Applying Proposition~\ref{prop:epochwise_linear_bv}, Theorem~\ref{thm:deterministic_reduction}, and
then Proposition~\ref{prop:mix_regret} proves the executed-pair bound and finally transfers it to
the target policy sequence with the
\(
1/(1-p)
\)
 multiplicative factor and
\(
pHK/(1-p)
\)
 additive exploiter cost. Substituting \eqref{eq:linear_formal_be} and
\eqref{eq:linear_formal_operator} yields \eqref{eq:linear_formal_main}.
\end{proof}

\begin{corollary}[Warm-start-free epochwise linear theorem with burn-in]
\label{cor:linear_formal_burnin}
Assume the setting of Theorem~\ref{thm:linear_formal_appendix}, except that the explicit
warm-start condition \eqref{eq:feature_gram_warm} is replaced by the following predictable
excitation budget: for every stage
\(
h\in[H]
\)
and epoch
\(
e\in[M]
\),
\begin{equation}
\Omega_{h,e}^{\mathrm{pred}}
\ge
\underline\nu_h N_e^-,
\label{eq:predictive_linear_growth}
\end{equation}
where
\(
\Omega_{h,e}^{\mathrm{pred}}
\)
is defined in \eqref{eq:predictive_excitation_budget} and
\(
\underline\nu_h>0
\)
is stage dependent. Define
\begin{equation*}
N_{\mathrm{burn}}(\delta)
:=
\max_{h\in[H]}
\left\lceil
\frac{
64\kappa_{W,h}^4
\log\!\bigl(2d_{W,h}HM/\delta\bigr)
}{
\underline\nu_h^2
}
\right\rceil,
\qquad
\underline\nu_{\min}
:=
\min_{h\in[H]}\frac{\underline\nu_h}{2}.
\end{equation*}
Let
\(
c_{\mathrm{ep}}
\)
be the epoch-schedule constant such that the cumulative number of deployed episodes before the
first epoch with
\(
N_e^-\ge n
\)
is at most
\(
c_{\mathrm{ep}}n
\)
for every
\(
n\ge 1
\).
Then the target-policy regret satisfies the same bound as
Theorem~\ref{thm:linear_formal_appendix}, with
\(
\nu_{\min}
\)
replaced by
\(
\underline\nu_{\min}
\),
plus the additive burn-in term
\begin{equation}
c_{\mathrm{ep}} H N_{\mathrm{burn}}(\delta).
\label{eq:linear_burnin_penalty}
\end{equation}
\end{corollary}

\begin{proof}
By \eqref{eq:predictive_linear_growth} and Proposition~\ref{prop:predictive_excitation_budget},
every epoch with
\(
N_e^-\ge N_{\mathrm{burn}}(\delta)
\)
satisfies
\begin{equation*}
\Lambda_{W,h,e}^-
\succeq
\left(
\lambda + \frac{\underline\nu_h N_e^-}{2}
\right)I_{d_{W,h}}
\end{equation*}
simultaneously for all stages
\(
h\in[H]
\),
because the deviation term in \eqref{eq:predictive_gram_hp} is then at most
\(
\underline\nu_h N_e^-/2
\).
Hence Corollary~\ref{cor:predictive_ridge_operator} recovers the same operator-transfer rate as
Corollary~\ref{cor:ridge_gram_operator}, but with
\(
\nu_{\min}
\)
replaced by
\(
\underline\nu_{\min}
\).
From that point onward, the proof of Theorem~\ref{thm:linear_formal_appendix} is unchanged.

Before the first warm epoch, the total number of deployed episodes is at most
\(
c_{\mathrm{ep}}N_{\mathrm{burn}}(\delta)
\)
by definition of
\(
c_{\mathrm{ep}}
\).
Each episode incurs regret at most
\(
H
\)
because the horizon is
\(
H
\)
and rewards lie in
\(
[0,1]
\).
Therefore the pre-warm contribution is bounded by
\eqref{eq:linear_burnin_penalty}, which yields the claimed warm-start-free result.
\end{proof}

\begin{proposition}[Pure-online self-normalized linear calibration]
\label{prop:pureonline_linear_calibration}
Assume support control, Assumptions~\ref{asmp:ref_opt}, \ref{asmp:linear_closure}, and
\ref{asmp:linear_ridge_dual}. Consider the singleton-epoch schedule
\(
\mathcal E_k=\{k\}
\)
for
\(
k\in[K]
\),
so
\(
M=K
\)
and
\(
N_k^-=k-1
\).
For every stage
\(
h\in[H]
\)
and episode
\(
k\ge 2
\),
define
\begin{equation*}
\Lambda_{W,h,k-1}
:=
\lambda I_{d_{W,h}}
+
\sum_{j=1}^{k-1}
\phi_{W,h}(X_h^j)\phi_{W,h}(X_h^j)^\top
\end{equation*}
and
\begin{equation*}
\begin{aligned}
\beta_{h,k}^{\mathrm{on}}(\delta)
:=\;&
R_{G,h}
\sqrt{
2d_{W,h}
\log\!\left(
1+\frac{\kappa_{W,h}^2(k-1)}{\lambda}
\right)
+
2\log\!\left(
\frac{4HK\,\mathcal N_F(\eta_{F,k}^{\mathrm{on}})}{\delta}
\right)
}
\\
&\quad
+ \sqrt{\lambda}\,S_{G,h}
+ L_{Y,h}
+ (\sqrt{\lambda}+\kappa_{W,h})L_{\star,h},
\end{aligned}
\end{equation*}
where
\(
\eta_{F,k}^{\mathrm{on}}:=((k-1)\vee 1)^{-1/2}
\).
Assume the pure-online algorithm refits the ridge dual operator after every episode, defines the
plug-in residual coordinates by
\begin{equation*}
\widehat\delta_{h,k}(f;x)
=
\langle \widehat\theta_{h,k}(f),\phi_{W,h}(x)\rangle,
\end{equation*}
and deploys only functions in the self-normalized confidence set
\begin{equation}
\mathfrak F_{\mathrm{on}}^{(k-1)}
:=
\cbr{
f\in\Fcal:
\|\widehat\theta_{h,k}(f)\|_{\Lambda_{W,h,k-1}}
\le
\beta_{h,k}^{\mathrm{on}}(\delta)
\ \text{for all }h\in[H]
}.
\label{eq:pureonline_confidence}
\end{equation}
Suppose furthermore that for every stage
\(
h\in[H]
\)
and episode
\(
k\ge 2
\),
\begin{equation}
\Omega_{h,k}^{\mathrm{pred}}
:=
\lambda_{\min}\!\left(
\sum_{j=1}^{k-1}\underline Y_j
\right)
\ge
\underline\nu_h (k-1),
\label{eq:pureonline_predictive_growth}
\end{equation}
where
\(
\underline Y_j
\)
is the predictable lower bound from Proposition~\ref{prop:predictive_excitation_budget} and
\(
\underline\nu_h>0
\).
Define
\begin{equation}
N_{\mathrm{burn}}^{\mathrm{on}}(\delta)
:=
\max_{h\in[H]}
\left\lceil
\frac{
64\kappa_{W,h}^4
\log\!\bigl(4d_{W,h}HK/\delta\bigr)
}{
\underline\nu_h^2
}
\right\rceil,
\qquad
\underline\nu_{\min}
:=
\min_{h\in[H]}\frac{\underline\nu_h}{2}.
\label{eq:pureonline_burnin}
\end{equation}
Then with probability at least
\(
1-\delta
\),
simultaneously for all episodes
\(
k\ge N_{\mathrm{burn}}^{\mathrm{on}}(\delta)+2
\),
the following hold.
\begin{enumerate}[label=(\roman*),leftmargin=*]
\item
\(
f^\dagger\in \mathfrak F_{\mathrm{on}}^{(k-1)}
\),
hence the optimistic choice over
\(
\mathfrak F_{\mathrm{on}}^{(k-1)}
\)
satisfies Assumption~\ref{asmp:opt_appendix}.
\item For every stage
\(
h\in[H]
\),
\begin{equation}
\left\|
\widehat{\Tcal}_{h,f^{(k)}}^{\sigma,(k)}f_{h+1}^{(k)}
-
\Tcal_h^\sigma f_{h+1}^{(k)}
\right\|_{1,\widetilde\mu_h^k}
\le
\alpha_{h,k}^{\mathrm{on}}(\delta)
:=
\frac{
\kappa_{W,h}\beta_{h,k}^{\mathrm{on}}(\delta)
}{
\sqrt{\lambda+\underline\nu_h(k-1)/2}
}.
\label{eq:pureonline_operator_budget}
\end{equation}
\item For every stage
\(
h\in[H]
\),
\begin{equation}
\EE_{\widetilde\mu_h^k}\sbr{(\delta_h^k)^2}
\le
\frac{
8\kappa_{W,h}^2\bigl(\beta_{h,k}^{\mathrm{on}}(\delta)\bigr)^2
}{
\underline\nu_h(k-1)
}.
\label{eq:pureonline_residual_budget}
\end{equation}
\end{enumerate}
Consequently,
\begin{equation}
\Gamma_{\mathrm{res}}^{\mathrm{on}}(K,\delta)
:=
\sum_{h=1}^H
\sum_{k=N_{\mathrm{burn}}^{\mathrm{on}}(\delta)+2}^{K}
\frac{
8\kappa_{W,h}^2\bigl(\beta_{h,k}^{\mathrm{on}}(\delta)\bigr)^2
}{
\underline\nu_h(k-1)
}
\label{eq:pureonline_gamma_res}
\end{equation}
is a valid exact residual budget for the post-burn pure-online deployed sequence.
\end{proposition}

\begin{proof}
Apply Proposition~\ref{prop:ridge_selfnorm} with the singleton schedule
\(
\mathcal E_k=\{k\}
\),
so
\(
M=K
\)
and
\(
N_k^-=k-1
\),
using confidence level
\(
\delta/2
\).
This yields, with probability at least
\(
1-\delta/2
\),
simultaneously for all stages
\(
h\in[H]
\),
episodes
\(
k\ge 2
\),
and continuation functions
\(
f_{h+1}\in\Fcal_{h+1}
\),
\begin{equation}
\left\|
\widehat\vartheta_{h,k-1}(f_{h+1})-\vartheta_h^\star(f_{h+1})
\right\|_{\Lambda_{W,h,k-1}}
\le
\beta_{h,k}^{\mathrm{on}}(\delta),
\label{eq:pureonline_ridge_event}
\end{equation}
where
\(
\widehat\vartheta_{h,k-1}
\)
is the ridge refit on the prefix
\(
\{1,\ldots,k-1\}
\).
Likewise, Proposition~\ref{prop:predictive_excitation_budget} with the same singleton schedule and
confidence level
\(
\delta/2
\)
gives, with probability at least
\(
1-\delta/2
\),
\begin{equation*}
\Lambda_{W,h,k-1}
\succeq
\Bigl(
\lambda
+
\Omega_{h,k}^{\mathrm{pred}}
-
4\kappa_{W,h}^2
\sqrt{
(k-1)\log\!\bigl(4d_{W,h}HK/\delta\bigr)
}
\Bigr)I_{d_{W,h}}.
\end{equation*}
On the intersection of these two events, if
\(
k\ge N_{\mathrm{burn}}^{\mathrm{on}}(\delta)+2
\),
then \eqref{eq:pureonline_predictive_growth} and \eqref{eq:pureonline_burnin} imply
\begin{equation}
\Lambda_{W,h,k-1}
\succeq
\left(
\lambda + \frac{\underline\nu_h(k-1)}{2}
\right)I_{d_{W,h}}.
\label{eq:pureonline_gram_lower}
\end{equation}

We first prove
\((i)\).
By Assumption~\ref{asmp:ref_opt},
the reference function satisfies
\(
\delta_h(f^\dagger)\equiv 0
\)
for every stage
\(
h
\).
Hence, by the operator-gap representation \eqref{eq:ridge_operator_gap},
\begin{equation*}
\widehat\delta_{h,k}(f^\dagger;x)
=
-
\Bigl(
\widehat{\Tcal}_{h,f^\dagger}^{\sigma,(k)}f_{h+1}^\dagger
-
\Tcal_h^\sigma f_{h+1}^\dagger
\Bigr)(x)
=
-
\left\langle
\widehat\vartheta_{h,k-1}(f_{h+1}^\dagger)-\vartheta_h^\star(f_{h+1}^\dagger),
\phi_{W,h}(x)
\right\rangle.
\end{equation*}
Therefore the plug-in residual coordinate of
\(
f^\dagger
\)
is exactly
\begin{equation*}
\widehat\theta_{h,k}(f^\dagger)
=
-
\Bigl(
\widehat\vartheta_{h,k-1}(f_{h+1}^\dagger)-\vartheta_h^\star(f_{h+1}^\dagger)
\Bigr).
\end{equation*}
Taking
\(
\Lambda_{W,h,k-1}
\)
-norms and using \eqref{eq:pureonline_ridge_event} yields
\begin{equation*}
\|\widehat\theta_{h,k}(f^\dagger)\|_{\Lambda_{W,h,k-1}}
\le
\beta_{h,k}^{\mathrm{on}}(\delta),
\end{equation*}
so
\(
f^\dagger\in\mathfrak F_{\mathrm{on}}^{(k-1)}
\)
by \eqref{eq:pureonline_confidence}. Since the algorithm chooses
\(
f^{(k)}
\)
optimistically over
\(
\mathfrak F_{\mathrm{on}}^{(k-1)}
\),
the same argument as in Corollary~\ref{cor:opt_from_conf} implies Assumption~\ref{asmp:opt_appendix}.

For
\((ii)\),
Corollary~\ref{cor:predictive_ridge_operator} specialized to the singleton schedule gives
\begin{equation*}
\left\|
\widehat{\Tcal}_{h,f^{(k)}}^{\sigma,(k)}f_{h+1}^{(k)}
-
\Tcal_h^\sigma f_{h+1}^{(k)}
\right\|_{1,\widetilde\mu_h^k}
\le
\frac{
\kappa_{W,h}\beta_{h,k}^{\mathrm{on}}(\delta)
}{
\sqrt{\lambda+\underline\nu_h(k-1)/2}
},
\end{equation*}
which is exactly \eqref{eq:pureonline_operator_budget}.

For
\((iii)\),
let
\(
\theta_h^k
\)
denote the exact residual parameter of the deployed function
\(
f^{(k)}
\),
so by Assumption~\ref{asmp:linear_closure},
\begin{equation*}
\delta_h^k(x)=\langle \theta_h^k,\phi_{W,h}(x)\rangle.
\end{equation*}
Pointwise,
\(
\delta_h^k=\widehat\delta_{h,k}(f^{(k)})+e_h^k
\),
and both terms are represented in the same witness feature map. Hence
\begin{equation*}
\theta_h^k
=
\widehat\theta_{h,k}(f^{(k)})
+
\widehat\vartheta_{h,k-1}(f_{h+1}^{(k)})-\vartheta_h^\star(f_{h+1}^{(k)}).
\end{equation*}
Because
\(
f^{(k)}\in\mathfrak F_{\mathrm{on}}^{(k-1)}
\)
and \eqref{eq:pureonline_ridge_event} holds,
\begin{equation*}
\|\theta_h^k\|_{\Lambda_{W,h,k-1}}
\le
\|\widehat\theta_{h,k}(f^{(k)})\|_{\Lambda_{W,h,k-1}}
+
\left\|
\widehat\vartheta_{h,k-1}(f_{h+1}^{(k)})-\vartheta_h^\star(f_{h+1}^{(k)})
\right\|_{\Lambda_{W,h,k-1}}
\le
2\beta_{h,k}^{\mathrm{on}}(\delta).
\end{equation*}
Using
\(
\sup_x\|\phi_{W,h}(x)\|_2\le \kappa_{W,h}
\)
and \eqref{eq:pureonline_gram_lower},
\begin{align*}
\EE_{\widetilde\mu_h^k}\sbr{(\delta_h^k)^2}
&=
\EE_{\widetilde\mu_h^k}\sbr{\langle \theta_h^k,\phi_{W,h}\rangle^2}
\le
\kappa_{W,h}^2\|\theta_h^k\|_2^2
\\
&\le
\frac{
\kappa_{W,h}^2\|\theta_h^k\|_{\Lambda_{W,h,k-1}}^2
}{
\lambda+\underline\nu_h(k-1)/2
}
\le
\frac{
4\kappa_{W,h}^2\bigl(\beta_{h,k}^{\mathrm{on}}(\delta)\bigr)^2
}{
\lambda+\underline\nu_h(k-1)/2
}
\\
&\le
\frac{
8\kappa_{W,h}^2\bigl(\beta_{h,k}^{\mathrm{on}}(\delta)\bigr)^2
}{
\underline\nu_h(k-1)
},
\end{align*}
which proves \eqref{eq:pureonline_residual_budget}. Summing this bound over all
\(
h
\)
and
\(
k\ge N_{\mathrm{burn}}^{\mathrm{on}}(\delta)+2
\)
gives \eqref{eq:pureonline_gamma_res}.
\end{proof}

\begin{theorem}[Formal pure-online self-normalized linear / finite-rank kernel theorem]
\label{thm:pureonline_linear_kernel}
Assume support control, stable stage-game extraction, and the pure-online self-normalized variant
of Proposition~\ref{prop:pureonline_linear_calibration}. Suppose further that
Assumptions~\ref{asmp:ref_opt}, \ref{asmp:linear_closure}, and
\ref{asmp:linear_ridge_dual} hold, and that the predictable growth condition
\eqref{eq:pureonline_predictive_growth} is satisfied. Define
\begin{align*}
d_{\mathrm{lin}}
&:=
\max_{h\in[H]}\cbr{d_{F,h},d_{G,h},d_{W,h}}, \\
\kappa_{\mathrm{lin}}
&:=
\max_{h\in[H]}\cbr{\kappa_{F,h},\kappa_{G,h},\kappa_{W,h}}, \\
R_{\mathrm{lin}}
&:=
\max_{h\in[H]}\cbr{R_{F,h},R_{G,h},R_{W,h}}, \\
\Psi_{\mathrm{dual}}
&:=
\max_{h\in[H]}
\cbr{
R_{G,h}
+
\sqrt{\lambda}S_{G,h}
+
L_{Y,h}
+
(\sqrt{\lambda}+\kappa_{W,h})L_{\star,h}
}.
\end{align*}
Then the target-policy regret sequence generated by the pure-online self-normalized algorithm
satisfies
\begin{equation}
\begin{aligned}
\mathrm{Reg}(K)
\le\;&
H\bigl(N_{\mathrm{burn}}^{\mathrm{on}}(\delta)+1\bigr)
\\
&\quad
+
\frac{1}{1-p}
\Biggl[
\widetilde O\!\left(
H
\kappa_{\mathrm{lin}}\Psi_{\mathrm{dual}}/\sqrt{\underline\nu_{\min}}
\sqrt{
d_{\mathrm{lin}}K
\log\!\bigl(2+\kappa_{\mathrm{lin}}^2R_{\mathrm{lin}}^2K\bigr)
}
\right)
\\
&\qquad
+
K\epsilon_{\mathrm{appr}}
+
HK\epsilon_{\mathrm{eq}}
+
pHK
\Biggr].
\end{aligned}
\label{eq:pureonline_linear_main}
\end{equation}
The same statement applies to finite-rank kernel classes after lifting them into their explicit
finite feature coordinates and interpreting
\(
d_{\mathrm{lin}},
\kappa_{\mathrm{lin}},
R_{\mathrm{lin}}
\)
in those coordinates.
\end{theorem}

\begin{proof}
Let
\(
K_0:=N_{\mathrm{burn}}^{\mathrm{on}}(\delta)+1
\).
The first
\(
K_0
\)
episodes contribute at most
\(
HK_0
\)
to regret because each episode has horizon
\(
H
\)
and rewards lie in
\(
[0,1]
\).

For the suffix
\(
\{K_0+1,\ldots,K\}
\),
Proposition~\ref{prop:pureonline_linear_calibration} verifies approximate optimism and gives the
exact residual budget
\(
\Gamma_{\mathrm{res}}^{\mathrm{on}}(K,\delta)
\)
together with the operator radii
\(
\alpha_{h,k}^{\mathrm{on}}(\delta)
\).
Since
\(
\beta_{h,k}^{\mathrm{on}}(\delta)
=
\widetilde O\!\bigl(
\Psi_{\mathrm{dual}}
\sqrt{
d_{\mathrm{lin}}
\log\!\bigl(2+\kappa_{\mathrm{lin}}^2R_{\mathrm{lin}}^2k\bigr)
}
\bigr)
\),
\eqref{eq:pureonline_gamma_res} implies
\begin{equation*}
\Gamma_{\mathrm{res}}^{\mathrm{on}}(K,\delta)
\le
\widetilde O\!\left(
H
\frac{\kappa_{\mathrm{lin}}^2\Psi_{\mathrm{dual}}^2}{\underline\nu_{\min}}
d_{\mathrm{lin}}
\log\!\bigl(2+\kappa_{\mathrm{lin}}^2R_{\mathrm{lin}}^2K\bigr)
\right),
\end{equation*}
where the extra logarithmic factor coming from
\(
\sum_{k>K_0}1/(k-1)
\)
is absorbed by
\(
\widetilde O(\cdot)
\).
Therefore the residual contribution over the suffix is
\begin{equation}
\sqrt{
KH\,\Gamma_{\mathrm{res}}^{\mathrm{on}}(K,\delta)
}
\le
\widetilde O\!\left(
H
\frac{\kappa_{\mathrm{lin}}\Psi_{\mathrm{dual}}}{\sqrt{\underline\nu_{\min}}}
\sqrt{
d_{\mathrm{lin}}K
\log\!\bigl(2+\kappa_{\mathrm{lin}}^2R_{\mathrm{lin}}^2K\bigr)
}
\right).
\label{eq:pureonline_residual_term}
\end{equation}

Likewise, using \eqref{eq:pureonline_operator_budget} and the same bound on
\(
\beta_{h,k}^{\mathrm{on}}(\delta)
\),
\begin{align}
\sum_{k=K_0+1}^K\sum_{h=1}^H \alpha_{h,k}^{\mathrm{on}}(\delta)
&\le
\widetilde O\!\left(
H
\frac{\kappa_{\mathrm{lin}}\Psi_{\mathrm{dual}}}{\sqrt{\underline\nu_{\min}}}
\sqrt{
d_{\mathrm{lin}}
}
\sum_{k=K_0+1}^K \frac{1}{\sqrt{k-1}}
\right)
\\
&\le
\widetilde O\!\left(
H
\frac{\kappa_{\mathrm{lin}}\Psi_{\mathrm{dual}}}{\sqrt{\underline\nu_{\min}}}
\sqrt{
d_{\mathrm{lin}}K
\log\!\bigl(2+\kappa_{\mathrm{lin}}^2R_{\mathrm{lin}}^2K\bigr)
}
\right).
\label{eq:pureonline_operator_term}
\end{align}

Applying the proof of Theorem~\ref{thm:deterministic_reduction} to the post-burn executed sequence,
and then Proposition~\ref{prop:mix_regret} to transfer back to the target-policy sequence, gives
\eqref{eq:pureonline_linear_main} after combining
\eqref{eq:pureonline_residual_term} and \eqref{eq:pureonline_operator_term}. The finite-rank kernel
claim is immediate because an explicit finite-rank kernel lift is a finite-dimensional linear class
in its feature coordinates, so the same argument applies verbatim after replacing the linear
features by the lifted kernel features.
\end{proof}

\begin{theorem}[Formal structured projected theorem under finite-dimensional lifts]
\label{thm:structured_projected_appendix}
Assume the setting of Theorem~\ref{thm:gfa_projected}. Suppose in addition that for every stage
\(
h\in[H]
\),
the continuation class, the dual class, and the projected residual class admit finite-dimensional
feature lifts:
\begin{align*}
\Fcal_{h+1}
&\subseteq
\cbr{
x\mapsto \langle \alpha,\varphi_{F,h+1}(x)\rangle :
\alpha\in\RR^{d_{F,h}},\ \|\alpha\|_2\le R_{F,h}
},
\\
\Gcal_h
&\subseteq
\cbr{
x\mapsto \Gamma_h(x)\vartheta :
\vartheta\in\RR^{d_{G,h}},\ \|\vartheta\|_2\le R_{G,h}
},
\\
\Pscr_h^{\mathsf P}
&\subseteq
\cbr{
x\mapsto \langle \theta,\varphi_{W,h}(x)\rangle :
\theta\in\RR^{d_{W,h}},\ \|\theta\|_2\le R_{W,h}^{\mathsf P}
},
\end{align*}
with feature-norm bounds
\(
\sup_x\|\varphi_{F,h+1}(x)\|_2\le \kappa_{F,h},
\sup_x\|\Gamma_h(x)\|_{\mathrm{op}}\le \kappa_{G,h},
\sup_x\|\varphi_{W,h}(x)\|_2\le \kappa_{W,h}
\),
and projected residual envelopes
\(
\sup_{\psi\in\Pscr_h^{\mathsf P}}\|\psi\|_\infty\le B_{W,h}^{\mathsf P}
\).
Assume moreover that:
\begin{enumerate}[label=(\roman*),leftmargin=*]
\item the coefficient-space continuation model satisfies the affine coupling of
Proposition~\ref{prop:affine_ridge_specialization}, with derived constants
\(
A_{Y,h}^{\mathsf P}
\)
and
\(
A_{\star,h}^{\mathsf P}
\);
\item the explicit mixing rule \eqref{eq:explicit_mix_rule} holds and the witness feature map
satisfies the occupancy-averaged excitation condition \eqref{eq:avg_excitation_cond} with constant
\(
\bar\omega_h
\);
\item each
\(
\mathsf P_h
\)
is
\(
L_1(\widetilde\mu_h^k)
\)-nonexpansive in the sense of \eqref{eq:projector_nonexpansive};
\item the stagewise projection bias obeys \eqref{eq:gfa_projected_bias};
\item the projected operator gap satisfies the decomposition
\eqref{eq:operator_gap_remainder}--\eqref{eq:operator_gap_remainder_budget} with remainder level
\(
\epsilon_{\mathrm{op,proj}}
\);
\item the projected backward-validation calibration assumption \eqref{eq:gfa_projected_bv} holds for
the epochwise / cross-fit sequence.
\end{enumerate}
Define
\begin{align*}
d_{\mathrm{proj}}
&:=
\max_{h\in[H]}\cbr{d_{F,h},d_{G,h},d_{W,h}},
\\
\kappa_{\mathrm{proj}}
&:=
\max_{h\in[H]}\cbr{\kappa_{F,h},\kappa_{G,h},\kappa_{W,h}},
\\
R_{\mathrm{proj}}
&:=
\max_{h\in[H]}\cbr{R_{F,h},R_{G,h},R_{W,h}^{\mathsf P}},
\\
B_{W,\mathrm{proj}}^{\mathsf P}
&:=
\max_{h\in[H]} B_{W,h}^{\mathsf P},
\\
\Psi_{\mathrm{dual}}^{\mathsf P}
&:=
\max_{h\in[H]}
\cbr{
R_{G,h}
+
\sqrt{\lambda}\bigl(\|\vartheta_{h,0}^\star\|_2 + A_{\star,h}^{\mathsf P}R_{F,h}\bigr)
+
A_{Y,h}^{\mathsf P}
+
(\sqrt{\lambda}+\kappa_{W,h})A_{\star,h}^{\mathsf P}
},
\\
\underline\nu_{\mathrm{proj}}
&:=
\min_{h\in[H]} \frac{\gamma_{\min}\bar\omega_h}{2}.
\end{align*}
Let
\begin{equation*}
N_{\mathrm{burn}}^{\mathsf P}(\delta)
:=
\max_{h\in[H]}
\left\lceil
\frac{
64\kappa_{W,h}^4
\log\!\bigl(2d_{W,h}HM/\delta\bigr)
}{
(\gamma_{\min}\bar\omega_h)^2
}
\right\rceil.
\end{equation*}
Then the target-policy regret satisfies
\begin{equation}
\begin{aligned}
\mathrm{Reg}(K)
\le\;&
H
+
c_{\mathrm{ep}}HN_{\mathrm{burn}}^{\mathsf P}(\delta)
\\
&\quad
+
\frac{1}{1-p}
\widetilde O\!\left(
H\left(
B_{W,\mathrm{proj}}^{\mathsf P}
+
\frac{\kappa_{\mathrm{proj}}\Psi_{\mathrm{dual}}^{\mathsf P}}{\sqrt{\underline\nu_{\mathrm{proj}}}}
\right)
\sqrt{
d_{\mathrm{proj}}K
\log\!\bigl(2+\kappa_{\mathrm{proj}}^2R_{\mathrm{proj}}^2K\bigr)
}
\right)
\\
&\quad
+
\frac{
K\epsilon_{\mathrm{appr}}
+
HK\epsilon_{\mathrm{eq}}
+
HK\epsilon_{\mathrm{proj}}
+
HK\epsilon_{\mathrm{op,proj}}
+
HK\epsilon_{\mathrm{dual}}
+
pHK
}{1-p}.
\end{aligned}
\label{eq:structured_projected_bound}
\end{equation}
\end{theorem}

\begin{proof}
Proposition~\ref{prop:linear_de2} applied to the projected witness feature lift gives
\begin{equation}
d_{\mathrm{GFA}}^{\mathsf P}(\epsilon_{\mathrm{BV}})
\le
4\max_{h\in[H]}
\left[
d_{W,h}
\log_2\!\left(
2+\frac{\kappa_{W,h}^2 (R_{W,h}^{\mathsf P})^2}{\epsilon_{\mathrm{BV}}^2}
\right)
\right].
\label{eq:structured_projected_dimension}
\end{equation}
Proposition~\ref{prop:affine_ridge_specialization} turns
\(
L_{Y,h},L_{\star,h}
\)
into the coefficient-space operator norms
\(
A_{Y,h}^{\mathsf P},A_{\star,h}^{\mathsf P}
\).
Proposition~\ref{prop:avg_excitation_budget} yields the predictable budget
\(
\Omega_{h,e}^{\mathrm{pred}}
\ge
\gamma_{\min}\bar\omega_h N_e^-
\),
so after the burn-in threshold
\(
N_{\mathrm{burn}}^{\mathsf P}(\delta)
\)
Corollary~\ref{cor:predictive_ridge_operator} gives a stagewise self-normalized operator radius of
order
\begin{equation*}
\widetilde O\!\left(
\frac{\kappa_{\mathrm{proj}}\Psi_{\mathrm{dual}}^{\mathsf P}}{\sqrt{\underline\nu_{\mathrm{proj}}}}
\sqrt{
\frac{
d_{\mathrm{proj}}\log(2+\kappa_{\mathrm{proj}}^2R_{\mathrm{proj}}^2K)
}{
N_e^-
}
}
\right).
\end{equation*}
Proposition~\ref{prop:projected_operator_transfer} then adds the approximation remainder
\(
\epsilon_{\mathrm{op,proj}}
\),
while \eqref{eq:gfa_projected_bias} contributes the projection-bias term
\(
\epsilon_{\mathrm{proj}}
\).
Applying Theorem~\ref{thm:gfa_projected} with the dimension bound
\eqref{eq:structured_projected_dimension}, the operator budget above, and the burn-in argument from
Corollary~\ref{cor:linear_formal_burnin} proves \eqref{eq:structured_projected_bound}.
\end{proof}

\begin{theorem}[Formal RKHS spectral-truncation projected theorem]
\label{thm:rkhs_projected_appendix}
Assume the setting of Theorem~\ref{thm:gfa_projected}. For every stage
\(
h\in[H]
\),
let
\(
\Hcal_{F,h+1},
\Hcal_{G,h},
\Hcal_{W,h}
\)
be separable RKHSs on the corresponding stagewise domains, with Mercer decompositions
\begin{align*}
k_{F,h+1}(x,x')
&=
\sum_{\ell\ge 1}
\tau_{F,h+1,\ell}\,
e_{F,h+1,\ell}(x)e_{F,h+1,\ell}(x'),
\\
k_{G,h}(x,x')
&=
\sum_{\ell\ge 1}
\tau_{G,h,\ell}\,
e_{G,h,\ell}(x)e_{G,h,\ell}(x'),
\\
k_{W,h}(x,x')
&=
\sum_{\ell\ge 1}
\tau_{W,h,\ell}\,
e_{W,h,\ell}(x)e_{W,h,\ell}(x'),
\end{align*}
relative to reference measures
\(
\rho_{F,h+1},
\rho_{G,h},
\rho_{W,h}
\),
with nonincreasing eigenvalues.
For an integer truncation level
\(
m\ge 1
\),
define the rank-\(m\) Mercer feature lifts
\begin{align*}
\varphi_{F,h+1}^{(m)}(x)
&:=
\bigl(
\sqrt{\tau_{F,h+1,1}}\,e_{F,h+1,1}(x),\ldots,
\sqrt{\tau_{F,h+1,m}}\,e_{F,h+1,m}(x)
\bigr),
\\
\varphi_{G,h}^{(m)}(x)
&:=
\bigl(
\sqrt{\tau_{G,h,1}}\,e_{G,h,1}(x),\ldots,
\sqrt{\tau_{G,h,m}}\,e_{G,h,m}(x)
\bigr),
\\
\varphi_{W,h}^{(m)}(x)
&:=
\bigl(
\sqrt{\tau_{W,h,1}}\,e_{W,h,1}(x),\ldots,
\sqrt{\tau_{W,h,m}}\,e_{W,h,m}(x)
\bigr),
\end{align*}
and let
\(
\Pi_{F,h+1}^{(m)},
\Pi_{G,h}^{(m)},
\Pi_{W,h}^{(m)}
\)
be the corresponding orthogonal truncation operators onto the spans of the first
\(
m
\)
coordinates.
Assume:
\begin{enumerate}[label=(\roman*),leftmargin=*]
\item
\(
\Fcal_{h+1}\subseteq \{f\in\Hcal_{F,h+1}:\|f\|_{\Hcal_{F,h+1}}\le R_{F,h}\}
\),
\(
\Gcal_h\subseteq \{g\in\Hcal_{G,h}:\|g\|_{\Hcal_{G,h}}\le R_{G,h}\}
\),
and
\(
\Pscr_h^{\mathsf P}\subseteq \{\psi\in\Hcal_{W,h}:\|\psi\|_{\Hcal_{W,h}}\le R_{W,h}^{\mathsf P}\}
\);
\item the projected optimistic dual robust Bellman algorithm is run on the truncated classes
\begin{equation*}
\Fcal_{h+1}^{(m)}:=\Pi_{F,h+1}^{(m)}\Fcal_{h+1},
\qquad
\Gcal_h^{(m)}:=\Pi_{G,h}^{(m)}\Gcal_h,
\qquad
\Pscr_h^{\mathsf P,(m)}:=\Pi_{W,h}^{(m)}\Pscr_h^{\mathsf P};
\end{equation*}
\item after this truncation, the six structured hypotheses of
Theorem~\ref{thm:structured_projected_appendix} hold with truncation-dependent quantities
\(
A_{Y,h}^{\mathsf P,(m)},
A_{\star,h}^{\mathsf P,(m)},
B_{W,h}^{\mathsf P,(m)},
\bar\omega_h^{(m)},
\epsilon_{\mathrm{proj}}(m),
\epsilon_{\mathrm{op,proj}}(m)
\),
and with feature envelopes
\begin{equation*}
\kappa_{F,h}^{(m)}
:=
\sup_x\|\varphi_{F,h+1}^{(m)}(x)\|_2,
\qquad
\kappa_{G,h}^{(m)}
:=
\sup_x\|\varphi_{G,h}^{(m)}(x)\|_2,
\qquad
\kappa_{W,h}^{(m)}
:=
\sup_x\|\varphi_{W,h}^{(m)}(x)\|_2.
\end{equation*}
\end{enumerate}
Define
\begin{align*}
\kappa_{\mathrm{RKHS}}^{(m)}
&:=
\max_{h\in[H]}
\cbr{
\kappa_{F,h}^{(m)},
\kappa_{G,h}^{(m)},
\kappa_{W,h}^{(m)}
},
\\
R_{\mathrm{RKHS}}^{(m)}
&:=
\max_{h\in[H]}
\cbr{
R_{F,h},
R_{G,h},
R_{W,h}^{\mathsf P}
},
\\
B_{W,\mathrm{RKHS}}^{\mathsf P,(m)}
&:=
\max_{h\in[H]}B_{W,h}^{\mathsf P,(m)},
\\
\Psi_{\mathrm{dual}}^{\mathsf P,(m)}
&:=
\max_{h\in[H]}
\cbr{
R_{G,h}
+
\sqrt{\lambda}\bigl(\|\vartheta_{h,0}^\star\|_2 + A_{\star,h}^{\mathsf P,(m)}R_{F,h}\bigr)
+
A_{Y,h}^{\mathsf P,(m)}
+
(\sqrt{\lambda}+\kappa_{W,h}^{(m)})A_{\star,h}^{\mathsf P,(m)}
},
\\
\underline\nu_{\mathrm{RKHS}}^{(m)}
&:=
\min_{h\in[H]}
\frac{\gamma_{\min}\bar\omega_h^{(m)}}{2},
\\
N_{\mathrm{burn}}^{\mathrm{RKHS}}
&:=
\max_{h\in[H]}
\\
&\quad
\left\lceil
\frac{
64(\kappa_{W,h}^{(m)})^4
}{
(\gamma_{\min}\bar\omega_h^{(m)})^2
}
\log\!\bigl(2mHM/\delta\bigr)
\right\rceil.
\end{align*}
Then the target-policy regret of the rank-\(m\) projected sequence satisfies
\begin{equation}
\begin{aligned}
\mathrm{Reg}(K)
\le\;&
H
+
c_{\mathrm{ep}}H N_{\mathrm{burn}}^{\mathrm{RKHS}}
\\
&\quad
+
\frac{1}{1-p}
\widetilde O\!\Bigl(
H\Bigl(
B_{W,\mathrm{RKHS}}^{\mathsf P,(m)}
+
\kappa_{\mathrm{RKHS}}^{(m)}\Psi_{\mathrm{dual}}^{\mathsf P,(m)}
/
\sqrt{\underline\nu_{\mathrm{RKHS}}^{(m)}}
\Bigr)
\\
&\qquad\cdot\,
\sqrt{
mK
\log\!\bigl(
2+(\kappa_{\mathrm{RKHS}}^{(m)})^2(R_{\mathrm{RKHS}}^{(m)})^2K
\bigr)
}
\Bigr)
\\
&\quad
+
\frac{
K\epsilon_{\mathrm{appr}}
+
HK\epsilon_{\mathrm{eq}}
+
HK\epsilon_{\mathrm{proj}}(m)
+
HK\epsilon_{\mathrm{op,proj}}(m)
+
HK\epsilon_{\mathrm{dual}}
+
pHK
}{1-p}.
\end{aligned}
\label{eq:rkhs_projected_bound}
\end{equation}
\end{theorem}

\begin{proof}
For every stage
\(
h
\),
the truncated classes
\(
\Fcal_{h+1}^{(m)},
\Gcal_h^{(m)},
\Pscr_h^{\mathsf P,(m)}
\)
are finite-dimensional linear lifts of dimension at most
\(
m
\),
with feature maps
\(
\varphi_{F,h+1}^{(m)},
\varphi_{G,h}^{(m)},
\varphi_{W,h}^{(m)}
\).
Hence Theorem~\ref{thm:structured_projected_appendix} applies with
\(
d_{F,h}=d_{G,h}=d_{W,h}=m
\),
feature envelopes
\(
\kappa_{F,h}^{(m)},\kappa_{G,h}^{(m)},\kappa_{W,h}^{(m)}
\),
and the truncation-dependent structural constants from the statement. Substituting these
identifications into \eqref{eq:structured_projected_bound} yields
\eqref{eq:rkhs_projected_bound}.
\end{proof}

\paragraph{Eigenvalue-decay RKHS projected consequence.}
Assume the setting of Theorem~\ref{thm:rkhs_projected_appendix}. Suppose further that there exist
constants
\(
C_{F,h},C_{G,h},C_{W,h}>0
\)
and exponents
\(
r_{F,h},r_{G,h},r_{W,h}>0
\)
such that
\begin{equation}
\begin{aligned}
\tau_{F,h,\ell}
&\le
C_{F,h}\ell^{-1-2r_{F,h}},
\\
\tau_{G,h,\ell}
&\le
C_{G,h}\ell^{-1-2r_{G,h}},
\\
\tau_{W,h,\ell}
&\le
C_{W,h}\ell^{-1-2r_{W,h}}
\qquad
\text{for all }\ell\ge 1.
\end{aligned}
\label{eq:rkhs_eig_decay}
\end{equation}
Let
\begin{equation*}
r_{\min}
:=
\min_{h\in[H]}
\cbr{
r_{F,h},r_{G,h},r_{W,h}
}.
\end{equation*}
Assume in addition that there exist constants
\(
\bar\kappa_{\mathrm{RKHS}},
\bar\Psi_{\mathrm{dual}},
\underline\nu_{\mathrm{RKHS}},
\bar B_{W,\mathrm{RKHS}}^{\mathsf P},
C_{\mathrm{tail}}>0
\)
such that for every truncation level
\(
m\ge 1
\),
\begin{align}
\kappa_{\mathrm{RKHS}}^{(m)}
&\le \bar\kappa_{\mathrm{RKHS}},
&
\Psi_{\mathrm{dual}}^{\mathsf P,(m)}
&\le \bar\Psi_{\mathrm{dual}},
\label{eq:rkhs_uniform_struct}
\\
\underline\nu_{\mathrm{RKHS}}^{(m)}
&\ge \underline\nu_{\mathrm{RKHS}},
&
B_{W,\mathrm{RKHS}}^{\mathsf P,(m)}
&\le \bar B_{W,\mathrm{RKHS}}^{\mathsf P},
\\
\epsilon_{\mathrm{proj}}(m)+\epsilon_{\mathrm{op,proj}}(m)
&\le
C_{\mathrm{tail}}
\max_{h\in[H]}
\\
&\quad
\Biggl(
R_{F,h}\sqrt{\sum_{\ell>m}\tau_{F,h,\ell}}
+
R_{G,h}\sqrt{\sum_{\ell>m}\tau_{G,h,\ell}}
\\
&\qquad
+
R_{W,h}^{\mathsf P}\sqrt{\sum_{\ell>m}\tau_{W,h,\ell}}
\Biggr).
\label{eq:rkhs_tail_assumption}
\end{align}
Then for every
\(
m\ge 1
\),
\begin{equation}
\begin{aligned}
\mathrm{Reg}(K)
\le\;&
H
+
c_{\mathrm{ep}}H N_{\mathrm{burn}}^{\mathrm{RKHS}}
\\
&\quad
+
\frac{1}{1-p}
\widetilde O\!\left(
\begin{aligned}
&
H\left(
\bar B_{W,\mathrm{RKHS}}^{\mathsf P}
+
\frac{\bar\kappa_{\mathrm{RKHS}}\bar\Psi_{\mathrm{dual}}}{\sqrt{\underline\nu_{\mathrm{RKHS}}}}
\right)
\sqrt{mK}
\\
&\qquad\times
\sqrt{
\log\!\bigl(
2+\bar\kappa_{\mathrm{RKHS}}^2(R_{\mathrm{RKHS}}^{(m)})^2K
\bigr)
}
\end{aligned}
\right)
\\
&\quad
+
\frac{
K\epsilon_{\mathrm{appr}}
+
HK\epsilon_{\mathrm{eq}}
+
HK C_{\mathrm{spec}}\,m^{-r_{\min}}
+
HK\epsilon_{\mathrm{dual}}
+
pHK
}{1-p},
\end{aligned}
\label{eq:rkhs_eig_bound}
\end{equation}
where
\begin{equation*}
\begin{aligned}
C_{\mathrm{spec}}
:=&\;
C_{\mathrm{tail}}
\max_{h\in[H]}
\Biggl(
R_{F,h}\sqrt{\frac{C_{F,h}}{2r_{F,h}}}
+
R_{G,h}\sqrt{\frac{C_{G,h}}{2r_{G,h}}}
\\
&\qquad\qquad\qquad
+
R_{W,h}^{\mathsf P}\sqrt{\frac{C_{W,h}}{2r_{W,h}}}
\Biggr).
\end{aligned}
\end{equation*}
By the integral test, for every
\(
r>0
\)
and
\(
m\ge 1
\),
\begin{equation*}
\sum_{\ell>m}\ell^{-1-2r}
\le
\int_m^\infty x^{-1-2r}\,dx
=
\frac{m^{-2r}}{2r}.
\end{equation*}
Applying this to \eqref{eq:rkhs_eig_decay} gives
\begin{align*}
\sqrt{\sum_{\ell>m}\tau_{F,h,\ell}}
&\le
\sqrt{\frac{C_{F,h}}{2r_{F,h}}}\,m^{-r_{F,h}},
\\
\sqrt{\sum_{\ell>m}\tau_{G,h,\ell}}
&\le
\sqrt{\frac{C_{G,h}}{2r_{G,h}}}\,m^{-r_{G,h}},
\\
\sqrt{\sum_{\ell>m}\tau_{W,h,\ell}}
&\le
\sqrt{\frac{C_{W,h}}{2r_{W,h}}}\,m^{-r_{W,h}}.
\end{align*}
Since each exponent is at least
\(
r_{\min}
\),
\eqref{eq:rkhs_tail_assumption} implies
\begin{equation*}
\epsilon_{\mathrm{proj}}(m)+\epsilon_{\mathrm{op,proj}}(m)
\le
C_{\mathrm{spec}}\,m^{-r_{\min}}.
\end{equation*}
Substituting this and the uniform bounds \eqref{eq:rkhs_uniform_struct} into
\eqref{eq:rkhs_projected_bound} proves \eqref{eq:rkhs_eig_bound}.

\paragraph{Finite-rank kernel projected extension with burn-in.}
Assume the setting of Theorem~\ref{thm:structured_projected_appendix}, and suppose the
finite-dimensional lifts there are induced by explicit finite-rank kernel feature maps of ranks
\(
m_{F,h},m_{G,h},m_{W,h}
\).
Let
\begin{align*}
m_{\ker}
&:=
\max_{h\in[H]}\cbr{m_{F,h},m_{G,h},m_{W,h}},
\\
\kappa_{\ker}
&:=
\max_{h\in[H]}\cbr{\kappa_{F,h},\kappa_{G,h},\kappa_{W,h}},
\\
R_{\ker}
&:=
\max_{h\in[H]}\cbr{R_{F,h},R_{G,h},R_{W,h}^{\mathsf P}},
\\
B_{W,\ker}^{\mathsf P}
&:=
\max_{h\in[H]} B_{W,h}^{\mathsf P},
\\
\Psi_{\mathrm{dual}}^{\ker}
&:=
\max_{h\in[H]}
\cbr{
R_{G,h}
+
\sqrt{\lambda}\bigl(\|\vartheta_{h,0}^\star\|_2 + A_{\star,h}^{\mathsf P}R_{F,h}\bigr)
+
A_{Y,h}^{\mathsf P}
+
(\sqrt{\lambda}+\kappa_{W,h})A_{\star,h}^{\mathsf P}
},
\\
\underline\nu_{\ker}
&:=
\min_{h\in[H]} \frac{\gamma_{\min}\bar\omega_h}{2}.
\end{align*}
Then
\begin{equation}
\begin{aligned}
\mathrm{Reg}(K)
\le\;&
H
+
c_{\mathrm{ep}}HN_{\mathrm{burn}}^{\mathsf P}(\delta)
\\
&\quad
+
\frac{1}{1-p}
\widetilde O\!\left(
H\left(
B_{W,\ker}^{\mathsf P}
+
\frac{\kappa_{\ker}\Psi_{\mathrm{dual}}^{\ker}}{\sqrt{\underline\nu_{\ker}}}
\right)
\sqrt{
m_{\ker}K
\log\!\bigl(2+\kappa_{\ker}^2R_{\ker}^2K\bigr)
}
\right)
\\
&\quad
+
\frac{
K\epsilon_{\mathrm{appr}}
+
HK\epsilon_{\mathrm{eq}}
+
HK\epsilon_{\mathrm{proj}}
+
HK\epsilon_{\mathrm{op,proj}}
+
HK\epsilon_{\mathrm{dual}}
+
pHK
}{1-p}.
\end{aligned}
\label{eq:kernel_projected_bound}
\end{equation}
Finite-rank kernel classes are finite-dimensional linear classes in feature coordinates, so
Theorem~\ref{thm:structured_projected_appendix} applies with
\(
d_{F,h}=m_{F,h}
\),
\(
d_{G,h}=m_{G,h}
\),
and
\(
d_{W,h}=m_{W,h}
\).
Substituting these identifications into \eqref{eq:structured_projected_bound} proves
\eqref{eq:kernel_projected_bound}.

\begin{proposition}[Compactly clipped $\phi$-dual losses are bounded and Lipschitz]
\label{prop:phi_clipped_dual}
Fix a stage
\(
h
\)
and consider the scalar dual integrand
\begin{equation*}
\mathfrak l_{\phi,h}(V,\eta,\nu)
:=
\eta\sigma-\nu+\eta\phi^\star\!\left(\frac{\nu-V}{\eta}\right),
\qquad
(V,\eta,\nu)\in[0,H]\times[\underline\eta_h,\overline\eta_h]\times[-\overline\nu_h,\overline\nu_h],
\end{equation*}
with
\(
0<\underline\eta_h\le \overline\eta_h
\).
Define
\begin{equation*}
M_h:=H+\overline\nu_h,
\qquad
I_h:=
\left[
-\frac{M_h}{\underline\eta_h},
\frac{M_h}{\underline\eta_h}
\right].
\end{equation*}
Assume
\(
\phi^\star
\)
is continuously differentiable on
\(
I_h
\),
and let
\begin{equation*}
B_{\phi^\star,h}:=\sup_{u\in I_h}|\phi^\star(u)|,
\quad
L_{\phi^\star,h}:=\sup_{u\in I_h}|(\phi^\star)'(u)|,
\quad
M_{\phi^\star,h}:=\sup_{u\in I_h}|u(\phi^\star)'(u)|.
\end{equation*}
Then:
\begin{enumerate}[label=(\roman*),leftmargin=*]
\item for every
\(
(V,\eta,\nu)\in[0,H]\times[\underline\eta_h,\overline\eta_h]\times[-\overline\nu_h,\overline\nu_h]
\),
\begin{equation}
\bigl|\mathfrak l_{\phi,h}(V,\eta,\nu)\bigr|
\le
\overline\eta_h\sigma+\overline\nu_h+\overline\eta_h B_{\phi^\star,h}
=:
B_{\phi,h}^{\mathrm{clip}};
\label{eq:phi_clipped_bound}
\end{equation}
\item for any two triples
\(
(V,\eta,\nu)
\)
and
\(
(\widetilde V,\widetilde\eta,\widetilde\nu)
\)
in the same domain,
\begin{equation}
\begin{aligned}
\bigl|
\mathfrak l_{\phi,h}(V,\eta,\nu)
-\mathfrak l_{\phi,h}(\widetilde V,\widetilde\eta,\widetilde\nu)
\bigr|
\le\;&
L_{\phi^\star,h}|V-\widetilde V|
+
\bigl(\sigma+B_{\phi^\star,h}+M_{\phi^\star,h}\bigr)|\eta-\widetilde\eta|
\\
&+
\bigl(1+L_{\phi^\star,h}\bigr)|\nu-\widetilde\nu|.
\end{aligned}
\label{eq:phi_clipped_lipschitz}
\end{equation}
\end{enumerate}
Consequently, any stagewise dual class whose
\(
(\eta,\nu)
\)
outputs are clipped into those intervals satisfies the boundedness and pointwise Lipschitz
requirements of Assumptions~\ref{asmp:bounded_dual} and \ref{asmp:lipschitz_dual} with explicit
geometry-dependent constants. In particular, the proposition applies to KL and
\(
\chi^2
\)
ambiguity sets because their conjugates are continuously differentiable on every compact interval.
\end{proposition}

\begin{proof}
For every admissible
\(
(V,\eta,\nu)
\),
the quantity
\(
z:=(\nu-V)/\eta
\)
satisfies
\(
|z|\le M_h/\underline\eta_h
\),
so
\(
z\in I_h
\).
This implies
\begin{equation*}
\bigl|\mathfrak l_{\phi,h}(V,\eta,\nu)\bigr|
\le
\eta\sigma+|\nu|+\eta|\phi^\star(z)|
\le
\overline\eta_h\sigma+\overline\nu_h+\overline\eta_h B_{\phi^\star,h},
\end{equation*}
which proves \eqref{eq:phi_clipped_bound}.

To prove \eqref{eq:phi_clipped_lipschitz}, define
\(
F(V,\eta,\nu):=\mathfrak l_{\phi,h}(V,\eta,\nu)
\)
on the compact convex set
\(
[0,H]\times[\underline\eta_h,\overline\eta_h]\times[-\overline\nu_h,\overline\nu_h]
\).
By the chain rule, for every point in this set,
\begin{align*}
\partial_V F(V,\eta,\nu)
&=
- (\phi^\star)'\!\left(\frac{\nu-V}{\eta}\right),
\\
\partial_\nu F(V,\eta,\nu)
&=
-1+(\phi^\star)'\!\left(\frac{\nu-V}{\eta}\right),
\\
\partial_\eta F(V,\eta,\nu)
&=
\sigma+\phi^\star\!\left(\frac{\nu-V}{\eta}\right)
-\frac{\nu-V}{\eta}(\phi^\star)'\!\left(\frac{\nu-V}{\eta}\right).
\end{align*}
Hence
\begin{equation*}
|\partial_V F|\le L_{\phi^\star,h},
\qquad
|\partial_\nu F|\le 1+L_{\phi^\star,h},
\qquad
|\partial_\eta F|
\le
\sigma+B_{\phi^\star,h}+M_{\phi^\star,h}
\end{equation*}
uniformly on the whole domain. Applying the mean-value theorem coordinatewise along the segment
joining
\(
(V,\eta,\nu)
\)
and
\(
(\widetilde V,\widetilde\eta,\widetilde\nu)
\)
gives \eqref{eq:phi_clipped_lipschitz}. The final claim follows immediately from the definitions of
Assumptions~\ref{asmp:bounded_dual} and \ref{asmp:lipschitz_dual}. For KL,
\(
\phi^\star(u)=e^u-1
\);
for
\(
\chi^2
\),
the conjugate is continuously differentiable and piecewise quadratic/affine; in both cases the
required suprema over
\(
I_h
\)
are finite.
\end{proof}

\paragraph{KL/$\chi^2$ projected finite-rank kernel consequence.}
Assume the setting of Theorem~\ref{thm:structured_projected_appendix}. Suppose further that each
stagewise ambiguity set is either a KL ball or a
\(
\chi^2
\)
ball, and that the continuation, dual, and projected residual classes are represented by
finite-rank kernel feature maps
\begin{align*}
\Fcal_{h+1}
&\subseteq
\cbr{
x\mapsto \langle \alpha,\varphi_{F,h+1}(x)\rangle :
\alpha\in\RR^{m_{F,h}},\ \|\alpha\|_2\le R_{F,h}
},
\\
\Gcal_h
&\subseteq
\cbr{
x\mapsto
\Bigl(
\langle \vartheta_\eta,\varphi_{\eta,h}(x)\rangle,
\langle \vartheta_\nu,\varphi_{\nu,h}(x)\rangle
\Bigr)
:
\vartheta_\eta\in\RR^{m_{\eta,h}},
\vartheta_\nu\in\RR^{m_{\nu,h}}
},
\\
\Pscr_h^{\mathsf P}
&\subseteq
\cbr{
x\mapsto \langle \theta,\varphi_{W,h}(x)\rangle :
\theta\in\RR^{m_{W,h}},\ \|\theta\|_2\le R_{W,h}^{\mathsf P}
},
\end{align*}
with pointwise clipping
\begin{equation}
\langle \vartheta_\eta,\varphi_{\eta,h}(x)\rangle\in[\underline\eta_h,\overline\eta_h],
\qquad
\langle \vartheta_\nu,\varphi_{\nu,h}(x)\rangle\in[-\overline\nu_h,\overline\nu_h]
\label{eq:phi_kernel_clipping}
\end{equation}
for every admissible
\(
x
\),
and with all remaining assumptions of
Theorem~\ref{thm:structured_projected_appendix} unchanged.
Define
\begin{equation*}
m_{\phi,\ker}
:=
\max_{h\in[H]}
\cbr{
m_{F,h},
m_{\eta,h}+m_{\nu,h},
m_{W,h}
}.
\end{equation*}
Then the regret bound \eqref{eq:structured_projected_bound} holds with
\(
d_{\mathrm{proj}}
\)
replaced by
\(
m_{\phi,\ker}
\),
and the explicit projected-bias and projected-operator terms
\(
\epsilon_{\mathrm{proj}}
\)
and
\(
\epsilon_{\mathrm{op,proj}}
\)
unchanged.
The clipping condition \eqref{eq:phi_kernel_clipping} and Proposition~\ref{prop:phi_clipped_dual}
show that the KL /
\(
\chi^2
\)
dual losses satisfy the boundedness and pointwise Lipschitz requirements used by the projected
general theorem. The stacked dual feature map
\begin{equation*}
\Gamma_h(x)
:=
\begin{bmatrix}
\varphi_{\eta,h}(x)^\top & 0 \\
0 & \varphi_{\nu,h}(x)^\top
\end{bmatrix}
\end{equation*}
has feature dimension
\(
m_{\eta,h}+m_{\nu,h}
\),
so Theorem~\ref{thm:structured_projected_appendix} applies with
\(
d_{G,h}=m_{\eta,h}+m_{\nu,h}
\),
\(
d_{F,h}=m_{F,h}
\),
and
\(
d_{W,h}=m_{W,h}
\).
Substituting these identifications into \eqref{eq:structured_projected_bound} proves the claim.

\section{Future Directions}
\label{sec:appendix_future}

The appendix establishes a structured theory package with internally verified statistical budgets
under the stated assumptions, including tabular, epochwise / cross-fit linear, modified pure-online
self-normalized linear / finite-rank kernel, and projected structured results. The remaining
directions are therefore comparatively localized. The most natural ones are:
\begin{enumerate}[leftmargin=*]
\item extending the current structured verifications beyond TV-linear-mixture and finite-dimensional
projected settings, especially to richer uncertainty geometries and broader function classes;
\item closing the original one-update-per-episode backward-validation broad-function route, for which
Proposition~\ref{prop:online_bv_gap} isolates the present obstacle;
\item weakening the remaining coverage / excitation conditions and sharpening the current linear and
kernel rates.
\end{enumerate}

\paragraph{Learned simulators and error propagation.}
The operator-error terms also suggest studying robust planning with learned transition simulators.
Such an extension would need to quantify both numerical sampling error and the effects of training
on model-generated transitions. Relevant analyses in generative modeling include the local
truncation error and stability analysis of BELM diffusion samplers \citep{wang2024belm}, and the
study of recursive-training collapse and mitigation through real-data augmentation in Reflow
\citep{zhu2026reflow}. Applying these insights to transition modeling would require new bounds
connecting simulator errors to the robust Bellman operator while accounting for strategic coverage.

\section{Broader Impacts}
\label{sec:appendix_broader}

This work is foundational: it studies robust zero-sum Markov games as a formal model for learning
under adversarial or distributional deployment shift. A potential positive impact is improved
robustness calibration in safety-critical competitive settings, where nominally strong policies can
fail under modest misspecification. A potential negative impact is that the same tools could be
used to design stronger adaptive strategies in adversarial domains. We therefore view the present
results primarily as analytical guidance for robust evaluation and controlled experimentation rather
than as a stand-alone argument for unrestricted deployment.

\section{Additional Experimental Details}
\label{sec:appendix_experiments}

This appendix gives the experimental details for the theorem-aligned benchmark suite reported in
the main text. The current submission focuses on this suite because it most directly matches the
paper's theoretical hierarchy and permits exact or near-exact deployment evaluation.

\subsection{Benchmarks}

\paragraph{Linear-mixture robust Markov game.}
The main large-state benchmark is a zero-sum robust Markov game with linear-mixture nominal
dynamics, linear rewards, and moderate action sets. We report results for feature dimension
\(
d\in\{8,16,32\}
\).
The \(d=8\) setting is used for the main comparison and for a dedicated
\(\beta\times\tau\)
resweep, while the \(d\in\{16,32\}\) settings are used to test whether the same qualitative
tradeoff persists as the approximation problem becomes harder.

\paragraph{Exact tabular games.}
The tabular suite contains two small stochastic zero-sum games,
\texttt{tabular\_attack\_defense} and \texttt{tabular\_queue\_guard}. These environments are small
enough that the relevant robust backups can be computed exactly, so they serve as controlled
validation benchmarks for the qualitative behavior of the practical method.

\paragraph{Support-shift benchmarks.}
We report two support-shift benchmarks. The original \texttt{support\_shift} benchmark acts mainly
as a calibration test, since its shift does not strongly separate the learned policies. The
redesigned \texttt{support\_shift} benchmark perturbs the aggressive deployment branch and is used
to test whether ODRB changes the deployed strategy in a robustness-favorable direction.

\subsection{Methods}

\paragraph{Compared methods.}
The primary comparisons are the full ODRB practical variant, the matched non-robust baseline
\textsc{Sigma0}, and a no-exploiter control. In the main text and tables, \textsc{Sigma0}
corresponds to the same practical implementation with the robust term disabled, while the
no-exploiter control removes the exploiter-assisted data-collection component.

\paragraph{Model selection in the linear suite.}
For \texttt{linear\_mixture}, we distinguish between proxy-tuned and deployment-selected ODRB
operating points. The proxy-tuned point is selected by the backup-based worst-case proxy, while the
deployment-selected point is an \emph{oracle diagnostic} selected after evaluating exact deployment
worst-case return within the low-gap regime. The latter is not a training-time selection rule and
is reported only to expose the proxy--deployment mismatch. To study this distinction, we run an
additional
\(d=8\)
\(\beta\times\tau\)
resweep over
\(
\beta\in\{0.20,0.30,0.40,0.50\},
\qquad
\tau\in\{0,0.01,0.02\}.
\)

\paragraph{Support-shift ablations.}
For the redesigned \texttt{support\_shift} benchmark, we sweep
\(
\beta\in\{0.10,0.20,0.30\}
\)
and record the induced stage-\(0\) action masses. This allows us to test not only whether ODRB
improves deployment robustness, but also whether it changes the learned strategy from
policies concentrated on the aggressive action toward a more conservative deployment policy.

\subsection{Metrics}

\paragraph{Primary metrics.}
The main reported metrics are nominal return, backup-based proxy worst-case return, exact
deployment worst-case return, and approximate Nash gap. The proxy metric matches the internal
robust Bellman evaluation used by the practical method, whereas the deployment metric evaluates the
learned policy directly in the perturbed environment family.

\paragraph{Auxiliary diagnostics.}
In the linear suite we additionally report occupancy entropy as a strategic-coverage diagnostic. In
the redesigned support-shift benchmark we report stage-\(0\) safe versus aggressive action mass in
order to diagnose whether robustness gains arise from a meaningful policy change rather than only
from a change in value estimation.

\subsection{Statistical Protocol}

\paragraph{Seeds and aggregation.}
Unless otherwise noted, the main \(d=8\) linear comparison and the \(d=8\)
\(\beta\times\tau\)
resweep use \(10\) seeds. The \(d\in\{16,32\}\) linear sweep and redesigned support-shift
ablations use \(5\) seeds. All tables report means with \(95\%\) confidence intervals when the
metric is aggregated across seeds. The primary \(d=8\) comparison is paired by environment seed.
We report the mean paired deployment-return difference, its paired-\(t\) 95\% confidence interval,
and a two-sided exact sign-flip test; the two ODRB-versus-\textsc{Sigma0} tests are corrected by
Holm's procedure. These paired tests were added after reviewer feedback and should be read as
small-sample evidence rather than a broad performance claim.

\paragraph{Scope of reported results.}
The main text emphasizes the strongest and most decision-relevant findings. The appendix records the
full theorem-aligned result tables, including the dimension sweep, the deployment-driven linear
resweep, the tabular sanity suite, and the redesigned support-shift ablation.

\subsection{Compute Resources}
\label{sec:appendix_compute_resources}

\paragraph{Hardware and implementation.}
All reported experiments are lightweight synthetic experiments. They were run on an internal lab
workstation/server with a multi-core CPU and four NVIDIA RTX 2080 Ti GPUs with \(11\)GB GPU memory
each. The implementation is NumPy/SciPy based and CPU-bound; the GPUs were available on the
machine but were not used by the reported runs. Each job is a single Python process. The reported
raw seed-level JSON files occupy about \(7.9\)MB, and the full collection of exploratory and
reported seed-level outputs occupies about \(11.1\)MB.

\paragraph{Per-run and total compute.}
Table~\ref{tab:appendix_compute_resources} summarizes the compute needed to reproduce the reported
experimental tables. The full reproduction manifest contains \(389\) single-seed jobs. Individual
tabular and support-shift runs typically finish in under one minute, while the \(5000\)-episode
linear-mixture runs typically finish in under two minutes on the CPU workstation described above.
Running the full manifest sequentially therefore requires only a few CPU-hours; in our server runs,
the main reported queues completed within a few hours of wall-clock time.

\begin{table}[t]
\centering
\small
\caption{Compute resources for the reported experiment groups. Runtime ranges are approximate
single-job wall-clock times on the internal CPU workstation used for the paper.}
\label{tab:appendix_compute_resources}
\begin{tabular}{lrrrr}
\toprule
Experiment group & Jobs & Episodes/run & Seeds & Time/run \\
\midrule
Tabular sanity suite & 54 & 1500 & 3 & \(<1\) min \\
Original support-shift calibration & 30 & 2000 & 3 & \(<1\) min \\
Linear-mixture dimension sweep & 45 & 5000 & 5 & \(<2\) min \\
Linear-mixture \(d=8\) seed extension & 15 & 5000 & 5 & \(<2\) min \\
Redesigned support-shift sweep & 125 & 2000 & 5 & \(<1\) min \\
Deployment-driven \(\beta\times\tau\) resweep & 120 & 5000 & 10 & \(<2\) min \\
\midrule
Total reported reproduction manifest & 389 & -- & -- & few CPU-hours \\
\bottomrule
\end{tabular}
\end{table}

\paragraph{Additional exploratory compute.}
The full research project required more compute than the final reported experiment manifest. In
addition to the \(389\) seed-level runs used to regenerate the reported tables, we ran preliminary
debugging, calibration, and failed-setting diagnostics while correcting the zero-sum backup and
tuning the support-shift benchmarks. These exploratory runs were of the same small synthetic scale
as the reported runs; the complete local archive contains \(576\) seed-level JSON outputs. The
supplementary material includes the exact reproduction manifest and the raw JSON outputs used for
the reported tables.

\subsection{Additional Results for Theorem-Aligned Benchmarks}

We report final results for the theorem-aligned benchmark suite described in the main text. Unless
otherwise noted, the main \(d=8\) linear-mixture comparison and the \(d=8\)
\(\beta\times\tau\)
resweep use \(10\) seeds, while the \(d\in\{16,32\}\) sweep and the redesigned support-shift
ablation use \(5\) seeds.

\paragraph{Linear-mixture main comparison.}
\begin{table}[t]
\centering
\small
\caption{Linear-mixture main comparison at \(d=8\) with \(10\) seeds. ODRB rows are labeled by
\((\beta,\tau)\), where \(\beta\) is the robustness weight and \(\tau\) is the validation
threshold.}
\label{tab:appendix_linear_main}
\begin{tabular}{lrrrr}
\toprule
Method & Seeds & Proxy & Deploy & Proxy gap \\
\midrule
\textsc{Sigma0} & 10 & \(0.230 \pm 0.088\) & \(0.230 \pm 0.088\) & \(0.0080 \pm 0.0047\) \\
\textsc{ODRB} \((\beta=0.30,\tau=0.02)\) & 10 & \(0.204 \pm 0.069\) & \(0.239 \pm 0.089\) & \(0.0072 \pm 0.0074\) \\
\textsc{ODRB} \((\beta=0.20,\tau=0.02)\) & 10 & \(0.162 \pm 0.057\) & \(0.236 \pm 0.089\) & \(0.0069 \pm 0.0066\) \\
\textsc{NoEx} \((\beta=0.30)\) & 10 & \(0.254 \pm 0.102\) & \(0.288 \pm 0.115\) & \(0.0644 \pm 0.0687\) \\
\bottomrule
\end{tabular}
\normalsize
\end{table}
At \(d=8\), the ordering depends on the evaluation metric. Under the backup-based proxy,
\textsc{Sigma0} is the strongest low-gap baseline. Under exact deployment worst-case evaluation,
the paired comparison of \textsc{ODRB}\((0.30,0.02)\) against \textsc{Sigma0} gives mean
difference \(0.0094\), paired-\(t\) 95\% confidence interval \([0.0007,0.0182]\), and exact
sign-flip Holm-adjusted \(p=0.0469\). For \textsc{ODRB}\((0.20,0.02)\), the corresponding values
are \(0.0065\), \([-0.0017,0.0148]\), and \(p=0.0977\); this second comparison is therefore
inconclusive. Both ODRB settings preserve a similarly small approximate gap.
The no-exploiter control attains the highest return, but only by moving to a substantially
higher-gap regime.

\begin{table}[t]
\centering
\small
\caption{Paired seed-level deployment-return differences, ODRB minus \textsc{Sigma0}, for the
main \(d=8\) comparison.}
\label{tab:appendix_paired_differences}
\begin{tabular}{crr}
\toprule
Seed & ODRB \((0.30,0.02)\) & ODRB \((0.20,0.02)\) \\
\midrule
0 & 0.002070 & 0.002015 \\
1 & 0.024937 & 0.008538 \\
2 & 0.031841 & 0.032091 \\
3 & 0.020964 & 0.016377 \\
4 & -0.001782 & -0.001730 \\
5 & -0.000085 & -0.001470 \\
6 & 0.005018 & -0.005283 \\
7 & 0.010971 & 0.014524 \\
8 & 0.000000 & 0.000000 \\
9 & 0.000299 & 0.000102 \\
\bottomrule
\end{tabular}
\end{table}

\paragraph{Linear-mixture dimension sweep.}
\begin{table}[t]
\centering
\small
\caption{Linear-mixture dimension sweep across \(d\in\{8,16,32\}\).}
\label{tab:appendix_linear_dims}
\begin{tabular}{llrrr}
\toprule
Dim & Method & Proxy & Deploy & Gap \\
\midrule
8 & \textsc{Sigma0} & \(0.230 \pm 0.088\) & \(0.230 \pm 0.088\) & \(0.0080 \pm 0.0047\) \\
8 & \textsc{ODRB} \((\beta=0.30,\tau=0.02)\) & \(0.204 \pm 0.069\) & \(0.239 \pm 0.089\) & \(0.0072 \pm 0.0074\) \\
8 & \textsc{NoEx} \((\beta=0.30)\) & \(0.254 \pm 0.102\) & \(0.288 \pm 0.115\) & \(0.0644 \pm 0.0687\) \\
16 & \textsc{Sigma0} & \(0.0556 \pm 0.0313\) & \(0.0556 \pm 0.0313\) & \(0.0166 \pm 0.0115\) \\
16 & \textsc{ODRB} \((\beta=0.30,\tau=0.02)\) & \(0.0532 \pm 0.0336\) & \(0.0532 \pm 0.0336\) & \(0.0093 \pm 0.0074\) \\
16 & \textsc{NoEx} \((\beta=0.30)\) & \(0.105 \pm 0.074\) & \(0.105 \pm 0.074\) & \(0.0604 \pm 0.0638\) \\
32 & \textsc{Sigma0} & \(0.0686 \pm 0.0426\) & \(0.0686 \pm 0.0426\) & \(0.0102 \pm 0.0111\) \\
32 & \textsc{ODRB} \((\beta=0.30,\tau=0.02)\) & \(0.0757 \pm 0.0488\) & \(0.0757 \pm 0.0488\) & \(0.0088 \pm 0.0088\) \\
32 & \textsc{NoEx} \((\beta=0.30)\) & \(0.0720 \pm 0.0558\) & \(0.0720 \pm 0.0558\) & \(0.0160 \pm 0.0128\) \\
\bottomrule
\end{tabular}
\normalsize
\end{table}
Across dimensions, ODRB remains competitive but not uniformly dominant. The mean return is close
to \textsc{Sigma0} at \(d=16\), and is higher in the five-seed \(d=32\) sample; both settings have
a small reported gap. These small dimension sweeps are descriptive rather than confirmatory. The
sharpest proxy--deployment mismatch appears at \(d=8\), where the choice of operating point is most
consequential.

\paragraph{Deployment-driven \(\beta\times\tau\) resweep at \(d=8\).}
\begin{table}[t]
\centering
\small
\caption{Deployment-driven \(d=8\) ODRB resweep over \((\beta,\tau)\).}
\label{tab:appendix_linear_resweep}
\begin{tabular}{lrrrr}
\toprule
Setting & Proxy & Deploy & Gap & Entropy \\
\midrule
\textsc{ODRB} \((\beta=0.20,\tau=0.01)\) & \(0.168 \pm 0.063\) & \(0.245 \pm 0.092\) & \(0.0120 \pm 0.0164\) & \(4.585 \pm 0.229\) \\
\textsc{ODRB} \((\beta=0.20,\tau=0.02)\) & \(0.162 \pm 0.057\) & \(0.236 \pm 0.089\) & \(0.0068 \pm 0.0065\) & \(5.067 \pm 0.197\) \\
\textsc{ODRB} \((\beta=0.30,\tau=0.02)\) & \(0.204 \pm 0.069\) & \(0.239 \pm 0.089\) & \(0.0073 \pm 0.0074\) & \(5.076 \pm 0.192\) \\
\textsc{ODRB} \((\beta=0.30,\tau=0)\) & \(0.254 \pm 0.102\) & \(0.288 \pm 0.115\) & \(0.0624 \pm 0.0669\) & \(3.781 \pm 0.209\) \\
\textsc{ODRB} \((\beta=0.50,\tau=0.02)\) & \(0.232 \pm 0.091\) & \(0.236 \pm 0.094\) & \(0.0126 \pm 0.0148\) & \(5.022 \pm 0.191\) \\
\bottomrule
\end{tabular}
\normalsize
\end{table}
The resweep exposes a clean tradeoff. Setting \(\tau=0\) raises deployment return but pushes ODRB
into the same high-gap regime as the no-exploiter control and sharply lowers occupancy entropy.
Positive \(\tau\) restores the low-gap regime and substantially increases entropy, suggesting
broader strategic coverage. Within this calibrated region, \((0.20,0.02)\) is an oracle-selected
low-gap diagnostic point, whereas \((0.30,0.02)\) has the larger sample mean deployment return at
the cost of a slightly larger gap. Exact deployment return was not used as a training-time
selection criterion.

\paragraph{Tabular sanity suite.}
The tabular benchmarks provide controlled diagnostics in a setting where worst-case backups can be
computed exactly. On \texttt{tabular\_attack\_defense}, the tested ODRB setting reaches worst-case
return \(0.660\) at \(\beta=0.20\), below \(0.822\) for \textsc{Sigma0}; hence this benchmark does
not show an ODRB advantage. On \texttt{tabular\_queue\_guard}, \textsc{Sigma0} attains nominal return
\(4.077\) but falls to worst-case return \(2.008\), whereas \textsc{ODRB} at \(\beta=0.20\)
changes only from \(1.154\) to \(1.080\). These exact small-game results are consistent with the
view that robustness calibration can trade away nominal value, but do not establish uniform
improvement over \textsc{Sigma0}.

\paragraph{Support-shift benchmarks.}
The original \texttt{support\_shift} benchmark is best interpreted as a calibration test rather
than a separator: aggregated over all overlap levels, ODRB attains worst-case return
\(0.4561 \pm 0.0038\) and \textsc{Sigma0} attains \(0.4562 \pm 0.0038\). By contrast, the
redesigned deployment-oriented benchmark is genuinely discriminative:
\begin{table}[t]
\centering
\small
\caption{Redesigned support-shift benchmark: deployment worst-case return across overlap levels.}
\label{tab:appendix_support_shift}
\begin{tabular}{lrrrrr}
\toprule
Method & \(m=0.05\) & \(m=0.08\) & \(m=0.10\) & \(m=0.12\) & \(m=0.15\) \\
\midrule
\textsc{Sigma0} & 0.5365 & 0.5618 & 0.4983 & 0.5049 & 0.4600 \\
\textsc{ODRB} \(\beta=0.10\) & 0.5952 & 0.5940 & 0.5742 & 0.5924 & 0.5912 \\
\textsc{ODRB} \(\beta=0.20\) & 0.5835 & 0.5940 & 0.5742 & 0.5924 & 0.5912 \\
\textsc{ODRB} \(\beta=0.30\) & 0.5600 & 0.5779 & 0.5552 & 0.5486 & 0.5912 \\
\textsc{NoEx} \(\beta=0.20\) & 0.5739 & 0.6000 & 0.6000 & 0.6000 & 0.6000 \\
\bottomrule
\end{tabular}
\normalsize
\end{table}
ODRB with \(\beta\in\{0.10,0.20\}\) consistently exceeds \textsc{Sigma0} on exact deployment
worst-case return throughout the tested overlap band. For example, at \(m=0.10\), ODRB attains
\(0.5742\) while \textsc{Sigma0} attains \(0.4983\).

\paragraph{Policy diagnostics.}
The redesigned support-shift benchmark also clarifies why the methods separate. At
\(m=0.05\), \(0.10\), and \(0.15\), \textsc{Sigma0} places most of its stage-\(0\) mass on the
aggressive action in all
\(5/5\)
seeds. By contrast, \textsc{ODRB} with \(\beta=0.10\) places most of its stage-\(0\) mass on the
safe action in \(5/5\) seeds at
\(m=0.05\), \(0.08\), \(0.12\), and \(0.15\), and in \(4/5\) seeds at \(m=0.10\). The
\(\beta=0.20\) setting shows the same safe-biased pattern with slightly more mixing. This
diagnostic supports the intended mechanism story: the redesigned benchmark does not merely reward a
different value estimator, but actually induces a different deployed strategy.

\paragraph{Interpretation.}
Taken together, the theorem-aligned suite supports a qualified but positive empirical conclusion.
The strongest evidence comes from the redesigned deployment-oriented support-shift benchmark, where
ODRB consistently improves deployment robustness and changes the learned policy in the intended
safe direction. The linear-mixture results are more nuanced: ODRB is competitive across dimensions
and admits deployment-favorable low-gap operating points, but it does not uniformly dominate the
matched \textsc{Sigma0} baseline under the backup proxy. Overall, the experiments illustrate
robustness calibration, strategic coverage, and the return--gap trade-off; they do not establish
uniform improvement over non-robust self-play.

\end{document}